\documentclass{article}

\usepackage{etoolbox}
\usepackage{comment}
\newtoggle{neurips}
\togglefalse{neurips}

\newcommand{\arxiv}[1]{\iftoggle{neurips}{}{#1}}
\newcommand{\loose}{\looseness=-1}

\arxiv{
\usepackage[letterpaper, left=1in, right=1in, top=1in,
bottom=1in]{geometry}
  \usepackage{parskip}
}

\PassOptionsToPackage{hypertexnames=false}{hyperref}  %

\usepackage{amsmath}

\usepackage[dvipsnames]{xcolor}
\usepackage[colorlinks=true, linkcolor=blue!70!black, citecolor=blue!70!black,urlcolor=black,breaklinks=true]{hyperref}
\usepackage{microtype}
\usepackage{hhline}

\usepackage{amsthm}
\usepackage{mathtools}
\usepackage{bbm}
\usepackage{amsfonts}
\usepackage{amssymb}
\usepackage[nameinlink,capitalize]{cleveref}

\makeatletter
\newcommand{\neutralize}[1]{\expandafter\let\csname c@#1\endcsname\count@}
\makeatother

\usepackage{algorithm}

\arxiv{
\usepackage{natbib}
\bibliographystyle{plainnat}
\bibpunct{(}{)}{;}{a}{,}{,}
}

\usepackage{xpatch}

\usepackage{thm-restate}
\usepackage{autonum}
\declaretheorem[name=Theorem,parent=section]{theorem}
\declaretheorem[name=Lemma,parent=section]{lemma}
\declaretheorem[name=Assumption, parent=section]{assumption}
\declaretheorem[name=Definition, parent=section]{definition}

\declaretheorem[name=Corollary, parent=section]{corollary}

\declaretheorem[name=Remark, parent=section]{remark}
\declaretheorem[name=Proposition, parent=section]{proposition}

\usepackage{crossreftools}
  \newcommand{\creftitle}[1]{\crtcref{#1}}

\makeatletter
  \renewenvironment{proof}[1][Proof]%
  {%
   \par\noindent{\bfseries\upshape {#1.}\ }%
  }%
  {\qed\newline}
  \makeatother

\xpatchcmd{\proof}{\itshape}{\normalfont\proofnameformat}{}{}
\newcommand{\proofnameformat}{\bfseries}

\newcommand{\pref}[1]{\cref{#1}}
\newcommand{\pfref}[1]{Proof of \pref{#1}}

\renewcommand{\eqref}[1]{\texorpdfstring{\hyperref[#1]{(\ref*{#1})}}{(\ref*{#1})}}
\crefformat{equation}{#2Eq.\,(#1)#3}
\Crefformat{equation}{#2Eq.\,(#1)#3}

\Crefformat{figure}{#2Figure~#1#3}
\Crefformat{assumption}{#2Assumption~#1#3}

\Crefname{assumption}{Assumption}{Assumptions}

\crefname{fact}{Fact}{Facts}

\Crefformat{figure}{#2Figure #1#3}
\Crefformat{assumption}{#2Assumption #1#3}

\usepackage{xparse}

\ExplSyntaxOn
\DeclareDocumentCommand{\XDeclarePairedDelimiter}{mm}
 {
  \__egreg_delimiter_clear_keys: %
  \keys_set:nn { egreg/delimiters } { #2 }
  \use:x %
   {
    \exp_not:n {\NewDocumentCommand{#1}{sO{}m} }
     {
      \exp_not:n { \IfBooleanTF{##1} }
       {
        \exp_not:N \egreg_paired_delimiter_expand:nnnn
         { \exp_not:V \l_egreg_delimiter_left_tl }
         { \exp_not:V \l_egreg_delimiter_right_tl }
         { \exp_not:n { ##3 } }
         { \exp_not:V \l_egreg_delimiter_subscript_tl }
       }
       {
        \exp_not:N \egreg_paired_delimiter_fixed:nnnnn 
         { \exp_not:n { ##2 } }
         { \exp_not:V \l_egreg_delimiter_left_tl }
         { \exp_not:V \l_egreg_delimiter_right_tl }
         { \exp_not:n { ##3 } }
         { \exp_not:V \l_egreg_delimiter_subscript_tl }
       }
     }
   }
 }

\keys_define:nn { egreg/delimiters }
 {
  left      .tl_set:N = \l_egreg_delimiter_left_tl,
  right     .tl_set:N = \l_egreg_delimiter_right_tl,
  subscript .tl_set:N = \l_egreg_delimiter_subscript_tl,
 }

\cs_new_protected:Npn \__egreg_delimiter_clear_keys:
 {
  \keys_set:nn { egreg/delimiters } { left=.,right=.,subscript={} }
 }

\cs_new_protected:Npn \egreg_paired_delimiter_expand:nnnn #1 #2 #3 #4
 {%
  \mathopen{}
  \mathclose\c_group_begin_token
   \left#1
   #3
   \group_insert_after:N \c_group_end_token
   \right#2
   \tl_if_empty:nF {#4} { \c_math_subscript_token {#4} }
 }
\cs_new_protected:Npn \egreg_paired_delimiter_fixed:nnnnn #1 #2 #3 #4 #5
 {
  \mathopen{#1#2}#4\mathclose{#1#3}
  \tl_if_empty:nF {#5} { \c_math_subscript_token {#5} }
 }
\ExplSyntaxOff

\XDeclarePairedDelimiter{\supnorm}{
  left=\lVert,
  right=\rVert,
  subscript=\infty
  }

\usepackage[utf8]{inputenc} %
\usepackage[T1]{fontenc}    %
\usepackage{url}            %
\usepackage{booktabs}       %
\usepackage{amsfonts}       %
\usepackage{nicefrac}       %
\usepackage{microtype}      %

\usepackage{tocloft}            %

\usepackage{makecell}

\usepackage{enumitem}

\usepackage{breakcites}

\usepackage{mathrsfs}

\usepackage{algorithm}
\usepackage{verbatim}
\usepackage[noend]{algpseudocode}

\usepackage{multicol}

\usepackage{colortbl}
\usepackage{caption}
\usepackage{subcaption}
\usepackage{setspace}

\usepackage{transparent}

\usepackage{inconsolata}
\usepackage[scaled=.90]{helvet}
\usepackage{xspace}

\usepackage{pifont}
\newcommand{\cmark}{\ding{51}}%
\newcommand{\xmark}{\ding{55}}%
\usepackage{bm}

\usepackage{tablefootnote}

\DeclareFontFamily{U}{jkpmia}{}
\DeclareFontShape{U}{jkpmia}{m}{it}{<->s*jkpmia}{}
\DeclareFontShape{U}{jkpmia}{bx}{it}{<->s*jkpbmia}{}
\DeclareMathAlphabet{\mathfrak}{U}{jkpmia}{m}{it}
\SetMathAlphabet{\mathfrak}{bold}{U}{jkpmia}{bx}{it}

\DeclarePairedDelimiter{\abs}{\lvert}{\rvert} %
\DeclarePairedDelimiter{\brk}{[}{]}
\DeclarePairedDelimiter{\crl}{\{}{\}}
\DeclarePairedDelimiter{\prn}{(}{)}
\DeclarePairedDelimiter{\nrm}{\|}{\|}
\DeclarePairedDelimiter{\tri}{\langle}{\rangle}

\let\Pr\undefined

\DeclareMathOperator{\En}{\mathbb{E}}

\DeclareMathOperator{\Pr}{Pr}

\DeclareMathOperator*{\argmin}{arg\,min} %
\DeclareMathOperator*{\argmax}{arg\,max}

\newcommand{\wt}[1]{\widetilde{#1}}
\newcommand{\wh}[1]{\widehat{#1}}
\newcommand{\wb}[1]{\widebar{#1}}

\def\ddefloop#1{\ifx\ddefloop#1\else\ddef{#1}\expandafter\ddefloop\fi}
\def\ddef#1{\expandafter\def\csname bb#1\endcsname{\ensuremath{\mathbb{#1}}}}
\ddefloop ABCDEFGHIJKLMNOPQRSTUVWXYZ\ddefloop
\def\ddefloop#1{\ifx\ddefloop#1\else\ddef{#1}\expandafter\ddefloop\fi}
\def\ddef#1{\expandafter\def\csname b#1\endcsname{\ensuremath{\mathbf{#1}}}}
\ddefloop ABCDEFGHIJKLMNOPQRSTUVWXYZ\ddefloop
\def\ddef#1{\expandafter\def\csname sf#1\endcsname{\ensuremath{\mathsf{#1}}}}
\ddefloop ABCDEFGHIJKLMNOPQRSTUVWXYZ\ddefloop
\def\ddef#1{\expandafter\def\csname c#1\endcsname{\ensuremath{\mathcal{#1}}}}
\ddefloop ABCDEFGHIJKLMNOPQRSTUVWXYZ\ddefloop
\def\ddef#1{\expandafter\def\csname h#1\endcsname{\ensuremath{\widehat{#1}}}}
\ddefloop ABCDEFGHIJKLMNOPQRSTUVWXYZ\ddefloop
\def\ddef#1{\expandafter\def\csname hc#1\endcsname{\ensuremath{\widehat{\mathcal{#1}}}}}
\ddefloop ABCDEFGHIJKLMNOPQRSTUVWXYZ\ddefloop
\def\ddef#1{\expandafter\def\csname t#1\endcsname{\ensuremath{\widetilde{#1}}}}
\ddefloop ABCDEFGHIJKLMNOPQRSTUVWXYZ\ddefloop
\def\ddef#1{\expandafter\def\csname tc#1\endcsname{\ensuremath{\widetilde{\mathcal{#1}}}}}
\ddefloop ABCDEFGHIJKLMNOPQRSTUVWXYZ\ddefloop
\def\ddefloop#1{\ifx\ddefloop#1\else\ddef{#1}\expandafter\ddefloop\fi}
\def\ddef#1{\expandafter\def\csname scr#1\endcsname{\ensuremath{\mathscr{#1}}}}
\ddefloop ABCDEFGHIJKLMNOPQRSTUVWXYZ\ddefloop

\newcommand{\ind}{\mathbbm{1}}    %

\newcommand{\veps}{\varepsilon}

\newcommand{\ldef}{\vcentcolon=}
\newcommand{\rdef}{=\vcentcolon}

\newcommand{\Shat}{\wh{S}}

\newcommand{\Astar}{A^{\pistar}}

\newcommand{\sigmab}{\wb{\sigma}}
\newcommand{\sigmabs}{\wb{\sigma}_{\pistar}}

\newcommand{\Hstar}{H^{\star}}%
\newcommand{\pitil}{\wt{\pi}}%

\newcommand{\gbar}{\wb{g}}

\newcommand{\mutil}{\wt{\mu}}

\newcommand{\sigmastar}{\sigma_{\pistar}}

\newcommand{\vepsapx}{\veps_{\texttt{apx}}}
\newcommand{\vepsopt}{\veps_{\texttt{opt}}}

\newcommand{\Drho}[2]{\rho\prn*{#1\dmid{}#2}}

\newcommand{\murec}{\mu}

\newcommand{\astari}[1][i]{a^{\star,\sss{#1}}}

\newcommand{\Qstar}{Q^{\pistar}}
\newcommand{\Vstar}{V^{\pistar}}

  \newcommand{\Lbc}{L_{\textsf{bc}}}
    \newcommand{\Lhatbc}{\wh{L}_{\textsf{bc}}}

\renewcommand{\dagger}{\texttt{Dagger}\xspace}
\newcommand{\aggrevate}{\texttt{Aggrevate}\xspace}
\newcommand{\loglossbc}{\texttt{LogLossBC}\xspace}
\newcommand{\loglossdagger}{\texttt{LogLossDagger}\xspace}

\newcommand{\gstar}{g^{\star}}

\newcommand{\pibar}{\wb{\pi}}

  \newcommand{\sfrak}{\mathfrak{s}}
  \newcommand{\afrak}{\mathfrak{a}}
  \newcommand{\bfrak}{\mathfrak{b}}
  \newcommand{\cfrak}{\mathfrak{c}}
  \newcommand{\xfrak}{\mathfrak{x}}
  \newcommand{\yfrak}{\mathfrak{y}}
  \newcommand{\zfrak}{\mathfrak{z}}

\newcommand{\var}{\mathrm{Var}}
\newcommand{\Var}{\var}

\renewcommand{\emptyset}{\varnothing}

\newcommand{\filt}{\mathscr{F}}

\newcommand{\M}[1]{^{{\scriptscriptstyle M}}}  %

\newcommand{\sups}[1]{^{{\scriptscriptstyle#1}}}

\newcommand{\sss}[1]{{\scriptscriptstyle#1}}

\newcommand{\pistar}{\pi^{\star}}

\newcommand{\pihat}{\wh{\pi}}

\newcommand{\EstOnHel}{\mathrm{\mathbf{Est}}^{\mathsf{on}}_{\mathsf{H}}}

  \newcommand{\AlgEst}{\mathrm{\mathbf{Alg}}_{\mathsf{Est}}}

\newcommand{\Mstar}{M^{\star}}

  \newcommand{\ghat}{\wh{g}}

\newcommand{\Otilde}{\wt{O}}

\newcommand{\approxleq}{\lesssim}
\newcommand{\approxgeq}{\gtrsim}

\renewcommand{\ind}[1]{^{{\scriptscriptstyle#1}}}

\newcommand{\bigoh}{O}
\newcommand{\bigoht}{\wt{O}}
\newcommand{\bigom}{\Omega}

\newcommand{\indic}{\mathbb{I}}

\renewcommand{\Pr}{\bbP}

\newcommand{\Dkl}[2]{D_{\mathsf{KL}}\prn*{#1\,\|\,#2}}

\newcommand{\Dhels}[2]{D^{2}_{\mathsf{H}}\prn*{#1,#2}}
\newcommand{\Ddel}[2]{D_{\Delta}\prn*{#1,#2}}

\newcommand{\Dchis}[2]{D_{\chi^2}\prn*{#1\dmid{}#2}}
\newcommand{\DchisX}[3]{D_{\chi^2}\prn[#1]{#2\dmid{}#3}}

\newcommand{\Dtv}[2]{D_{\mathsf{TV}}\prn*{#1,#2}}
\newcommand{\Dtvs}[2]{D^2_{\mathsf{TV}}\prn*{#1,#2}}

\newcommand{\DhelsX}[3]{D^{2}_{\mathsf{H}}\prn[#1]{#2,#3}}

\newcommand{\Ber}{\mathrm{Ber}}

\newcommand{\dmid}{\;\|\;}

\newcommand{\unif}{\mathrm{unif}}

\newcommand{\astar}{a^{\star}}

\newcommand{\mathand}{\quad\text{and}\quad}

\def\multiset#1#2{\ensuremath{\left(\kern-.3em\left(\genfrac{}{}{0pt}{}{#1}{#2}\right)\kern-.3em\right)}}

\newcommand{\iid}{i.i.d.\xspace}

\renewcommand{\emptyset}{\varnothing}

\input{widebar}

\makeatletter
\let\OldStatex\Statex
\renewcommand{\Statex}[1][3]{%
  \setlength\@tempdima{\algorithmicindent}%
  \OldStatex\hskip\dimexpr#1\@tempdima\relax}
\makeatother

 \usepackage{accents}

\let\oldparagraph\paragraph

\renewcommand{\paragraph}[1]{\oldparagraph{#1.}}

\title{Is Behavior Cloning All You Need?\\
    Understanding Horizon in Imitation Learning}

\arxiv{
\author{%
  Dylan J. Foster\\
  Microsoft Research
  \and
  Adam Block\\
  MIT
  \and
  Dipendra Misra\\
  Microsoft Research
}}

\date{}

\usepackage{autonum}

  \addtocontents{toc}{\protect\setcounter{tocdepth}{0}}

\begin{document}
\maketitle

\begin{abstract}
Imitation learning (IL) aims to mimic the behavior of an expert in a
sequential decision making task by learning from demonstrations,
and has been widely applied to robotics, autonomous driving, and autoregressive text generation. The simplest approach to IL, \emph{behavior cloning} (BC),
is thought to incur sample
complexity with unfavorable \emph{quadratic}
  dependence on the problem horizon, motivating a variety of
    different \emph{online} algorithms that attain improved
    \emph{linear} horizon dependence under stronger assumptions on the data and the learner's access to the expert.

We revisit the apparent gap between
offline and online IL from a learning-theoretic perspective,
with a focus on the realizable/well-specified setting with general policy classes up to and including deep neural networks. Through a
  new analysis of behavior cloning with the \emph{logarithmic loss}, we show that  it is possible to achieve
  \emph{horizon-independent} sample complexity in offline IL whenever
  (i) the range of the cumulative payoffs is controlled, and (ii) an
  appropriate notion of supervised learning complexity for the policy
  class is controlled. Specializing our results to deterministic,
  stationary policies, we show that the gap between offline and online
  IL is smaller than previously thought: (i) it is possible to achieve
  \emph{linear} dependence on horizon in offline IL under dense
  rewards (matching what was previously only known to be achievable in
  online IL); and (ii) without further assumptions on the policy class, online IL cannot improve over offline IL with the logarithmic loss, even in benign MDPs.  We complement our theoretical results with experiments on standard RL tasks and autoregressive language generation to validate the practical relevance of our findings.

\end{abstract}

\section{Introduction}
\label{sec:intro}

Imitation learning (IL) is the problem of emulating an expert policy for
sequential decision making by learning from demonstrations. Compared to
reinforcement learning (RL), the learner in IL does not observe
reward-based feedback, and must imitate the expert's behavior
based on demonstrations alone; their objective is to achieve performance
close to that of the expert on an \emph{unobserved} reward
function.
\arxiv{

}
Imitation learning is motivated by the observation that in
many domains, demonstrating the desired behavior for a task (e.g.,
robotic grasping) is simple, while designing a reward function to
elicit the desired behavior can be challenging. IL is also often preferable to RL because it removes the need for exploration, leading to empirically reduced sample complexity and often much more stable training.  Indeed, the relative
ease of applying IL (over RL methods) has led to extensive
adoption, ranging from classical applications in autonomous
driving \citep{pomerleau1988alvinn} and helicopter flight
\citep{abbeel2004apprenticeship} to contemporary works that leverage
deep learning to achieve state-of-the-art performance for self-driving vehicles
\citep{bojarski2016end,bansal2018chauffeurnet,hussein2017imitation},
visuomotor control \citep{finn2017one,zhang2018deep}, navigation
\citep{hussein2018deep}, and game AI \citep{ibarz2018reward,vinyals2019grandmaster}.
Imitation learning also offers a
conceptual framework through which to study autoregressive language
modeling \citep{chang2023learning,block2023butterfly}, and a number of useful empirical
insights have arisen as a result of this perspective. However, a central challenge
limiting broader real-world deployment is to understand and
improve the reliability and stability properties of algorithms that support general-purpose (deep/neural) function approximation.\loose

In more detail, imitation learning algorithms can be loosely grouped
into \emph{offline} and \emph{online} approaches. Offline imitation
learning algorithms only require access to a dataset of logged
trajectories from the expert, making them broadly applicable. The most
widely used approach, \emph{behavior cloning}, reduces imitation
learning to a standard supervised learning problem in
  which the learner attempts to predict the expert's actions from
  observations given the collected trajectories.  The simplicity of
  this approach allows the learner to leverage the considerable
  machinery developed for supervised learning and readily incorporate
  complex function approximation with deep models
  \citep{beygelzimer2005error,ross2010efficient}.  On the other hand,
  BC seemingly ignores the problem of \emph{distribution shift},
  wherein small deviations from the expert policy early in rollout
  lead the learner off-distribution to regions where they are less able to accurately imitate. 
This apparent \emph{error
  amplification} phenomenon has been widely observed
empirically
\citep{ross2010efficient,laskey2017dart,block2023butterfly}, and motivates \emph{online} or \emph{interactive} approaches to
  imitation learning
  \citep{ross2010efficient,ross2011reduction,ross2014reinforcement,sun2017deeply},
  which avoid error amplification by interactively
  querying the expert and learning to correct mistakes on-policy.\loose
  
  In
  theory, online imitation learning enables sample complexity
  guarantees with improved (linear, as opposed to quadratic)
  dependence on horizon for favorable MDPs. Yet, while online imitation learning has found
  empirical success \citep{ross2013learning,kim2013learning,gupta2017cognitive,bansal2018chauffeurnet,choudhury2018data,kelly2019hg,barnes2023world,zhuang2023robot,lum2024dextrah},
online access to the expert can be costly or
  infeasible in many applications, and offline imitation learning
  remains a dominant empirical paradigm. Motivated by
  this disconnect between theory and practice,  we aim to understand
  to what extent the apparent gap between
  offline and online imitation learning is fundamental.  
We ask:
\begin{center}
  \emph{Is online imitation learning truly more sample-efficient than
  offline imitation learning, or can existing algorithms or analyses
  be improved?}
\end{center}

\begin{figure}
  \centering
  \begin{subfigure}[t]{0.4\textwidth}
    \includegraphics[width=\textwidth]{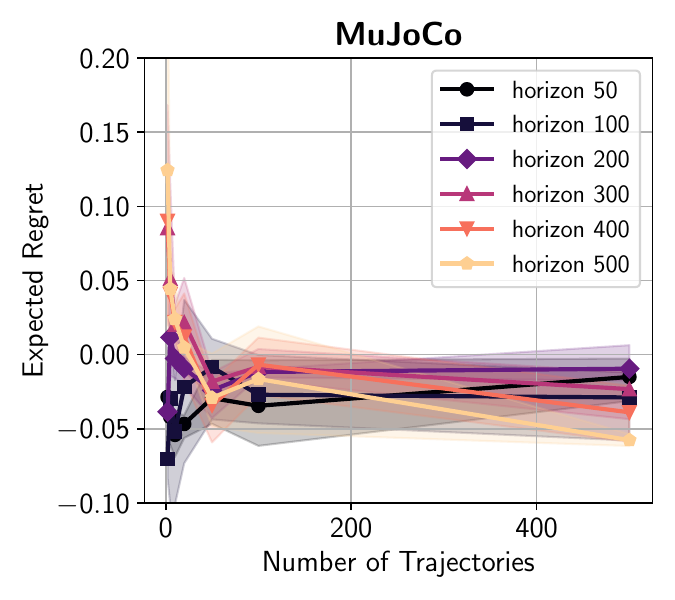}
    \caption{}
    \label{fig:mujoco}
  \end{subfigure}
\hspace{.5in}
  \begin{subfigure}[t]{0.4\textwidth}
    \includegraphics[width=\textwidth]{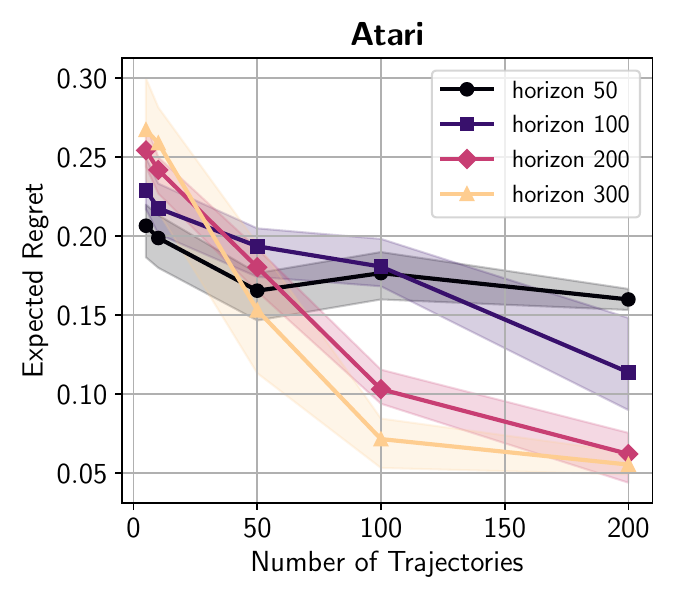}
    \caption{}
    \label{fig:atari}
  \end{subfigure}

  \caption{Suboptimality of a policy learned with log-loss behavior
    cloning (\loglossbc) as a function of the number of expert trajectories, for
    varying values of horizon $H$. In each environment, an imitator is trained
    according to \loglossbc and the regret with respect to
    the expert is reported, with reward normalized to be horizon-independent. \textbf{(a)} Continuous control with MuJoCo
    environment \textsf{Walker2d-v4}. \textbf{(b)} Discrete control with Atari
    environment \textsf{BeamRiderNoFrameskip-v4}. For both environments, we
    find that the regret is \emph{independent of horizon} (or in the
    case of Atari, slightly improving with horizon), as
    predicted by our theoretical results. Full experimental details are provided in
    \cref{sec:experiments}.\loose
  }
  \label{fig:intro}
\end{figure}

  \subsection{Background: Offline and Online Imitation Learning}
  \label{sec:background}
  
To motivate our results, we begin by formally introducing the offline and online imitation
learning frameworks, highlighting gaps in current sample complexity
guarantees concerning \emph{horizon dependence}.
We take a \emph{learning-theoretic} perspective, with a focus on general policy classes.\loose

  \paragraph{Markov decision processes}
  We study imitation learning in episodic Markov decision processes. Formally, a Markov decision process $M=(\cX, \cA, P, r,
H)$ consists of a (potentially large) state
space $\cX$, action space $\cA$, horizon $H$, probability transition function
$P=\crl{P_h}_{h=0}^{H}$, where $P_h:\cX\times\cA\to\Delta(\cX)$, and reward function
$r=\crl{r_h}_{h=1}^{H}$, where $r_h:\cX\times{}\cA\to\bbR$.
A (randomized) policy is a sequence of per-timestep functions
$\pi=\crl*{\pi_h:\cX\to\Delta(\cA)}_{h=1}^{H}$. The policy induces a
distribution over trajectories $(x_1,a_1, r_1),\ldots,(x_H,a_H,r_H)$
via the following process. The initial state is drawn via
$x_1\sim{}P_0(\emptyset)$,\footnote{We use the convention that
  $P_0(\emptyset)$ denotes the initial state distribution.} then for $h=1,\ldots,H$: $a_h\sim\pi(x_h)$,
$r_h=r_h(x_h,a_h)$, and $x_{h+1}\sim{}P_h(x_h,a_h)$. For notational
convenience, we use $x_{H+1}$ to denote a
  deterministic terminal state with zero reward. We let $\En\sups{\pi}\brk*{\cdot}$ and $\Pr\sups{\pi}\brk{\cdot}$
  denote expectation and probability law for
  $(x_1,a_1),\ldots,(x_H,a_H)$ under this process,
  respectively.\footnote{To simplify presentation, we assume that
    $\cX$ and $\cA$ are countable, but our results trivially extend
    to general spaces with an appropriate measure-theoretic treatment.}
    \loose

  The expected reward for policy $\pi$ is given by $J(\pi) \coloneqq
  \En^{\pi}\brk[\big]{\sum_{h=1}^{H}r_h}$, and the value
  functions for $\pi$ are given by %
\arxiv{\[\textstyle
    V_h^{\pi}(x)\coloneqq\En^{\pi}\brk*{\sum_{h'=h}^{H}r_{h'}\mid{}x_h=x},\mathand
    Q_h^{\pi}(x,a)\coloneqq\En^{\pi}\brk*{\sum_{h'=h}^{H}r_{h'}\mid{}x_h=x,
      a_h=a}.\]}

  \paragraph{Reward normalization}
  To study the role of horizon in imitation learning in a way that disentangles the effects
  of reward scaling from other factors, we assume that rewards are normalized such that
$\sum_{h=1}^{H}r_h\in\brk*{0,R}$ for a parameter $R>0$
\citep{jiang2018open,wang2020long,zhang2021reinforcement,jin2021bellman}. We
refer to the setting in which $r_h\in\brk{0,1}$ for all $h\in\brk{H}$,
which is the focus of most prior work
\citep{ross2010efficient,ross2011reduction,ross2014reinforcement,
  rajaraman2020toward,rajaraman2021value,rajaraman2021provably,swamy2022minimax}, 
as the \emph{dense
  reward setting}, which has $R\leq{}H$; we will frequently specialize
our results to this setting.\loose

\subsubsection{Offline Imitation Learning: Behavior Cloning}
\label{sec:bc_indicator}
Let $\pistar=\crl*{\pistar_h:\cX\to\Delta(\cA)}_{h=1}^{H}$ denote the
\emph{expert policy}. In the offline imitation learning setting, we
receive a dataset $\cD=\crl*{o\ind{i}}_{i=1}^{n}$ of (reward-free)
trajectories
$o\ind{i}=(x_1\ind{i},a_1\ind{i}),\ldots,(x_H\ind{i},a_H\ind{i})$
obtained by executing $\pistar$ in the underlying MDP $\Mstar$. Using
these trajectories, our goal is to learn a policy $\pihat$ such that
the rollout regret
$J(\pistar) -J(\pihat)$ to $\pistar$ is as small as possible. \emph{We emphasize that $\pistar$ is
an arbitrary policy, and is not assumed to be optimal.}\loose

\paragraph{Behavior cloning}
\emph{Behavior cloning}, which reduces
the imitation learning problem to supervised prediction, is the
dominant offline imitation learning paradigm. To describe
the algorithm in its simplest form, consider the case where
$\pistar: =\crl*{\pistar_h:\cX\to\cA}_{h=1}^{H}$ is
deterministic. For a user-specified policy class $\Pi\subset
\crl*{\pi_h:\cX\to\Delta(\cA)}_{h=1}^{H}$, the most basic version
of behavior cloning \citep{ross2010efficient} solves the
supervised classification problem
\arxiv{\begin{align}
  \label{eq:bc}
  \pihat=\argmin_{\pi\in\Pi}
  \underbrace{\sum_{i=1}^{n}\frac{1}{H}\sum_{h=1}^{H}\indic\crl*{\pi_h(x_h\ind{i})\neq{}a_h\ind{i}}}_{\rdef{}\Lhatbc(\pi)}.
\end{align}}
Naturally, other classification losses (e.g., square loss, logistic
loss, or log loss) may be used in place of the indicator
loss.\footnote{Behavior cloning for stochastic expert policies has
  received limited attention in theory \citep{rajaraman2020toward}, but the logarithmic loss is
  widely used in practice. One contribution of our work is to fill
  this lacuna.}
To provide sample complexity bounds for this algorithm, we make a
standard \emph{realizability assumption} (e.g., \citet{agarwal2019reinforcement,foster2023foundations}).
\begin{assumption}[Realizability]
  \label{ass:realizability}
  The policy class $\Pi$ contains the expert policy, i.e. $\pistar\in\Pi$.
\end{assumption}
This assumption asserts that $\Pi$ is expressive enough to represent
the expert policy. Depending on the application, $\Pi$ might be
parameterized by simple linear models, or by flexible models such as
convolutional neural networks or transformers. We primarily restrict our
attention to the realizable setting throughout the paper, as it is
meaningful and non-trivial, yet not fully understood. Our main results
extend to provide guarantees for the misspecified case, but
a thorough study of the role of misspecification is beyond the
scope of this work.\loose
To simplify
  presentation, we adopt a
  standard convention in RL theory and focus on finite classes with $\abs{\Pi}<\infty$
\citep{agarwal2019reinforcement,foster2023foundations}.

To proceed with analyzing \arxiv{\cref{eq:bc}}, a standard uniform convergence argument implies that if we define
$\Lbc(\pi)=\frac{1}{H}\sum_{h=1}^{H}\bbP^{\pistar}\brk*{\pi_h(x_h)\neq{}\pistar_h(x_h)}$,
then with probability at least $1-\delta$, behavior cloning has
\begin{align}
  \label{eq:lbc_generalization}
  \Lbc(\pihat) \approxleq{} \frac{\log(\abs{\Pi}\delta^{-1})}{n}.
\end{align}
Meanwhile, a standard error analysis for BC leads to the following
bound on rollout performance:
\begin{align}
    \label{eq:lbc_rollout}
  J(\pistar) - J(\pihat) \approxleq{} RH\cdot\Lbc(\pihat).
\end{align}
Combining these bounds, we conclude that
\begin{align}
  \label{eq:lbc_regret}
  J(\pistar) - J(\pihat) \approxleq{} RH \cdot{}\frac{\log(\abs{\Pi}\delta^{-1})}{n}.
\end{align}
For the \emph{dense reward setting} where $R=H$, this leads to
\emph{quadratic} dependence on horizon; that is, $\Omega(H^2)$
trajectories are required to achieve constant accuracy. Unfortunately,
both steps in this argument are tight in general:\loose
\begin{itemize}
\item The generalization bound $\Lbc(\pihat) \approxleq{}
  \frac{\log(\abs{\Pi}\delta^{-1})}{n}$ is tight even when
  $\abs{\Pi}=2$ (this is true not just for the indicator loss, but for
  other standard losses such as square loss, absolute loss, and hinge
  loss).
  Since the
    amount of information in a trajectory grows with $H$, one might
    hope a-priori that the generalization error would decrease
    with $H$; alas, this does not occur due to the
    \emph{dependence} between samples in each trajectory.\loose
\item \citet{ross2010efficient} show that the inequality $J(\pistar)
  - J(\pihat) \approxleq{} RH\cdot\Lbc(\pihat)$ is tight for MDPs
  with $3$ states; the quadratic scaling in $H$ this induces under dense rewards is often attributed to
  \emph{error amplification} or \emph{distribution shift} incurred by
  passing from error under
  the state distribution of $\pistar$ to the state distribution of $\pihat$.
\end{itemize}
Combining, these observations, %
it is natural to conclude that
offline imitation learning is fundamentally harder than supervised
classification, where linear dependence on horizon might be expected (e.g., \arxiv{if we
considered} $H$ independent prediction tasks).

\subsubsection{Online Imitation Learning and Recoverability}
The aforementioned limitations of behavior cloning
have motivated \emph{online} approaches to IL \citep{ross2010efficient,ross2011reduction,ross2014reinforcement,sun2017deeply}. In
the online framework, \iftoggle{neurips}{the learner can
    interactively choose policies to roll out and query the expert for the
    action at each state in the trajectory (see \cref{sec:dagger} for
    a formal description), representing a
    substantially stronger (and in some cases unrealistic) assumption on the learner's access both to the MDP and the expert than in the offline setting.}{learning proceeds in $n$ episodes in which the
learner can directly interact with the underlying MDP $\Mstar$ and
query the expert advice. Concretely, for each episode $i\in\brk{n}$,
the learner executes a policy $\pi\ind{i}=\crl*{\pi\ind{i}_h:\cX\to\Delta(\cA)}_{h=1}^{H}$ and
      receives a trajectory $o\ind{t} =
      (x\ind{i}_1,a\ind{i}_1,\astari_1),\ldots,(x\ind{i}_H,a\ind{i}_H,\astari_H)$,
      in which $a_h\ind{i}\sim{}\pi_h\ind{i}(x_h\ind{i})$,
      $\astari_h\sim\pistar(x\ind{t}_h)$, and
      $x_{h+1}\ind{i}\sim{}P_h(x_h\ind{i},a_h\ind{i})$; in other
      words, the trajectory induced by the learner's policy is
      annotated by the expert's action $\astar_h\sim{}\pistar_h(x_h)$
      at each state $x_h$ encountered.\footnote{All of the lower
        bounds in this paper continue to hold when the learner is
        allowed to select $a_h\ind{i}$ based on the sequence
        $(x\ind{i}_1,a\ind{i}_1,\astari_1),\ldots,(x\ind{i}_{h-1},a\ind{i}_{h-1},\astari_{h-1}),
        (x_h\ind{i}, \astari_h)$ at training time; we adopt the
        present formulation to keep notation compact.}
      After all $n$ episodes conclude, the learner produces a final
      policy $\pihat$ whose regret to $\pistar$ should be small.}
      Online imitation learning can avoid error amplification and achieve improved dependence on horizon for MDPs that
      satisfy a \emph{recoverability} condition
      \citep{ross2011reduction,rajaraman2021value}.\loose
      \begin{definition}[Recoverability parameter]
        \label{def:recoverability}
        The \emph{recoverability parameter} for an MDP $\Mstar$ and expert
        $\pistar$ is given by\footnote{For stochastic
          policies, we overload notation and write $f(\pi(x))$ as
          shorthand for $\En_{a\sim{}\pi(x)}\brk*{f(a)}$.}
        \arxiv{\[\murec =
            \max_{x\in\cX,a\in\cA,h\in\brk{H}}\crl*{(\Qstar_h(x,\pistar_h(x))-\Qstar_h(x,a))_{+}}\in\brk{0,R}.\]}
      \end{definition}
Under recoverability, the \dagger algorithm of
\citet{ross2011reduction} leverages online interaction by interactively
      querying the expert and learning to correct mistakes on-policy,
      leading to sample complexity\loose
      \begin{align}
        \label{eq:dagger_finite}
  J(\pistar) - J(\pihat) \approxleq{} \murec{}H\cdot\frac{\log\abs{\Pi}}{n}
\end{align}
for any finite class $\Pi$ and deterministic expert policy $\pistar$\arxiv{, when configured appropriately} (for
completeness, we include an analysis in \cref{sec:dagger}; see
\cref{prop:dagger,prop:dagger_finite}).
\arxiv{
  
}
For the dense reward setting where $R=H$, we can have $\mu=H$ in the
worst case, in which case \cref{eq:dagger_finite} matches
the quadratic horizon dependence of behavior cloning, but when
$\mu=\bigoh(1)$ (informally, this means it is possible to ``recover''
from a bad action that deviates from $\pistar$), the bound in
\cref{eq:dagger_finite} achieves linear dependence on horizon. Other
online IL algorithms such as \texttt{Forward}, \texttt{Smile} \citep{ross2010efficient}, and \aggrevate
\citep{ross2014reinforcement} achieve similar guarantees (we are not
aware of an approach that improves upon \cref{eq:dagger_finite} for
general finite classes).

The improvements of online IL notwithstanding, \cref{eq:lbc_regret} is known
  to be tight for BC, but this is an \emph{algorithm-dependent} (as
  opposed to information-theoretic) lower bound, and does not preclude the
  existence of more sample-efficient, purely offline algorithms. In this context,
  our central question can be restated as: \emph{Can offline imitation
    learning algorithms achieve sub-quadratic horizon dependence for
    general policy classes $\Pi$?} While prior work has investigated this question for tabular and linear policies
  \citep{rajaraman2020toward,rajaraman2021value,rajaraman2021provably},
  we approach the problem from a new (learning-theoretic) perspective by considering general policy classes.

\subsection{Contributions}

\newcommand{\No}{\xmark}
\newcommand{\Yes}{\cmark}
\begin{table}[tp]
	\renewcommand{\arraystretch}{1.6}
	\fontsize{9}{10}\selectfont
	\centering 
	\begin{tabular}{c|cc}
		\hline
          & \makecell{Parameter Sharing\\(\cref{cor:bc_deterministic_finite})}%
          & 
                                                                  \makecell{No
            Parameter
                                                                  Sharing
            ($\Pi=\Pi_1\times\cdots\Pi_H$)\\(e.g., \citep{ross2010efficient})}\\
		
		\hline
        Sparse Rewards & $\bigoh\left( \frac{R \log\left( |\Pi| \right)}{n} \right)$ & $\bigoh\left( \frac{H R \log\left( \max_h|\Pi_h| \right)}{n} \right)$ \\
        Dense Rewards ($R=H$) & $\bigoh\left( \frac{H \log\left( |\Pi|
                                \right)}{n} \right)$  & $\bigoh\left(
                                                        \frac{H^2
                                                        \log\left(
                                                        \max_h|\Pi_h|
                                                        \right)}{n}
                                                        \right)$\\
		\hline
	\end{tabular}
        	\caption{Summary of upper bounds for deterministic experts;
          lower bounds are more nuanced, and discussed in
          \cref{sec:deterministic}.  Each cell denotes the regret of a
          policy learned with log-loss behavior cloning (\loglossbc), which is optimal in
          each setting.  Here, $\Pi$ is the policy class, $R$ is the
          reward range, $H$ is the horizon, and $n$ is the number of
          expert trajectories. In the dense-reward setting, we set $R=H$.\loose
	}
	\label{tb:results_deterministic}
\end{table}

We present several new results that clarify the role of horizon in
offline and online imitation learning. %
\begin{enumerate}[leftmargin=*]
\item \textbf{Horizon-independent analysis of log-loss behavior cloning.}
    Through a new analysis of behavior cloning with the
  \emph{logarithmic loss} (\loglossbc), we show that \textbf{it is
    possible to achieve \emph{horizon-independent} sample complexity}
  \citep{jiang2018open,wang2020long,zhang2021reinforcement,zhang2022horizon}
  in offline
  imitation learning whenever (i) the range of the cumulative payoffs is
    normalized, and (ii) an appropriate notion of supervised learning
    complexity for the policy class is controlled. Our result is
    facilitated by a novel information-theoretic analysis which
    controls policy behavior at the trajectory level, supporting both
    deterministic and stochastic expert policies.\loose
  \item \textbf{Deterministic policies: Closing the gap between offline
      and online IL.}
    Specializing \loglossbc to \emph{deterministic stationary} policies (more generally, policies with
      parameter sharing) and cumulative rewards
      in the range $\brk*{0,H}$, we show that it is possible to
      achieve sample complexity with \emph{linear} dependence on
      horizon in offline IL in arbitrary MDPs, matching was was previously only known of \emph{online} IL. %
    We complement this result with a
      lower bound showing that, without further structural assumptions
      on the
    policy class (e.g., no parameter sharing 
\citep{rajaraman2020toward}), \textbf{online IL cannot
    improve over offline IL with \loglossbc},
  even for benign MDPs.  Our results are summarized in \cref{tb:results_deterministic}.
Nonetheless, as observed in prior work
    \citep{rajaraman2020toward}, online imitation
    learning can still be beneficial for \emph{non-stationary}
    policies.%
  \item \textbf{Stochastic policies: Tight understanding of optimal
      sample complexity.}
    For stochastic expert policies, our analysis of \loglossbc gives
    the first \emph{variance-dependent} sample complexity bounds for
    imitation learning with general policy classes, which we prove to be tight in a
    problem-dependent and minimax sense.  Using this result, we show
    that for stochastic stationary experts, (i) \emph{quadratic dependence on the horizon is necessary} when
    cumulative rewards lie in the range $\brk*{0,H}$, in contrast to the
    deterministic setting, but (ii) \loglossbc---through our
    variance-dependent analysis---can sidestep this hardness and achieve linear dependence on horizon under a
    recoverability-like condition. Finally, we show that---as in the
    deterministic case---online IL cannot improve over offline
    IL with \loglossbc without further assumptions on the policy class.  Our results are summarized in \cref{tb:results_stochastic}.

  \end{enumerate}

    \paragraph{Toward a learning-theoretic understanding of imitation learning}
    Our findings highlight the need to
  develop a fine-grained, problem-dependent understanding of
  algorithms and complexity for IL. Instabilities of
  offline IL \citep{muller2005off,de2019causal,block2023butterfly} and
  benefits of online IL
  \citep{ross2013learning,kim2013learning,gupta2017cognitive,bansal2018chauffeurnet,choudhury2018data,kelly2019hg,barnes2023world,zhuang2023robot,lum2024dextrah}
  likely arise in practice, but
  existing assumptions in theoretical research are often too coarse to give insights into the true nature of these
  phenomena, leading to an important gap between theory and practice.
  As a first step in this direction, we highlight several under-explored mechanisms
  through which online IL can lead to improved sample complexity,
  including representational benefits and exploration
  (\cref{sec:online}). We also complement our theoretical
    results with empirical demonstrations of the phenomena
    we describe (\cref{sec:experiments}).

  \paragraph{Experiments}  In \cref{sec:experiments}, we complement
  our theoretical results with an empirical demonstration of the
  horizon-independence of \loglossbc predicted by our theory (under
  parameter sharing and sparse rewards).  We consider tasks where the
  horizon $H$ can be naturally scaled up and down---for example, an
  agent walking for a set number of timesteps---and use an expert
  trained according to RL to generate expert trajectories, before
  training a policy using \loglossbc.  We consider both continuous
  action space (MuJoCo environment \textsf{Walker2d}) and discrete
  action space (Atari environment \textsf{Beamrider}) tasks to demonstrate the broad applicability of our theoretical results. As can be seen in
  \Cref{fig:intro}, the performance of the learned policy is
  independent or improving with horizon, consistent with our theoretical
  results. We also perform simplified experiments on autoregressive
  language generation with transformers. Here, we find that the performance of the imitator is
  largely independent of $H$, as predicted by our results, though the results are more nuanced.  \loose

  \iftoggle{neurips}{}{
  \subsection{Paper Organization}
  \cref{sec:main} presents the first of our
  main results, a horizon-independent sample complexity analysis for
  \loglossbc for deterministic experts, and discusses implications regarding the gap between
  offline and online IL as it concerns horizon. \cref{sec:stochastic}
  presents analogous results and implications for stochastic experts. \cref{sec:online}
  discusses mechanisms through which online IL can have benefits over
  offline IL, beyond horizon dependence, highlighting directions for
  future research.
  \cref{sec:experiments}
  presents an empirical validation, and we conclude with open problems
  and further directions for future research in
  \cref{sec:discussion}.
  Proofs and additional results are deferred to
  the appendix.\loose
}

  \paragraph{Notation}
\arxiv{For an integer $n\in\bbN$, we let $[n]$ denote the set
  $\{1,\dots,n\}$. For a set $\cX$, we let
        $\Delta(\cX)$ denote the set of all probability distributions
        over $\cX$. }We use
        $\indic_x\in\Delta(\cX)$ to denote the direct delta
        distribution, which places probability mass $1$ on $x$.
      We adopt standard
        big-oh notation, and write $f=\bigoht(g)$ to denote that $f =
        \bigoh(g\cdot{}\max\crl*{1,\mathrm{polylog}(g)})$ and
        $a\approxleq{}b$ as shorthand for $a=\bigoh(b)$.

\section{\mbox{Horizon-Independent Analysis of Log-Loss Behavior Cloning}}
  \label{sec:main}
  
This section presents the first of our main results, a
  horizon-independent sample complexity analysis of log-loss behavior
  cloning for the case of deterministic experts. Our second main
  result, which handles the case of stochastic experts, builds on our results
  here and is presented in \cref{sec:stochastic}.

\subsection{Log-Loss Behavior Cloning and Supervised Learning Guarantees}

The workhorse for all of our results (both for deterministic and
stochastic experts) is the following simple modification to behavior
cloning. For a class of (potentially stochastic) policies $\Pi$, we minimize the \emph{logarithmic loss}:\loose
\begin{align}
  \label{eq:log_loss_bc}
  \pihat=\argmin_{\pi\in\Pi}\sum_{i=1}^{n}\sum_{h=1}^{H}\log\prn*{\frac{1}{\pi_h(a_h\ind{i}\mid{}x_h\ind{i})}}.
\end{align}
This scheme is ubiquitous in practice
\citep{hussein2018deep,florence2022implicit}, and forms the basis for
autoregressive language modeling \citep{radford2019language}; we refer
to it as \loglossbc. We will show that this seemingly small change---moving from indicator loss to log loss---has significant benefits.\iftoggle{neurips}{}{\footnote{Beginning from \citet{foster2021efficient}, a recent line of work
  \citep{wang2023benefits,wang2024more,ayoub2024switching} shows that
  the logarithmic loss can be beneficial for deriving
  problem-dependent bounds for various reinforcement learning
  settings. We build upon the information-theoretic machinery of
  \citet{foster2021efficient,foster2021statistical}, but use it show
  that for imitation learning, the log-loss is beneficial
  even in a minimax sense.}}%
\arxiv{

}
Following the classical tradition of imitation learning
\citep{ross2010efficient,ross2011reduction,ross2014reinforcement}, our
analysis proceeds via \emph{reduction} to supervised learning. We
first show that \loglossbc satisfies an appropriate supervised
learning guarantee, then translate this into rollout performance. Our
starting point is to observe that \loglossbc, via
\cref{eq:log_loss_bc}, can be interpreted as performing maximum
likelihood estimation over the set $\crl*{\bbP^{\pi}}_{\pi\in\Pi}$ in
order to estimate the law $\bbP^{\pistar}$ over trajectories under $\pistar$
(see \cref{sec:bc_examples} for details). As a result, standard
guarantees for maximum likelihood estimation
\citep{Sara00,zhang2006from} imply convergence in distribution whenever
$\pistar\in\Pi$. To be precise, define the squared \emph{Hellinger distance} for probability measures $\bbP$ and $\bbQ$ %
\arxiv{with a common
dominating measure $\omega$ by
\begin{equation}
  \label{eq:hellinger}
  \Dhels{\bbP}{\bbQ}=\int\prn[\bigg]{\sqrt{\frac{\mathrm{d}\bbP}{\mathrm{d}\omega}}-\sqrt{\frac{\mathrm{d}\bbQ}{\mathrm{d}\omega}}}^{2}\mathrm{d}\omega.
\end{equation}
} Then for any finite policy class $\Pi$, we have the following
guarantee.\footnote{While unfamiliar readers might expect a bound on
  KL divergence, Hellinger distance turns out to be more natural due
  to a connection to the MGF of the log-loss
  \citep{Sara00,zhang2006from}. This facilitates scale-free
  generalization guarantees in spite of the potential unboundedness of
  the log-loss.
  .}
\loose
    \begin{proposition}[Supervised learning guarantee for \loglossbc (special case of \cref{thm:bc_generalization})]
      \label{prop:mle_finite}
      For any (potentially stochastic) expert $\pistar\in\Pi$, the
      \loglossbc algorithm \arxiv{in \cref{eq:log_loss_bc} }ensures that with
      probability at least $1-\delta$,\loose %
\arxiv{\begin{align}
  \label{eq:mle_hellinger}
  \DhelsX{\big}{\bbP^{\pihat}}{\bbP^{\pistar}} \leq
  2\frac{\log(\abs{\Pi}\delta^{-1})}{n}.
\end{align}}
\end{proposition}
That is, by performing \loglossbc, we are implicitly estimating the
law $\bbP^{\pistar}$; note that this result holds even if $\pistar$ is
stochastic, as long as $\pistar\in\Pi$. We will focus on finite, realizable policy
classes throughout this section to simplify presentation as much as
possible, but guarantees for infinite classes under misspecification
are given in \cref{sec:bc_examples}.

    \subsection{Horizon-Independent Analysis of \loglossbc for Deterministic
      Experts}
    \label{sec:deterministic}

    We first consider the case where the expert $\pistar$ is
    deterministic. Our main result is the following theorem, which
    translates the supervised learning error
    $\DhelsX{\big}{\bbP^{\pihat}}{\bbP^{\pistar}}$ into a bound on
    rollout performance in a horizon-independent fashion.
\begin{theorem}[Horizon-independent regret decomposition
  (deterministic case)]
  \label{thm:bc_deterministic}
  For any deterministic policy $\pistar$ and potentially stochastic
  policy $\pihat$,
  \begin{align}
    \label{eq:bc_deterministic}
      J(\pistar)-J(\pihat)
    \leq
    4R\cdot{}\DhelsX{\big}{\bbP^{\pihat}}{\bbP^{\pistar}}.
  \end{align}
\end{theorem}

This result shows that horizon-independent bounds on rollout
performance are possible whenever (i) rewards are appropriately
normalized, and (ii) the supervised learning error
$\DhelsX{\big}{\bbP^{\pihat}}{\bbP^{\pistar}}$ is appropriately
controlled. It is proven using novel trajectory-level control over
deviations between $\pihat$ and $\pistar$; we will elaborate upon this
in the sequel. We emphasize that this result would be trivial
  if squared Hellinger distance were replaced by total variation
  distance in \eqref{eq:bc_deterministic}; that the bound scales with \emph{squared} Hellinger
  distance is crucial for obtaining fast $1/n$-type rates and linear
  horizon dependence.  We further remark that this reduction is not
  specific to \loglossbc, and can be applied to any IL algorithm for
  which we can bound the Hellinger distance. Combining \cref{thm:bc_deterministic} with \cref{prop:mle_finite},
    we obtain the following guarantee for finite policy classes.
    \begin{corollary}[Regret of \loglossbc \arxiv{(deterministic case)}]
      \label{cor:bc_deterministic_finite}
            For any deterministic expert $\pistar\in\Pi$, the
      \loglossbc algorithm in \cref{eq:log_loss_bc} ensures that with
      probability at least $1-\delta$, \iftoggle{neurips}{$J(\pistar)-J(\pihat) \leq{} 8R\cdot{}\frac{\log(2\abs{\Pi}\delta^{-1})}{n}$.\loose}{
      \begin{align}
        \label{eq:bc_df}
        J(\pistar)-J(\pihat) \leq{} 8R\cdot{}\frac{\log(2\abs{\Pi}\delta^{-1})}{n}.
      \end{align}
      }
    \end{corollary}
To the best of our knowledge, this is the tightest available sample
complexity guarantee for offline imitation learning with general
policy classes. This bound improves upon the guarantee for indicator-loss behavior
cloning in \cref{eq:lbc_regret} by an $O(H)$ factor, and improves upon
the guarantee for \dagger in \cref{eq:dagger_finite} (replacing $H$
with $R\leq{}H$ under $r_h\in\brk{0,1}$) in the typical regime where
$\mu=\bigom(1)$.

\subsection{Interpreting the Sample Complexity of \loglossbc}

To understand the behavior of the bound for \loglossbc in
\cref{cor:bc_deterministic_finite} in more detail, we
consider two special cases (summarized in \cref{tb:results_deterministic}).
\paragraph{Stationary policies and parameter sharing}
If $\log\abs{\Pi}=\bigoh(1)$, the bound in \iftoggle{neurips}{\cref{cor:bc_deterministic_finite}}{\cref{eq:bc_df}} is
  \emph{independent of horizon} in the case of sparse rewards
  ($R=\bigoh(1)$), and \emph{linear in horizon} in the case of dense
  rewards ($R=\bigoh(H)$). In other words, our work establishes for
  the first time that:
  \begin{center}
    \emph{$O(H)$ sample complexity can be achieved in offline IL under
      dense rewards for general classes $\Pi$,}
  \end{center}
  as long as
  $\log\abs{\Pi}$ is appropriately controlled and
    realizability holds.
  This runs somewhat
  counter to intuition expressed in prior work
  \citep{ross2010efficient,ross2011reduction,ross2014reinforcement,rajaraman2020toward,rajaraman2021value,rajaraman2021provably,swamy2022minimax},
  but we will show in the sequel that there is no contradiction.

  Generally speaking, we expect to have
  $\log\abs{\Pi}=\bigoh(1)$ if
  $\Pi$ consists of stationary policies or more broadly, policies
  with parameter sharing across steps $h\in\brk{H}$ (as is the case in transformers used for autoregressive text generation). As an example, for
  a tabular (finite state/action) MDP, if $\Pi$ consists of all
  stationary policies, we have
  $\log\abs{\Pi}=\abs{\cX}\log\abs{\cA}$, so \iftoggle{neurips}{\cref{cor:bc_deterministic_finite}}{\cref{eq:bc_df}} gives $
  J(\pistar)-J(\pihat) \approxleq{}
  \frac{R\abs{\cX}\log(\abs{\cA}\delta^{-1})}{n}$; that is,
  stationary policies can be learned with horizon-independent
  samples complexity under sparse rewards and linear dependence on
  horizon under dense rewards. \iftoggle{neurips}{Similar
      behavior holds for non-stationary policies with parameter
      sharing (e.g., log-linear policies of the
  form $\pi_h(a\mid{}x)\propto\exp(\tri*{\phi_h(x,a),\theta})$); see
  \cref{sec:bc_examples} for details.}{

  Similar
  behavior holds for non-stationary policies with parameter
  sharing. For example, we show (\cref{sec:bc_examples}) that for linear policy classes of the
  form $\pi_h(a\mid{}x)\propto\exp(\tri*{\phi_h(x,a),\theta})$ for a
  feature map $\phi_h(x,a)\in\bbR^{d}$, one can take
  $\DhelsX{\big}{\bbP^{\pihat}}{\bbP^{\pistar}}=\bigoht(\frac{d}{n})$,
  so that \cref{thm:bc_deterministic} gives $J(\pistar)-J(\pihat) \leq{}
  \bigoht(\frac{Rd}{n})$.
  }

  \arxiv{\paragraph{Non-stationary policies or no parameter sharing}}
For non-stationary policies or policies with no parameter sharing
across steps $h$ (e.g., product classes where
$\Pi=\Pi_1\times\Pi_2\cdots\times\Pi_H$), we expect
$\log\abs{\Pi}=\bigoh(H)$ (more generally,
$\DhelsX{\big}{\bbP^{\pihat}}{\bbP^{\pistar}}=\bigoht(H/n)$). For
example, in a tabular MDP, if $\Pi$ consists of all
non-stationary policies, we have
  $\log\abs{\Pi}=H\abs{\cX}\log\abs{\cA}$. In this case,
  \iftoggle{neurips}{\cref{cor:bc_deterministic_finite}}{\cref{eq:bc_df}} gives linear dependence on horizon for
  sparse rewards ($J(\pistar)-J(\pihat) \approxleq{}
  \frac{RH\abs{\cX}\log(\abs{\cA}\delta^{-1})}{n}$) and
  quadratic dependence on horizon for dense rewards ($J(\pistar)-J(\pihat) \approxleq{}
  \frac{H^2\abs{\cX}\log(\abs{\cA}\delta^{-1})}{n}$). The latter bound
  is known to be optimal \citep{rajaraman2020toward} for offline IL.\loose

    \subsection{Optimality and Consequences for Online versus
      Offline Imitation Learning}
We now investigate the optimality of \cref{thm:bc_deterministic} and
discuss implications for online versus offline imitation learning, as
well as connections to prior work. Our main result here shows that in
the dense-reward regime where $r_h\in\brk{0,1}$ and $R=H$,
\cref{thm:bc_deterministic} cannot be improved when
$\log\abs{\Pi}=\bigoh(1)$---even with online access, recoverability,
and known dynamics.
    \begin{theorem}[Lower bound for deterministic experts]
      \label{prop:lb_deterministic}
        For any $n\in\bbN$ and $H\in\bbN$, there exists a (reward-free)
        MDP $\Mstar$ with $\abs{\cX}=\abs{\cA}=2$, a class of reward functions $\cR$ with
        $\abs*{\cR}=2$, and a class of
  deterministic policies $\Pi$ with $\abs{\Pi}=2$ with the following
  property. For any (online or offline) imitation learning algorithm, there exists a deterministic reward
  function $r=\crl*{r_h}_{h=1}^{H}$ with $r_h\in\brk{0,1}$ (in
  particular, $R\leq{}H$) and
  (optimal) expert policy $\pistar\in\Pi$
  with $\murec=1$ such that\arxiv{ the expected suboptimality is lower bounded as}
  \arxiv{\begin{align}
    \En\brk*{J(\pistar)-J(\pihat)}
    \geq{} c\cdot\frac{H}{n}
         \end{align}}
  for an absolute constant $c>0$.
  In addition, the dynamics, rewards, and expert policies are 
  stationary.\loose
\end{theorem}
Together, \cref{thm:bc_deterministic,prop:lb_deterministic} show that
without further assumptions on $\Pi$, \emph{online imitation learning cannot
  improve upon offline imitation learning} in the realizable setting. That is, even if
recoverability is satisfied, there is no online imitation learning
algorithm that improves upon \cref{thm:bc_deterministic} uniformly for
all policy classes. See \cref{sec:lb_weaker} for further lower bounds.

\paragraph{Benefits of online IL for policies with no parameter sharing}
How can we reconcile our results with prior work showing that that online IL improves the
horizon dependence of offline IL \citep{ross2010efficient,ross2011reduction,ross2014reinforcement,rajaraman2020toward,rajaraman2021value,rajaraman2021provably,swamy2022minimax}?
The important distinction here is that online IL can still improve on a
\emph{policy-class dependent} basis. In particular, methods like
\dagger can still lead to improved sample complexity for policy
classes with \emph{no parameter sharing} across steps
$h\in\brk{H}$. Let $\Pi_h\ldef{}\crl*{\pi_h\mid{}\pi\in\Pi}$ denote
the projection of $\Pi$ onto step $h$. In \cref{sec:dagger}, we prove
the following refined guarantee for a variant of \dagger based on the
log-loss (\loglossdagger).
\begin{proposition}[Special case of \cref{prop:dagger_finite}]
  \label{prop:dagger_det_body}
  When $\pistar\in\Pi$ is deterministic, \loglossdagger ensures that
  with probability at least $1-\delta$, %
  \arxiv{\begin{align}
    \label{eq:dagger_det_body}
    J(\pistar)-J(\pihat) \approxleq{}  \mu\cdot{}\sum_{h=1}^{H}\frac{\log(\abs{\Pi_h}{}H\delta^{-1})}{n}.
  \end{align}}
\end{proposition}
For classes with no parameter sharing (i.e., product classes where
$\Pi=\Pi_1\times\Pi_2\cdots\times\Pi_H$), we have
$\sum_{h=1}^{H}\log\abs{\Pi_h}=\log\abs{\Pi}$. In this case,
\cref{prop:dagger_det_body} scales as $J(\pistar)-J(\pihat)
\approxleq{}  \mu\cdot{}\frac{\log(\abs{\Pi}{}H\delta^{-1})}{n}$, improving on the bound for \loglossbc in
\cref{thm:bc_deterministic} by replacing $R$ with 
$\mu\leq{}R$. Thus, online \iftoggle{neurips}{IL}{imitation learning} can indeed
  improve over offline IL for classes with no parameter sharing. This
is consistent with \citet{rajaraman2020toward,rajaraman2021value}, who
proved a $\mu{}H$ vs. $H^2$ gap between online and offline IL for the
special case of non-stationary tabular policies (where $\Pi$ is a
product class with $\log\abs{\Pi}\propto{}H$) under dense rewards. However, for classes
with parameter sharing (i.e., where
$\log\abs{\Pi_h}\propto\log\abs{\Pi}$), the bound in
\cref{prop:dagger_det_body} scales as $\frac{\mu{}H\log\abs{\Pi}}{n}$, which does
not improve over \cref{thm:bc_deterministic} unless $\mu\ll{}1$. Since
virtually all empirical work on imitation learning uses parameter
sharing across steps $h\in\brk{H}$, we believe the finding that online
IL does not improve over offline IL in this regime is quite
salient. Nevertheless, it is important to emphasize that there
  are various practical considerations (e.g., misspecification or
  geometric structure) which this result may not account for.%
\iftoggle{neurips}{\footnote{Complementary to our results, various works show improved horizon
  dependence in the \emph{inverse RL} setup where either (i) the
  MDP dynamics are known, or (ii) the learner can interact with the
  MDP online, but cannot interact with the expert itself 
\citep{rajaraman2020toward,swamy2021moments}; see \cref{sec:related}\arxiv{ for discussion}.}}{
\begin{remark}[Known dynamics and inverse RL]
  Complementary to our results, various works show improved horizon
  dependence in the \emph{inverse RL} setup where either (i) the
  MDP dynamics are known, or (ii) the learner can interact with the
  MDP online, but cannot interact with the expert itself
  \citep{rajaraman2020toward,swamy2021moments}; see
  \cref{sec:related}\arxiv{ for additional discussion}.
\end{remark}
}

\subsection{Proving \creftitle{thm:bc_deterministic}: How Does \loglossbc Avoid
  Error Amplification?}

The central object in the proof of \cref{thm:bc_deterministic} is
the following \emph{trajectory-level} distance function between policies. For a pair of potentially stochastic policies $\pi$ and $\pi'$, define
\begin{align}
  \label{eq:traj_metric_body}
    \Drho{\pi}{\pi'}\ldef \En^{\pi}\En_{a'_{1:H}\sim\pi'(x_{1:H})}\brk*{\indic\crl*{\exists{}h:\,a_h\neq{}a'_h}},
\end{align}
where we use the shorthand $a'_{1:H}\sim\pi'(x_{1:H})$ to indicate
that $a'_1\sim{}\pi'(x_1),\ldots,a'_H\sim{}\pi'(x_H)$.  We begin by
showing (\cref{lem:regret_rho}) that for all (potentially stochastic) policies $\pistar$ and $\pihat$,  \iftoggle{neurips}{$J(\pistar)-J(\pihat) \leq R\cdot{}\Drho{\pistar}{\pihat}$.}{
  \begin{align}
    \label{eq:pf1}
    J(\pistar)-J(\pihat) \leq R\cdot{}\Drho{\pistar}{\pihat}.
  \end{align}
}
  We then show (\cref{lem:hellinger_policy}) that whenever $\pistar$ is deterministic, Hellinger
  distance satisfies\footnote{In fact, the opposite direction of this
    inequality holds as well, up to an absolute constant.} \iftoggle{neurips}{$\DhelsX{\Big}{\bbP^{\pihat}}{\bbP^{\pistar}} \geq
    \frac{1}{4}\cdot\Drho{\pihat}{\pistar}$.}{
    \begin{align}
      \label{eq:pf2}
\DhelsX{\Big}{\bbP^{\pihat}}{\bbP^{\pistar}} \geq
    \frac{1}{4}\cdot\Drho{\pihat}{\pistar}.
  \end{align}
    }
Finally, we show (\cref{lem:lmax_swap}) that the trajectory-level
distance is symmetric, i.e. %
\arxiv{\begin{align}
         \label{eq:pf3}
         \Drho{\pihat}{\pistar} =   \Drho{\pistar}{\pihat}.
\end{align}}This step is perhaps the most critical: by considering
trajectory-level errors, we can switch from the state distribution
induced by $\pihat$ to that of $\pistar$ for free, without incurring
error amplification or spurious horizon factors. Combining the
preceding inequalities yields \cref{thm:bc_deterministic}; see
\cref{sec:proofs_main} for the full proof.

This analysis is closely related to a result in
\citet{rajaraman2021value}. For the special case of deterministic, linearly
parameterized policies with parameter sharing, \citet{rajaraman2021value} consider an
algorithm that minimizes an empirical analogue of the trajectory-wise
distance in \cref{eq:traj_metric_body}, and show that it leads to a bound
similar to \iftoggle{neurips}{\cref{cor:bc_deterministic_finite}}{\cref{eq:bc_df}} (i.e., linear-in-$H$ sample complexity
under dense rewards). Relative to this work, our contributions are
threefold: (i) we show that horizon-independent sample complexity can
be achieved for \emph{arbitrary} policy classes with parameter
sharing, not just linear classes; (ii) we show that said guarantees can be achieved by a
natural algorithm, \loglossbc, which is already widely used in
practice; and (iii), by virtue of considering the log loss, our
results readily generalize to encompass stochastic expert policies, as
we will show in the sequel.\arxiv{\footnote{A fourth benefit is that our analysis
  supports the setting in which $\pistar$ is deterministic, yet $\Pi$
  contains stochastic policies. This is a natural setting which can
  arise when, for example, $\Pi$ is parameterized by softmax
  policies. Guarantees under misspecification, which support this
  setting, are given in \cref{sec:bc_examples}.}}

  \section{Horizon-Independent Analysis of \loglossbc for Stochastic
      Experts}
    \label{sec:stochastic}
In this section, we turn out attention to the general setting in which the expert policy
$\pistar$ is stochastic. Stochastic policies are widely used in
practice, where they are useful for modeling multimodal behavior
\citep{shafiullah2022behavior,chi2023diffusion,block2024provable}, but
have received relatively little exploration in theory beyond the work
of \citet{rajaraman2020toward} for tabular policies.\footnote{As
  discussed at length in \citet{rajaraman2020toward}, many prior works
  \citep{ross2010efficient,ross2011reduction} state results in a
  level of generality that allows for stochastic experts, but the
  notions of supervised learning error found in these works (e.g., TV distance) do
  not lead to tight rates when instantiated for stochastic experts.\loose
}
\arxiv{

}
Our main result for this
section\arxiv{, \cref{thm:bc_stochastic},} is a regret decomposition based on the supervised learning
error $\DhelsX{\big}{\bbP^{\pihat}}{\bbP^{\pistar}}$ that is horizon-independent and \emph{variance-dependent}
\citep{zhou2023sharp,zhao2023variance,wang2024more}. \arxiv{To state the
guarantee, we define the following notion of variance for the expert policy:
\begin{align}
  \label{eq:variance}
  \sigmastar^2\ldef{}\sum_{h=1}^{H}\En^{\pistar}\brk*{(\Qstar_h(x_h,\pistar_h(x_h))-\Qstar_h(x_h,a_h))^2}.
\end{align}%
We can equivalently write this as $\sigmastar^2=\sum_{h=1}^{H}\En^{\pistar}\brk*{(\Vstar_h(x_h)-\Qstar_h(x_h,a_h))^2}$.
}
\arxiv{Our main result is as follows.}
\begin{theorem}[Horizon-independent regret decomposition]
  \label{thm:bc_stochastic}
  Assume $R\geq{}1$. For any pair of (potentially stochastic) policies
  $\pistar$ and $\pihat$ and any $\veps\in(0,e^{-1})$, 
  \begin{align}
    \label{eq:bc_stochastic}
    J(\pistar)-J(\pihat)
    \leq{} \sqrt{6\sigmastar^2\cdot\DhelsX{\big}{\bbP^{\pihat}}{\bbP^{\pistar}}}
    + \bigoh\prn*{R\log(R\veps^{-1})}\cdot\DhelsX{\big}{\bbP^{\pihat}}{\bbP^{\pistar}} + \veps.
  \end{align}
\end{theorem}
Applying this result with \loglossbc leads to the following guarantee.
    \begin{corollary}[Regret of \loglossbc]
      \label{cor:bc_stochastic_finite}
            For any \arxiv{expert} $\pistar\in\Pi$, the
      \loglossbc algorithm in \cref{eq:log_loss_bc} ensures that with
      \arxiv{probability} at least $1-\delta$, \iftoggle{neurips}{$J(\pistar)-J(\pihat)
      \leq{} \bigoh(1)\cdot\sqrt{\frac{\sigmastar^2\log(\abs{\Pi}\delta^{-1})}{n}}
      +    \bigoh(R\log(n))\cdot\frac{\log(\abs{\Pi}\delta^{-1})}{n}$.\loose}{
        \begin{align}
    \label{eq:bc_stochastic_mle}
    J(\pistar)-J(\pihat)
    \leq{} \bigoh(1)\cdot\sqrt{\frac{\sigmastar^2\log(\abs{\Pi}\delta^{-1})}{n}}
    +    \bigoh(R\log(n))\cdot\frac{\log(\abs{\Pi}\delta^{-1})}{n}.
  \end{align}
      }
    \end{corollary}
As we show\arxiv{ in the sequel}, when the expert policy is stochastic,
we can no longer hope for a ``fast'' $1/n$-type rate, and must instead
settle for a ``slow'' $1/\sqrt{n}$-type rate. The slow
term in \iftoggle{neurips}{\cref{cor:bc_stochastic_finite}}{\cref{eq:bc_stochastic_mle}} is controlled by the variance
$\sigmastar^2$ for the optimal policy. In particular, if $\pistar$ is
deterministic, then $\sigmastar^2=0$, and \iftoggle{neurips}{\cref{cor:bc_stochastic_finite}}{\cref{eq:bc_stochastic_mle}}
recovers our bound for the deterministic setting in
\cref{cor:bc_deterministic_finite} up to a $\log(n)$ factor.\loose

\subsection{Horizon-Independence and Optimality for Stochastic Experts}
To understand the dependence on horizon in
\cref{cor:bc_stochastic_finite}, we restrict our attention to the
``parameter sharing'' case where
$\log\abs{\Pi}=\bigoh(1)$, and separately discuss the sparse and
dense reward settings (results summarized in
\cref{tb:results_stochastic}).

\arxiv{
  \arxiv{\begin{table}[tp]}
	\renewcommand{\arraystretch}{1.6}
	\fontsize{9}{10}\selectfont
	\centering 
	\begin{tabular}{c|ccc}
		\hline
		& Worst-case & Low-noise & $\mutil$-recoverable \\
		
		\hline
          \makecell{Sparse\\ Rewards} & $\Otilde\left( R \sqrt{\frac{ \log(|\Pi|)}{n}} \right)$ & $\Otilde\left( \sqrt{\frac{\sigma_{\pistar}^2 \log(|\Pi|)}{n} }+ \frac{R \log(|\Pi|)}{n} \right)$ & N/A \\
          \makecell{Dense\\ Rewards} & $\Otilde\left( H \sqrt{\frac{ \log(|\Pi|)}{n}} \right)$ & $\Otilde\left( \sqrt{\frac{\sigma_{\pistar}^2 \log(|\Pi|)}{n} }+ \frac{H \log(|\Pi|)}{n} \right)$ & $\Otilde\left( \mutil \sqrt{\frac{H \log(|\Pi|)}{n}} + \frac{H \log(|\Pi|)}{n} \right)$ \\
		\hline
	\end{tabular}
        	\caption{Summary of upper bounds for stochastic experts
          (\cref{cor:bc_stochastic_finite}).  Each cell denotes the expected
          regret of a policy learned with \loglossbc; lower bounds are
          more nuanced and discussed in \cref{sec:stochastic}.  Here
          $\Pi$ is the policy class, $R$ is the cumulative reward
          range, $H$ is the horizon, $n$ is the number of expert
          trajectories, $\sigma_{\pistar}^2$ is the variance of the
          expert policy \arxiv{(\cref{eq:variance})}, and $\mutil$ is
          the signed recoverability parameter
          \arxiv{(\cref{eq:signed_rec})}.\protect\footnote{For arbitrary, non-stationary policy classes, $\log(|\Pi|)$ hides an additional factor linear in $H$.\loose}
	}
	\label{tb:results_stochastic}
      \end{table}

  }

Consider the sparse
reward setting where $R=\bigoh(1)$. %
Here, at first glance it would appear that the variance $\sigmastar^2$
should scale with the horizon. Fortunately, this is not the case: The following result---via a
law-of-total-variance-type argument \citep{azar2017minimax}---implies that \cref{cor:bc_stochastic_finite} is \emph{fully
  horizon-independent}, with no explicit dependence on horizon when
$R=\bigoh(1)$ and $\log\abs{\Pi}=\bigoh(1)$. For a function $f(x_{1:H},a_{1:H})$, let
  $\Var^{\pi}\brk*{f}$ denote the variance of $f$ under
  $(x_1,a_1),\ldots,(x_H,a_H)\sim\pi$
.

\begin{proposition}
  \label{prop:sigma_ltv}
We have that $\sigmastar^2\leq{}\Var^{\pistar}\brk[\big]{\sum_{h=1}^{H}r_h}\leq{}R^2$.
\end{proposition}

For the dense-reward regime where $R=H$, \cref{prop:sigma_ltv} gives \iftoggle{neurips}{$ J(\pistar)-J(\pihat)
\approxleq{} H\sqrt{\log(\abs{\Pi})/n}$}{$ J(\pistar)-J(\pihat)
    \approxleq{} H\sqrt{\frac{\log(\abs{\Pi})}{n}}$}. This is somewhat disappointing,
    as we now require $\bigom(H^2)$ trajectories (quadratic sample complexity) to learn a non-trivial policy, even
    when $\log\abs{\Pi}=\bigoh(1)$. \iftoggle{neurips}{
      We show now that this quadratic lower bound
        is qualitatively tight: the slow \iftoggle{neurips}{$ H / \sqrt{n}$}{$H\sqrt{\frac{1}{n}}$} rate
        for $\sigmastar^2=H^2$ is
        necessary in both offline and online
        IL. This reveals a fundamental difference between deterministic and stochastic experts,
since $\bigoh(H)$ sample complexity is sufficient in the former case.\loose
        \begin{theorem}[informal]
          \label{thm:variance_lower_informal}
          For any $\sigma^2\in\brk*{H,H^2}$, there exists $\Pi$ with
          $\abs{\Pi}=2$ such that
          $\sigmastar^2\leq\sigma^2$ and any (offline or online) IL algorithm must have
          $J(\pistar)-J(\pihat) \approxgeq\sqrt{\frac{\sigma^2}{n}}$
          with constant probability.
\end{theorem}\loose
    }{The following result shows that the
    dependence on the variance in \cref{cor:bc_deterministic_finite}
    cannot be improved in general, which implies that the
    quadratic horizon dependence in this regime is tight.%

      \begin{theorem}[Lower bound for stochastic experts]
        \label{thm:variance_lower}
        Consider the dense reward setting where $r_h\in\brk{0,1}$ and $R=H$.
        For any $n\in\bbN$, $H\in\bbN$ and $\sigma^2\in\brk{H,H^2}$, there exists a reward-free
        MDP $\Mstar$ with $\abs{\cX}=3$ and $\abs{\cA}=2$, a class of reward functions $\cR$ with
        $\abs*{\cR}=2$, and a class of policies $\Pi$ with $\abs{\Pi}=2$ with the following
  property. For any (online or offline) imitation learning algorithm, there exists a deterministic reward
  function $r=\crl*{r_h}_{h=1}^{H}$ and expert
  policy $\pistar\in\Pi$ such that $\sigmastar^2\leq{}\sigma^2$ and
  $\mutil\leq\sigma^2/H$ (\cref{eq:signed_rec}), and for which
  \begin{align}
          \bbP\prn*{J(\pistar)-J(\pihat) \geq
          c\cdot\sqrt{\frac{\sigma^2}{n}}} \geq \frac{1}{8}
        \end{align}
        for an absolute constant $c\geq{}1$.
\end{theorem}
Beyond showing that a slow $1/\sqrt{n}$ rate is required for
stochastic policies,\footnote{\citet{rajaraman2020toward} show that
  for the tabular setting, it is possible to achieve a $1/n$-type rate
  \emph{in-expectation} for stochastic policies. Their result critically
  exploits the assumption that $\abs{\cX}$ and $\abs{\cA}$ are small
  and finite to argue that it is possible to build an unbiased
  estimator for $\pistar$. \cref{thm:variance_lower} shows that such a
  result cannot hold with even \emph{constant probability} for the
  same setting. We believe the fact that a $1/n$-type rate is 
  possible in expectation is an artifact of the tabular setting,
  and unlikely to hold for general policy classes.
  }
 when specializing to $\sigma^2=H^2$, this result shows that $\bigom(H^2)$
trajectories are required to learn a non-trivial policy under
a stochastic expert, even when $\log\abs{\Pi}=\bigoh(1)$; this reveals a fundamental difference between deterministic and stochastic experts,
since $\bigoh(H)$ sample complexity is sufficient in the former case.\loose}

Nonetheless, it is possible to obtain linear-in-$H$ sample complexity for
dense rewards under a recoverability-like condition. Let us define the
\emph{signed recoverability constant} via
\arxiv{\begin{align}
  \label{eq:signed_rec}
    \mutil = \max_{x\in\cX,a\in\cA,h\in\brk{H}}\abs[\big]{(\Qstar_h(x,\pistar_h(x))-\Qstar_h(x,a)}.
  \end{align}}
Note that $\mutil\in\brk{0,R}$, and that $\mutil\geq\murec$, since
this version counts actions $a$
that \emph{outperform $\pistar$}, not just those that underperform. It
is immediate to see that $\sigmastar^2\leq{}\mutil^2{}H$. Hence, even
if $R=H$, as long as $\mutil=\bigoh(1)$,
\cref{cor:bc_stochastic_finite} yields \iftoggle{neurips}{$ J(\pistar)-J(\pihat)
\approxleq{} \sqrt{H\log(\abs{\Pi})/n}$}{$ J(\pistar)-J(\pihat)
    \approxleq{} \sqrt{\frac{H\log(\abs{\Pi})}{n}}
    +\frac{H\log(\abs{\Pi})}{n}$}, so that
    $\bigoh\prn[\big]{\frac{H\log\abs{\Pi}}{\veps^2}}$ trajectories suffice to
    learn an $\veps$-optimal policy.\footnote{An interesting question
      for future work is to understand if a similar conclusion
      holds if we replace $\mutil$ with $\murec$.\loose}\loose

    See \cref{sec:additional_lower} for further
      results concerning tightness of \cref{thm:bc_stochastic},
      including instance-dependent lower bounds.\loose

    \paragraph{Consequences for online versus offline IL}
The lower bound in \cref{thm:variance_lower} holds even for online
imitation learning algorithms. Thus, similar to the deterministic setting, there is no online IL
algorithm that improves upon \cref{thm:bc_stochastic} uniformly for
all policy classes. This means that even for stochastic experts, online imitation learning cannot
  improve upon offline imitation learning in the realizable
    setting without further assumptions
  (e.g., no parameter sharing)
  on the policy class under consideration.

  \arxiv{\subsection{Proof Sketch for \creftitle{thm:bc_stochastic}}}
  When the expert is stochastic, the trajectory-wise distance in
  \cref{eq:traj_metric_body}, is no longer useful (i.e.,
  $\Drho{\pistar}{\pistar}\neq{}0$), which necessitates a more
  information-theoretic analysis. Our starting point is the following
  scale-sensitive change-of-measure lemma for Hellinger distance.\loose
  \begin{lemma}[Change-of-measure for Hellinger distance \citep{foster2021statistical,foster2022complexity}]
    \label{lem:hellinger_com} 
        Let $\bbP$ and $\bbQ$ be probability distributions over a measurable space
    $(\cX,\filt)$. Then for all functions $h:\cX\to\bbR$, 
            \begin{equation}
      \label{eq:com1}
      \abs*{\En_{\bbP}\brk*{h(X)}
        - \En_{\bbQ}\brk*{h(X)}}
      \leq{} \sqrt{\tfrac{1}{2}\prn*{\En_{\bbP}\brk*{h^2(X)} + \En_{\bbQ}\brk*{h^2(X)}}\cdot\Dhels{\bbP}{\bbQ}}.
    \end{equation}
    In particular, if $h\in\brk{0,R}$ almost surely, then \iftoggle{neurips}{$\En_{\bbP}\brk*{h(X)}
    \leq{} 2       \En_{\bbQ}\brk*{h(X)} + R\cdot\Dhels{\bbP}{\bbQ}$.}{
    \begin{align}
      \label{eq:com2}
      \En_{\bbP}\brk*{h(X)}
      \leq{} 2       \En_{\bbQ}\brk*{h(X)} + R\cdot\Dhels{\bbP}{\bbQ}.
    \end{align}
  }
  \end{lemma}
\iftoggle{neurips}{We sketch how to use
    \cref{lem:hellinger_com} to prove a weaker version of
    \cref{thm:bc_stochastic}, and defer the full proof, which builds
    on this argument, to \cref{sec:stochastic_proof}.}{We first sketch how to use this result to prove a weaker version
  of \cref{thm:bc_stochastic}, then explain how to strengthen this
  argument.} Define the \emph{sum of advantages} for a trajectory
  $o=(x_1,a_1),\ldots,(x_H,a_H)$ via %
  \arxiv{\begin{align}
           \Delta(o)=\sum_{h=1}^{H}\Qstar_h(x_h,\pistar_h(x_h))-\Qstar_h(x_h,a_h)
           =\sum_{h=1}^{H}\Vstar_h(x_h)-\Qstar_h(x_h,a_h).
  \end{align}}
  By the performance difference lemma, we can write
  $J(\pistar)-J(\pihat)=\En^{\pihat}\brk*{\Delta(o)}$, so applying
  \cref{eq:com1} yields
  \begin{align}
    J(\pistar)-J(\pihat)=\En^{\pihat}\brk*{\Delta(o)}
    \approxleq{} %
    \arxiv{\underbrace{\En^{\pistar}\brk*{\Delta(o)}}_{=0}}
    + \sqrt{(\En^{\pihat}\brk*{\Delta^2(o)}+\En^{\pistar}\brk*{\Delta^2(o)})\cdot\Dhels{\bbP^{\pihat}}{\bbP^{\pistar}}}.
  \end{align}
  From here, we observe that $\En^{\pistar}\brk*{\Delta(o)}=0$ and
  $\En^{\pistar}\brk*{\Delta^2(o)}=\sigmastar^2$ (this follows because advantages are a
  martingale difference sequence under $\bbP^{\pistar}$), so all that 
  remains is to bound the term $\En^{\pihat}\brk*{\Delta^2(o)}$. A
  crude approach is to observe that $\abs*{\Delta(o)}\leq{}\mutil{}H$,
  so that applying \iftoggle{neurips}{\cref{lem:hellinger_com}}{\cref{eq:com2}} gives \iftoggle{neurips}{$\approxleq     \En^{\pistar}\brk*{\Delta^2(o)} + (\mutil{}H)^2\cdot\Dhels{\bbP^{\pihat}}{\bbP^{\pistar}}$, }{
  \begin{align}
    \label{eq:com_loose}
    \En^{\pihat}\brk*{\Delta^2(o)}
    \arxiv{\approxleq     \underbrace{\En^{\pistar}\brk*{\Delta^2(o)}}_{=\sigmastar^2} + (\mutil{}H)^2\cdot\Dhels{\bbP^{\pihat}}{\bbP^{\pistar}},}
  \end{align}
  }
  and consequently %
  \arxiv{  \begin{align}
        J(\pistar)-J(\pihat)
    \approxleq{} \sqrt{\sigmastar^2\cdot\DhelsX{\big}{\bbP^{\pihat}}{\bbP^{\pistar}}}
    + \mutil{}H\cdot\DhelsX{\big}{\bbP^{\pihat}}{\bbP^{\pistar}}.
  \end{align}}
  This falls short of \cref{eq:bc_stochastic} due to
  the suboptimal lower-order term, which does not recover
  \cref{thm:bc_deterministic} when $\pistar$ is deterministic
  ($\sigmastar^2=0$). \iftoggle{neurips}{The full proof in
    \cref{sec:stochastic_proof} corrects this disparity using a subtle and
      significantly more involved argument based on stopping times and
      martingale concentration. }
  {
  To address this, we make use of the
  following \emph{advantage concentration lemma}.
      \begin{lemma}[Concentration for advantages]
    \label{lem:advantage_concentration}
  Assume that $r_h\geq{}0$ and $\sum_{h=1}^{H}r_h\in\brk{0,R}$ almost
  surely for some $R>0$. Then for any (potentially stochastic) policy $\pi$, it holds that for all
  $\delta\in(0,e^{-1})$,
  \begin{align}
    \bbP^{\pi}\brk*{\exists{}H' : \abs*{\sum_{h=1}^{H'}Q^{\pi}_h(x_h,a_h)-V^{\pi}_h(x_h)}\geq{}c\cdot{}R\log(\delta^{-1})}
    \leq \delta,
  \end{align}
  for an absolute constant $c>0$.
  \end{lemma}
This result shows that even though the sum of advantages $\Delta(o)$
could be as large as $\mutil{}H$ for a given realization of the
trajectory $o=(x_1,a_1),\ldots,(x_H,a_H)$, the range is bounded as
$\abs{\Delta(o)}\approxleq{}R$ (i.e., horizon-independent) \emph{with
  high probability} under $\pistar$. From here, the crux of the proof
is a stopping time argument, which we use to argue
that---up to negligible approximation error---we can truncate
$\Delta(o)$ to order $R$, facilitating a tighter application of the
change-of-measure argument in \cref{eq:com_loose}. The stopping time
argument is quite subtle and somewhat involved, owing to the fact that while $\Delta(o)$
concentrates well under $\bbP^{\pistar}$, it is not guaranteed (a-priori) to
concentrate under $\bbP^{\pihat}$.
  }

  \arxiv{
    \section{Benefits of Online Interaction}
        \label{sec:online}
        
Our results in \cref{sec:main,sec:stochastic} show that the benefits of online
interaction in imitation learning---to the extent that horizon is
concerned---are more limited than previously thought. We expect that
in practice, online interaction will likely lead to benefits, but only in a
problem-dependent sense. To this end, we first discuss the
  role of misspecification and the realizability assumption used by
  our results, then highlight several special cases in
  which online interaction is indeed beneficial, but in a policy class-dependent fashion not captured by existing theory. In particular, we identify three phenomena which lead to
  improved sample complexity: (i) \emph{representational benefits};
  (ii) \emph{value-based feedback}; and (iii) \emph{exploration}. Our
  results in this section can serve as a starting point toward developing a
more fine-grained understanding of algorithms and sample complexity of imitation
learning.

\subsection{The Role of Misspecification}
This paper (for both deterministic and stochastic experts) focuses on
the realizable setting in which $\pistar\in\Pi$.  It is natural to ask
how the role of horizon in imitation learning changes under
misspecification, and whether
  online interaction brings greater benefits in this case.
This is a subtle issue, as there are various
incomparable notions of misspecification error which can lead to
different forms of horizon dependence. For example, for deterministic
experts, if $\Pi$ is misspecified in the sense that
$\inf_{\pi\in\Pi}\Lbc(\pi)\leq\vepsapx$, the indicator-loss behavior cloning
algorithm in \cref{eq:bc} achieves $J(\pistar) -
J(\pihat)\approxleq{}RH\cdot{}\prn*{\frac{\log(\abs{\Pi}\delta^{-1})}{n}+\vepsapx}$,
which is tight in general. In other words, the dependence on
$\vepsapx$ is not horizon-independent. On the other hand, as we show
in \cref{sec:additional}, if we assume that
$\inf_{\pi\in\Pi}\Dchis{\bbP^{\pistar}}{\bbP^{\pi}}\leq\vepsapx$, a
stronger notion of misspecification error, then \loglossbc achieves a
horizon-independent guarantee of the form $J(\pistar) -
J(\pihat)\approxleq{}R\cdot{}\prn*{\frac{\log(\abs{\Pi}\delta^{-1})}{n}+\vepsapx}$. We
leave a detailed investigation of tradeoffs between
      misspecification and horizon (as well as interplay with
      online versus offline IL) for future work; by giving the
      first horizon-independent treatment for the realizable setting,
      we hope that our results can
      serve as a starting point.\loose

\subsection{Representational Benefits}
The classical intuition behind algorithms like \dagger and \aggrevate
(which \cref{def:recoverability} attempts to quantify) is
\emph{recoverability}: through online access, we can learn to correct
the mistakes of an imperfect policy. Our results in \cref{sec:main,sec:stochastic}
show that recoverability has limited benefits for stationary
policy classes as far as horizon is concerned. In spite of this, the
following proposition shows that recoverability can have pronounced
benefits for \emph{representational} reasons, even with constant horizon.
\begin{proposition}[Representational benefits of online IL]
  \label{prop:benefits_representation}
    For any $N\in\bbN$, there exists a class $\cM$ of MDPs with $H=2$ and a policy
  class $\Pi$ with $\log\abs{\Pi}=\bigoh(N)$ such that
  \begin{itemize}
  \item There is an online imitation learning algorithm that achieves
    $J(\pistar)-J(\pihat)=0$ with probability at least $1-\delta$ using $O(\log(\delta^{-1}))$ episodes for any MDP
    $\Mstar\in\cM$ and expert policy $\pistar\in\Pi$. In particular,
    this can be achieved by \dagger.
  \item Any proper offline imitation learning algorithm requires $n=\bigom(N)$
    trajectories to learn a non-trivial policy with
    $J(\pistar)-J(\pihat)\leq{}c$ for an absolute constant
    $c>0$.\footnote{We expect that this result extends to
      \emph{improper} offline IL algorithms for which
      $\pihat\notin\Pi$, but a more complicated construction is
      required; we leave this for the next version of the paper.}
  \end{itemize}
\end{proposition}
The idea behind this construction is as follows: The behavior of the
(stochastic) expert policy at step $h=1$ is very complex, and learning
to imitate it well in distribution (e.g., with respect to total
variation or Hellinger distance) is a difficult representation
learning problem (in the language of \cref{sec:main}, e.g., \cref{thm:bc_deterministic}, we must
take $\log\abs{\Pi_1}$ very large in order to realize $\pistar_1$). For
offline imitation learning, we have no choice but to imitate $\pistar_1$
well at $h=1$, leading to the lower bound in \cref{prop:benefits_representation}. With online access though, we can give up on learning
$\pistar_1$ well, and instead learn to correct our mistake at step
$h=2$. For the construction in \cref{prop:benefits_representation},
this a much easier representation learning problem,
and requires very low sample complexity (i.e., we can realize
$\pistar_2$ with a class
$\Pi_2$ for which $\log\abs{\Pi_2}$ is small. We conclude that \dagger can indeed lead to substantial benefits over
offline IL, but for representational reasons unrelated to horizon, and
not captured by existing theory. While this example is somewhat
contrived, it suggests that potential to develop a deeper
understanding of representational benefits in imitation learning,
which we leave as a promising direction for future work.

\subsection{Benefits of Value-Based Feedback}

Beginning with the work of \citet{ross2014reinforcement} on
\aggrevate, many works (e.g., \citet{sun2017deeply}) consider a
\emph{value-based feedback} variant of the online IL framework (\cref{sec:background}) where in
addition to (or instead of) observing $\astar_h$, the learner observes
the expert's advantage
function
$\Astar_h(x_h,\cdot)\ldef{}\Qstar_h(x_h,\pistar_h(x_h))-\Qstar_h(x_h,\cdot)$
or value function $\Qstar_h(x_h,\cdot)$ at every state visited by the learner (see \cref{sec:value}
for details, which are deferred to the appendix for space). While such feedback
intuitively seems useful, existing theoretical guarantees---to the
best of our knowledge---\citep{ross2014reinforcement,sun2017deeply}
only show that algorithms like \aggrevate are no worse than non-value based methods
like \dagger, and do not quantify situations in which value-based
feedback actually leads to improvement.\footnote{These results
  are reductions which bound regret in terms of different notions of supervised
  learning performance, which makes it somewhat difficult to compare them
  or derive concrete end-to-end guarantees.}

The following result shows that i) value-based feedback can lead to arbitrarily large
improvement over non-value based feedback for representational reasons similar to \cref{prop:benefits_representation}
(that is for a complicated stochastic expert, learning to optimize a
fixed value function can be much easier than learning to imitate the
expert well in TV distance), but ii) it is only possible to
exploit value-based feedback in this fashion under online interaction
(that is, even if we annotate the trajectories for offline imitation
learning with $\Astar_h(x_h,\cdot)$ for the visited states, this
cannot lead to improvement in sample complexity).
\begin{proposition}[Benefits of value-based feedback (informal)]
  \label{prop:benefits_value}
  For any $N\in\bbN$, there is a class of MDPs $\cM$ with $H=2$ and a policy
  class $\Pi$ with $\log\abs{\Pi}=\bigoh(N)$ such that
  \begin{itemize}
  \item There is an online imitation learning algorithm with
    value-based feedback that achieves
    $J(\pistar)-J(\pihat)=0$ with probability at least $1-\delta$
    using $O(\log(\delta^{-1}))$ episodes for every MDP
    $\Mstar\in\cM$ and expert $\pistar\in\Pi$.
    In particular, this can be achieved by
    \aggrevate.
  \item Any proper offline imitation learning algorithm (with value-based
    feedback) or proper online imitation learning algorithm (without
    valued-based feedback) requires $n=\bigom(N)$
    trajectories to learn a non-trivial policy with
    $J(\pistar)-J(\pihat)\leq{}c$ for an absolute constant
    $c>0$.\footnote{As with \cref{prop:benefits_representation}, we
      expect that this lower bound can be extended to improper learners, but a more
      complicated construction is required.}
    \loose
  \end{itemize}
\end{proposition}
As with \cref{prop:benefits_representation}, this example calls for a
fine-grained policy class-dependent theory, which we hope to explore
more deeply in future work.
\subsection{Benefits from Exploration}
A final potential benefit of online interaction arises in
\emph{exploration}. One might hope that with online access, we can directly guide the MDP
to informative states that will help to identify the optimal policy
faster. The following proposition gives an example in which deliberate
exploration can lead to arbitrarily large improvement over offline
imitation learning, as well as over naive online imitation learning
algorithms like \dagger that do not deliberately explore.
\begin{proposition}[Benefits of exploration for online IL]
  \label{prop:benefits_exploration}
      For any $n\in\bbN$ and $H\in\bbN$, there exists an MDP $\Mstar$ and a class of
  deterministic policies $\Pi$ with $\abs{\Pi}=2$ with the following
  properties.
  \begin{enumerate}
  \item There exists an online imitation learning algorithm that returns
    a policy $\pihat$ such that
    $J(\pistar) - J(\pihat) = 0$
    with probability at least $1-\delta$ using
    $\bigoh(\log(\delta^{-1}))$ episodes, for \emph{all possible
      reward functions} (i.e., even if $\mu=H$).\loose
  \item For any offline imitation learning algorithm, there exists a deterministic reward
  function $r=\crl*{r_h}_{h=1}^{H}$ and expert policy $\pistar\in\Pi$
  with $\mu=1$ such that any algorithm must have
  $\En\brk*{J(\pistar)-J(\pihat)}
    \geq{} \bigom(1)\cdot\frac{H}{n}$. In addition, \dagger has regret $\En\brk*{J(\pistar)-J(\pihat)}
    \geq{} \bigom(1)\cdot\frac{H}{n}$.\loose
  \end{enumerate}
\end{proposition}
The idea behind this construction is simple: We take the lower bound
construction from \cref{prop:lb_deterministic} and augment it with a
``revealing'' which directly reveals the identity of the underlying
expert. The true expert never visits this state, so offline imitation
learning algorithms cannot
exploit it (standard online IL algorithms like \dagger and relatives do not exploit
the revealing state for the same reason),\footnote{This phenomenon is
also distinct from ``active'' online imitation learning algorithms
\citep{sekhari2024selective} which can obtain improved sampling
complexity under strong distributional assumptions in the vein of
active learning \citep{hanneke2014theory}, but still do not
deliberately explore.} but a well-designed online
IL algorithm that deliberately navigates to the revealing state can
use it to identify $\pistar$ extremely quickly. 

As with the previous examples, this construction is somewhat
contrived, but it suggests that directly maximizing information acquisition may be a useful algorithm design
paradigm for online IL, and we hope to explore this more deeply in
future work.

      }

\arxiv{
\section{Experiments}
\label{sec:experiments}

  In this section, we validate our theoretical results empirically. We first provide a detailed overview of our experimental setup, including the control and natural language tasks we consider, then present empirical results for each task individually. %

\subsection{Experimental Setup}

We evaluate the effect of horizon on the performance of \loglossbc in three environments. We begin by describing our training and evaluation protocol (which is agnostic to the environment under consideration), then provide details for each environment.

In each experiment, we begin with an expert policy $\pistar$ (which is always a neural network; details below) and construct an offline dataset by rolling out with it $n$ times for $H$ timesteps per episode. To train the imitator policy $\pihat$, we use the same architecture as the expert, but randomly initialize the weights and use stochastic gradient descent with the Adam optimizer to minimize the \loglossbc objective for the offline dataset; this setup ensures that the realizability assumption used by our main results (\cref{ass:realizability}) is satisfied. We repeat this entire process for varying values of $H$.\loose

To evaluate the regret $J(\pistar)-J(\pihat)$ after training, we approximate the average reward of the imitator policy $\pihat$ by selecting new random seeds and collecting $n$ trajectories of length $H$ by rolling out with $\pihat$; we approximate the average reward of the expert $\pistar$ in the same fashion, and we also compute several auxiliary performance measures (details below) that aim to capture the distance between $\pihat$ and $\pistar$. In all environments, we normalize rewards so that the average reward of the expert is at most 1, in order to bring us to the sparse reward setting in \cref{sec:background} and keep the range of the possible rewards constant as a function of the (varying) horizon.\loose

We consider four diverse environments, with the aim of evaluating \loglossbc in qualitatively different domains: (i) \textsf{Walker2d}, a classical continuous control task from MuJoCo  \citep{towers_gymnasium_2023,todorov2012mujoco} where the learner attempts to make a stick figure-like agent walk to the right by controlling its joints; (ii) \textsf{Beamrider}, a standard discrete-action RL task from the Atari suite \citep{bellemare2013arcade}, where the learner attempts to play the game of Beamrider;  (iii) \textsf{Car}, a top-down discrete car racing environment where the car has to avoid obstacles to reach a goal, and (iv) \textsf{Dyck}, an autoregressive language generation task where the agent is given a sequence of brackets in $\left\{\{, \}, [, ], (, )\right\}$ and has to close all open brackets in the correct order.

We emphasize diversity in task selection in order to demonstrate the generality of our results, covering discrete and continuous actions spaces, as well as both control and language generation. For some of the environment (\textsf{Walker2d}, \textsf{Beamrider}), the task is intended to be ``stateless'', in the sense that varying the horizon $H$ does not change the difficulty of the task itself (e.g., complexity of the expert policy $\pistar$), allowing for an honest evaluation of the difficulty of the \emph{learning} problem as we vary the horizon $H$. For other domains, such as \textsf{Dyck},  horizon dependence is more nuanced, as here the capacity required to represent the expert grows as the horizon increases; this manifests itself in our theoretical results through the realizability condition (\cref{ass:realizability}), which necessitates a more complex function class $\Pi$ as $H$ increases.

\begin{figure}%
  \centering
  \begin{subfigure}[t]{0.49\textwidth}
    \includegraphics[width=\textwidth]{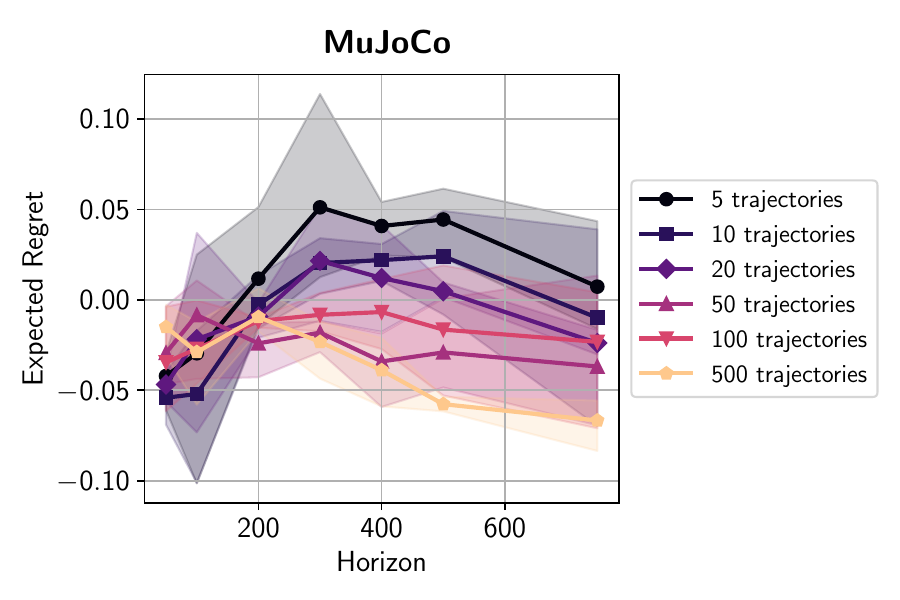}
    \caption{}
  \end{subfigure}
  \hfill
  \begin{subfigure}[t]{0.49\textwidth}
      \includegraphics[width=\textwidth]{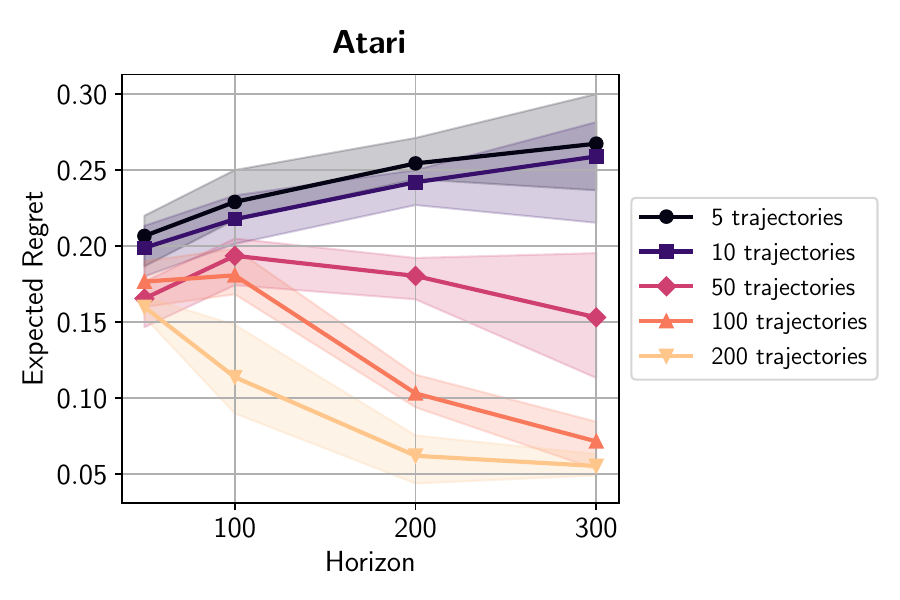}
      \caption{}
    \end{subfigure}
      \caption{Dependence of expected regret on the horizon for multiple choices for the number of imitator trajectories $n$. \textbf{(a)} Continuous control environment \textsf{Walker2d-v4}. \textbf{(b)} Discrete Atari environment \textsf{BeamriderNoFrameskip-v4}. For both environments, increasing the horizon does not lead to a significant increase in regret, as predicted by our theory. }
  \label{fig:horizon_dependence}
\end{figure}

We now provide details for our experimental setup for each environment.  

\paragraph{\textsf{Walker2d}}  We use the Gymnasium \citep{towers_gymnasium_2023} environment \textsf{Walker2d-v4}, which has continuous state and action spaces of dimensions 17 and 6 respectively. The agent is rewarded for moving to the right and staying alive as well as being penalized for excessively forceful actions; because we vary the horizon $H$, in order to make the comparison fair, we normalize the rewards so that our trained expert always has average reward 1.  Our expert is a depth-2 MLP with width 64. We use the Stable-Baselines3 \citep{sb3} implementation of the Proximal Policy Optimization (PPO) algorithm \citep{schulman2017proximal} with default settings to train the expert for 500K steps. The policy's action distribution is Gaussian, with the mean and covariance determined by the MLP; we use this for computation of the logarithmic loss. For data collection, we enforce a \emph{deterministic} expert by always playing the mean of the Gaussian distribution produced by their policy. Our imitator policy uses the same architecture as the expert policy, with the weights re-initialized randomly.  We train the imitator using the logarithmic loss by default, but as an ablation, we also evaluate the effect of training with the mean-squared-error loss on the Euclidean norm over the actions.  We train using the Adam optimizer \citep{kingma2014adam} with a learning rate of $10^{-3}$ and a batch size of 128.  We stop training early based on the validation loss on a held out set of expert trajectories. Note that the expert and imitator policies above are both \emph{stationary policies}.

\paragraph{\textsf{Beamrider}}  We use the Gymnasium environment \textsf{BeamRiderNoFrameskip-v4}, which has 9 discrete actions and a 210x160x3 image as the state; the rewards are computed as a function of how many enemies are destroyed.  As in the case of the previous setup, we account for the varying of $H$ by normalizing expert rewards to be 1.  Here we do not train our expert ourselves, but instead use the trained PPO agent provided by \citet{rlzoo3}, which is a convolutional neural network.  We use the same architecture for our imitator policy, with the weights re-initialized randomly.  Here, the expert (and imitator) policies map the observation to a point on the probability simplex over actions, and so logarithmic loss computation is immediate. Similar to the case of \textsf{Walker2d}, we enforce a deterministic expert for collecting trajectories by taking the action with maximal probability.  We then train our imitators using the same setup as in the \textsf{Walker2d} environment. As with \textsf{Walker2d}, the expert and imitator here are both stationary policies.

\begin{figure}
    \centering   
    \begin{subfigure}[b]{0.35\textwidth} 
      \includegraphics[width=\textwidth]{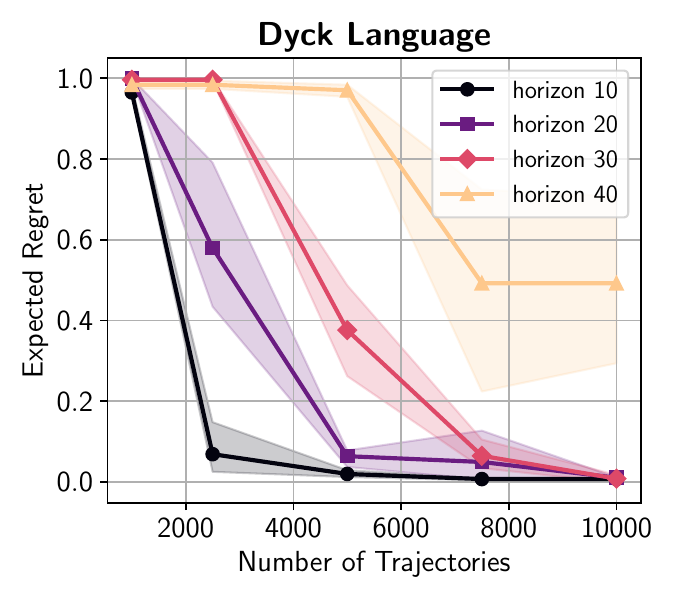}
        \caption{}
        \label{fig:dyck_n_dependence_logloss}
    \end{subfigure}
\hspace{.5in}
    \begin{subfigure}[b]{0.35\textwidth} 
        \includegraphics[width=0.915\textwidth]{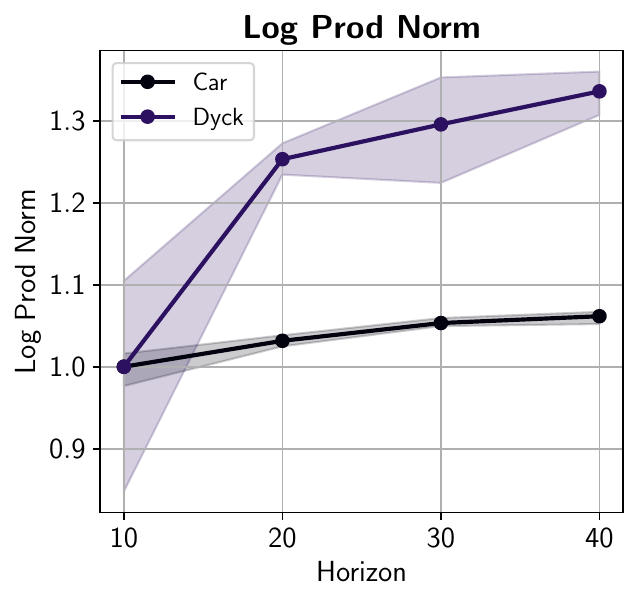}
        \caption{}
        \label{fig:dyck_norm}
      \end{subfigure}
        \caption{\textbf{(a)} Relationship between the number of expert trajectories and expected regret for the \textsf{Dyck} environment multiple choices of horizon $H$. The expert is trained to produce valid Dyck words of length $H$, and the imitator's ability to generate a valid word is evaluated. We find that regret increases as a function of $H$. \textbf{(b)} Logarithm of the product of weight matrix norms for the expert policy network as a function of $H$, for \textsf{Dyck} and \textsf{Car} environments. The log-product-norm acts as a proxy for complexity for the class $\Pi$; we rescale such that log-product-norm at $H=10$ is $1.0$ for both domains. For \textsf{Dyck}, we find that as $H$ increases, the complexity of $\Pi$ required to represent the expert policy (as measured by the log-product-norm) also increases, explaining the increasing regret in (a). However, the gain in log-product-norm for the \textsf{Car} domain is much lower, which is in line with the fact that the regret for the \textsf{Car} domain exhibits only mild scaling with horizon.}
    \label{fig:dyck}
  \end{figure}

\paragraph{\textsf{Car}} We introduce a simple top-down navigation task where the agent is a ``car'' that always moves forward by one step, but can take actions to move left, right, or remain in its lane to avoid obstacles and reach the desired destination.  There are $M$ possible lanes. At timestep $h \in [H+1]$, if the agent is in lane $i \in [M]$, then the agent's state is $(i, h)$. We view the state space as a $M \times (H + 1)$ grid; a given point $(i, j)$ in the grid can be empty, or contain an obstacle, or contain the agent. The agent's action space consists of 3 possible actions: stay in the current lane ($(i, h) \mapsto (i, h + 1)$), move one step left ($(i, h) \mapsto (i - 1, h + 1)$), or move one step right ($(i, h) \mapsto (i + 1, h + 1)$). If the agent's action causes it to collide with an obstacle or the boundary of the grid, it is sent to an absorbing state. The agent gets a reward of 1 for reaching the goal state for the first time, and a reward of 0 otherwise. When the agent occupies a state $(i, h)$, it observes an image-based observation $x_h$ showing the state of all lanes for $V$ steps ahead where $V$ is the size of the viewing field. At the start of each episode, we randomly sample obstacles positions, the start position, and the goal position. The goal can be reached after $H$ actions, and it is always possible to reach the goal.

  \paragraph{\textsf{Dyck}} In addition to the RL environments above, we evaluate \loglossbc for autoregressive language generation with transformers (cf. \cref{sec:autoregressive}),
where the goal of the ``agent'' is to complete a valid word of a
given length in a Dyck language; this has emerged as a popular
sandbox for understanding the nuances of autoregressive text
generation in theory
\citep{yao2021self,hahn2020theoretical,bhattamishra2020ability} and
empirically \citep{liu2022transformers,wen2024transformers}.  We recall that a Dyck language $\mathsf{Dyck}_k$ consists of $2k$ matched symbols thought of as open and closed parentheses, with concatenations being valid words if the parentheses are closed in the correct order.  For example, if we define the space of characters as `()', `[]', and `\{\}', then `$([()])\{\}$' is a valid word, whereas `$([)]$' and `(((\{\}' are not.

Our experiments use the Dyck language $\mathsf{Dyck}_3$.  For our expert, we train an instance of GPT-2 small \citep{radford2019language} with 6 layers, 3 heads, and 96 hidden dimensions from scratch to produce valid Dyck words. In particular, the training dataset consists of random Dyck prefixes that require exactly $H$ actions (symbols) to complete.  To imitate this expert, we train a GPT-2 small model with the same architecture, but with randomly initialized weights on an offline dataset of sequences generated by the expert. We assign a reward $1$ to each trajectory if the generated word is valid, and assign reward $0$ otherwise.  We use Adam optimization for training, with our experts trained for 40K iterations in order to ensure their quality. Note that in this environment, the expert and imitator policies are non-stationary, but use \emph{parameter sharing} via the transformer architecture.\loose

\begin{figure}
    \centering   
    \begin{subfigure}[b]{0.35\textwidth} 
      \includegraphics[width=\textwidth]{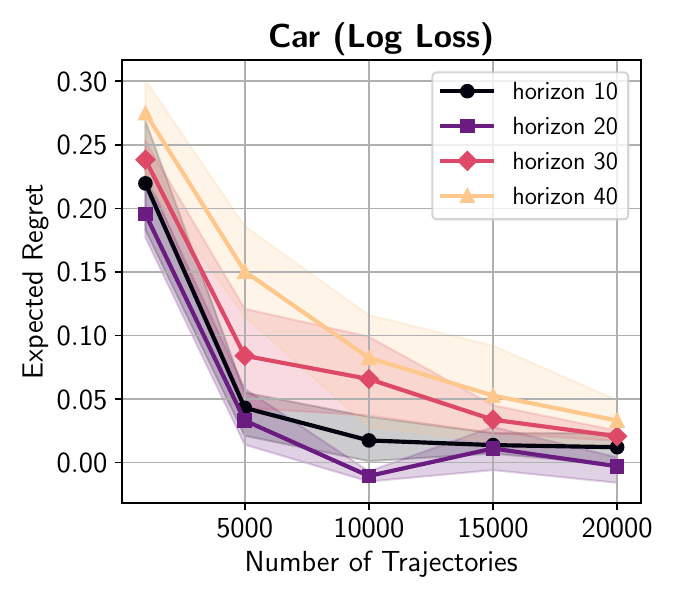}
        \caption{}
        \label{fig:car_n_dependence_logloss}
    \end{subfigure}
\hspace{.5in}
    \begin{subfigure}[b]{0.35\textwidth}
        \includegraphics[width=\textwidth]{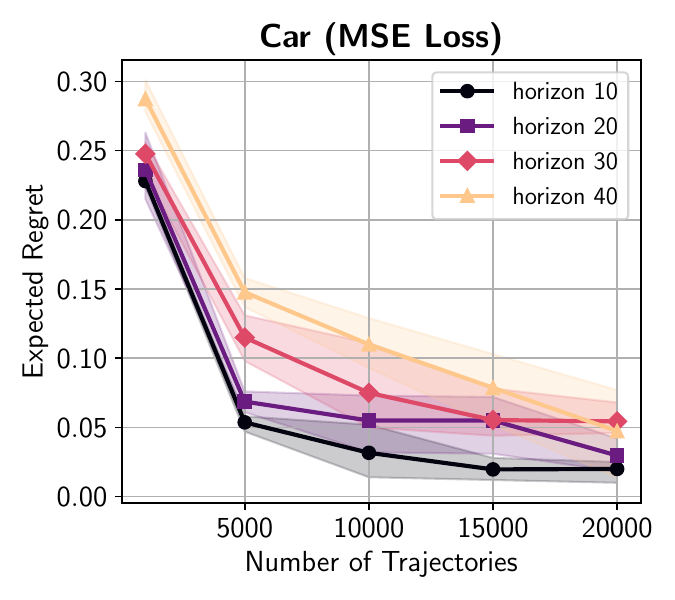}
        \caption{}
        \label{fig:car_norm}
      \end{subfigure}
      \caption{Dependence of expected regret on the number of expert trajectories for \textsf{Car} environment under varying values for horizon $H$ for log-loss (\textbf{a}) and mean-squared loss (\textbf{b}). The expert policy network is trained on a set of $2\times 10^4$ episodes generated by an optimal policy via behavior cloning. We use \loglossbc to train imitator policy for varying values of the horizon $H$ and number of trajectories $n$. For both losses, we find that the expected regret goes down as the number of expert trajectories increases, but degrades slightly as a function of $H$.\loose}
    \label{fig:car}
  \end{figure}

\subsection{Results}

We summarize our main findings below.

\paragraph{Effect of horizon on regret}
\Cref{fig:intro,fig:horizon_dependence} plot the relationship between expected regret and the number of expert trajectories for the \textsf{Walker2d} (MuJoCo), and \textsf{BeamriderNoFrameskip} (Atari) environments, as the horizon $H$ is varied from $50$ to $500$. For both environments, we find regret is largely independent of the horizon, consistent with our theoretical results. In fact, in the case of \textsf{BeamriderNoFrameskip}, we find that increasing the horizon leads to \emph{better} regret. To understand this, note that our theory provides horizon-agnostic \emph{upper bounds} independent of the environment. Our lower bounds are constructed for specific worst-case environments, and not rule out the possibility of improved performance with longer horizons environments with favorable structure. We conjecture that this phenomenon is related to the fact that longer horizons yield fundamentally more data, as the total number of state-action pairs in the expert dataset is equal to $nH$.\footnote{For example, if we repeat a fixed contextual bandit instance $H$ times across the horizon and train a stationary policy, it is clear that regret should decrease with $H$ under sparse rewards. Less trivial instances where increasing horizon provably leads to better performance are known in some special cases \citep{tu2022learning}.}\loose

\cref{fig:dyck}(a) plots our findings for the \textsf{Dyck} environment. Here, we see that with the number of trajectories $n$ fixed, regret does increase with $H$, which might appear to contradict our theory at first glance. However, we note that the policy class itself must become larger as $H$ increases, as the task itself becomes more difficult (equivalently, the supervised learning error $\Dhels{\bbP^{\pistar}}{\bbP^{\pihat}}$ must grow with $H$). As a result, the regret is not expected to be independent of $H$ for this environment, in spit of parameter sharing. To verify whether supervised learning error is indeed the cause for horizon dependence for \textsf{Dyck}, \cref{fig:dyck}(b) plots the logarithm of the product of the Frobenius norms of the weight matrices of the expert for varying values of $H$, as a proxy for supervised learning performance \citep{bartlett2017spectrally,golowich2018size}.\footnote{We only include log-product-norm plots for \textsf{Dyck} and \textsf{Car} because for the other environments (\textsf{Walker2d} and \textsf{BeamriderNoFrameskip}), we do not change the expert as a function of $H$.}
We find that the log-product-norms do in fact grow with $H$, consistent with the fact that the regret grows with $H$ in this case.

For the \textsf{Car} environment, we observe similar behavior to the \textsf{Dyck} environment, visualized in \cref{fig:car}. We find that performance degrades slightly as a function of the horizon $H$, but that this increase in regret can be explained by an increase in the log-product-norm (\cref{fig:dyck}(b)). However, the effect is mild compared to \textsf{Dyck}.\loose

\begin{figure}
    \centering
    \begin{subfigure}[b]{0.95\textwidth}
      \includegraphics[width=\textwidth]{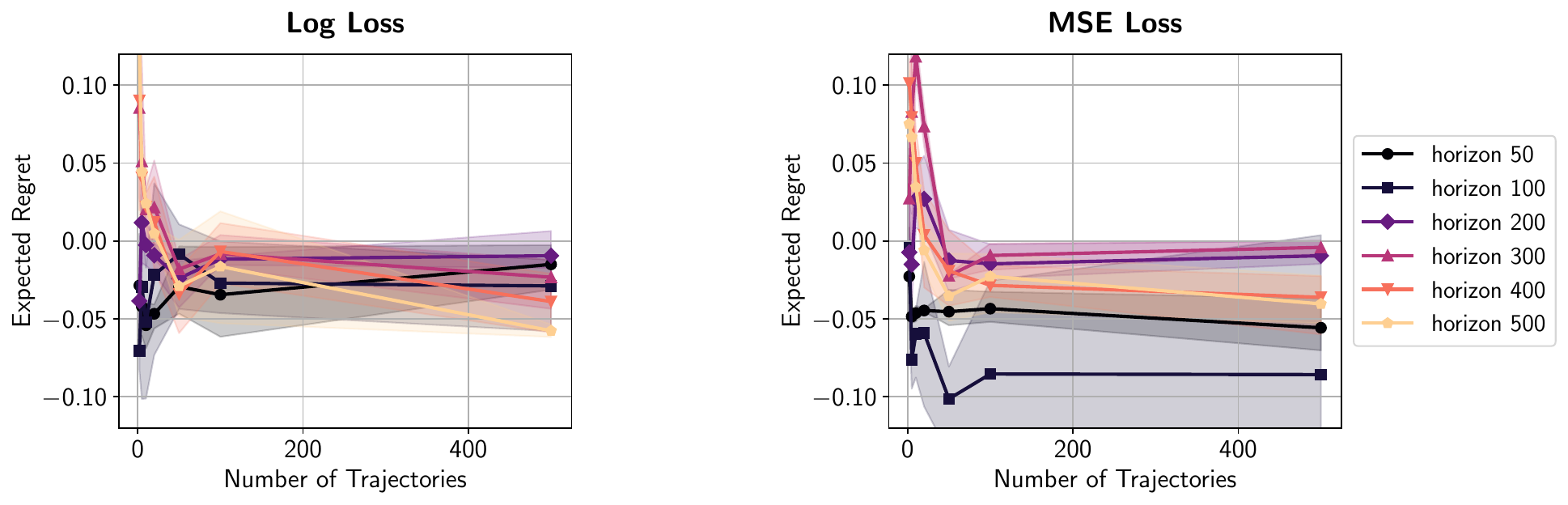}
      \label{fig:walker2d_n_dependence_bothlosses}
    \end{subfigure}
    \caption{Dependence of expected regret on the number of expert trajectories for continuous control environment \textsf{Walker2d-v4} under varying choices for horizon $H$. \textbf{(Left)} Behavior cloning with logarithmic loss (\loglossbc);  \textbf{(Right)} Behavior cloning with mean squared error (MSE) Loss. Both losses lead to similar performance for this environment, possibly due to Gaussian policy parameterization.}    
    \label{fig:n_dependence_bothlosses}
  \end{figure}

  \paragraph{Comparison between log loss and square loss} As an ablation, \Cref{fig:car,fig:n_dependence_bothlosses} compare \loglossbc to the original behavior cloning objective of \citet{pomerleau1988alvinn}, which uses the mean squared error (MSE) to regress expert actions to observations in the offline dataset. Focusing on the \textsf{Walker2d} environment (\cref{fig:n_dependence_bothlosses}) and \textsf{Car} environment (\cref{fig:car}) (other environments presented difficulties in training\footnote{In particular, we attempted a similar result in the Atari environment, using MSE loss being between vectors on the probability simplex over $|\mathcal A|$ actions. For MSE loss, we found that the imitator did not train, in the sense that even with 500 expert trajectories, the performance of the cloner did not improve. We suspect this was due to numerical instability in optimization for the MSE loss in this setup or a failure of hyperparameter optimization.}), we find that performance with the MSE loss is comparable to that of the logarithmic loss. For \textsf{Walker2d}, a possible explanation is that under the Gaussian policy parameterization we use, the MSE loss is the same as the logarithmic loss up to state-dependent heteroskedasticity.\footnote{In theory, the MSE loss can still underperform the logarithmic loss when the heteroskedasticity is severe \citep{foster2021efficient}, but this may not manifest for this environment.} Another possible explanation is that this is an instance of the phenomenon described in \cref{sec:indicator}.

    \begin{figure}
    \centering
    \begin{subfigure}[b]{0.45\textwidth}
      \includegraphics[width=\textwidth]{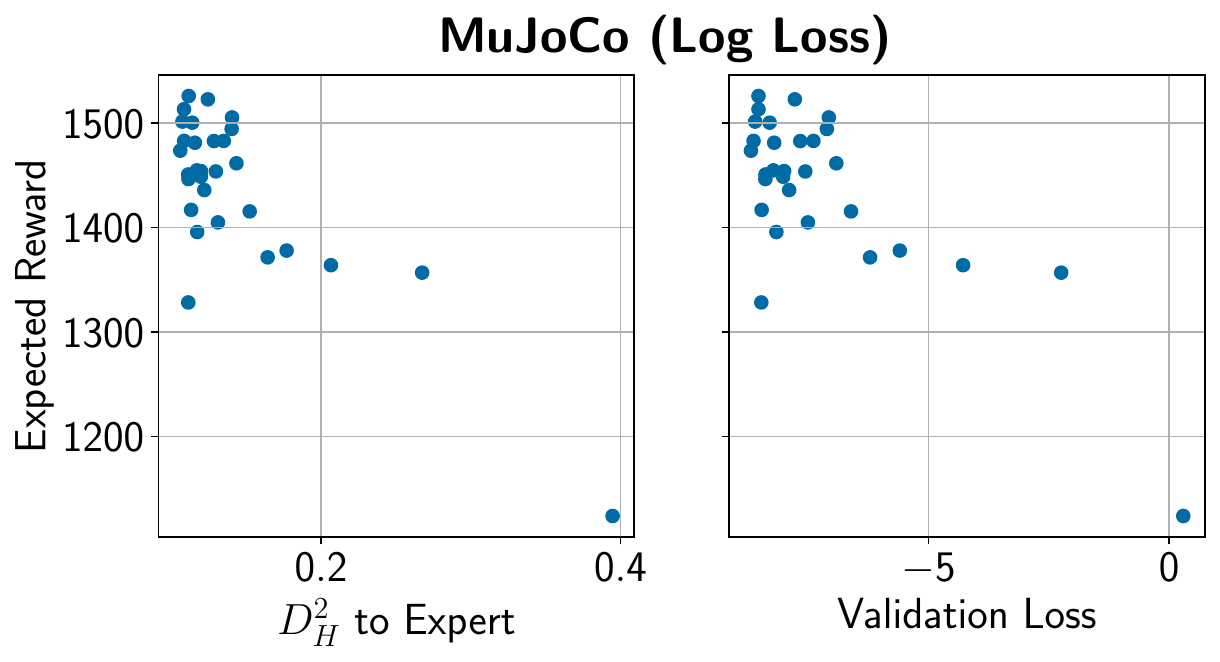}
      \caption{}
        \label{fig:walker2d_logloss}
    \end{subfigure}
    \hfill
    \begin{subfigure}[b]{0.45\textwidth}
        \includegraphics[width=\textwidth]{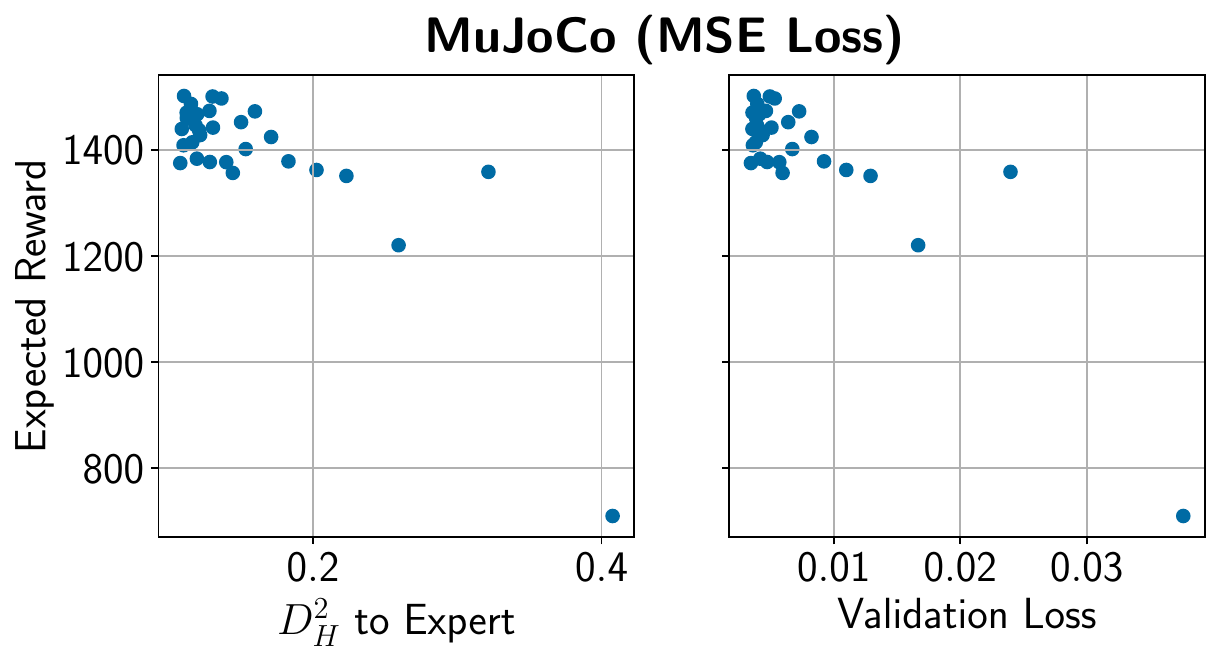}
        \caption{}
        \label{fig:walker2d_mseloss}
      \end{subfigure}
      \caption{Evaluation of the quality of (i) Hellinger distance $\Dhels{\bbP^{\pistar}}{\bbP^{\pihat}}$, and (ii) validation loss as a proxy for rollout reward. We plot Hellinger distance and validation loss against mean reward for a over a single training run for \textsf{Walker2d} environment with $H = 500$ and $n = 500$. \textbf{(a)} Results for \loglossbc, where the validation loss and Hellinger distance $D_{\mathsf{H}}^2$ are highly correlated, and serve as good proxies for the expected reward of the policy. \textbf{(b)} Results for MSE loss, where the validation loss is less well correlated with the expected reward (note the cluster in the upper left hand corner), but the Hellinger distance $D_{\mathsf{H}}^2$ remains a good proxy.\loose}      
    \label{fig:scatter_training_walker}
  \end{figure}

  \paragraph{Relationship between regret and Hellinger distance to expert}

  Finally, we directly evaluate the quality of (i) Hellinger distance $\Dhels{\bbP^{\pistar}}{\bbP^{\pihat}}$, and (ii) validation loss as proxies for rollout performance.  We estimate the Hellinger distance using sample trajectories. 
  \cref{fig:scatter_training_walker} displays our findings for \textsf{Walker2d} with $H=n=500$, where we observe that both metrics, particularly the Hellinger distance, are well correlated with rollout performance, as measured by average reward. In \Cref{fig:walker2d_logloss}, we see that under \loglossbc, Hellinger distance and validation loss are highly correlated with each other, and negatively correlated with expected reward, thereby acting as excellent proxies for rollout performance.
  Meanwhile, in \Cref{fig:walker2d_mseloss}, we find that under behavior cloning with the MSE loss, validation error is less well correlated with the expected reward of the imitator policy, as evinced by the cluster in the upper left corner, where there are policies with roughly the same validation loss, but variable expected reward.  On the other hand, the Hellinger distance $D_{\mathsf{H}}^2$ still appears to predict the performance of the policy well, as is consistent with our theoretical results.

}

\arxiv{
  \section{Discussion}
  \label{sec:discussion}
  We conclude with additional technical remarks and directions for
  future research.
  
\subsection{When is Indicator-Loss Behavior Cloning Suboptimal?}
\label{sec:indicator}
Our discussion in \cref{sec:bc_indicator} suggests that
indicator-loss behavior cloning, which solves
$
\pihat=\argmin_{\pi\in\Pi}\Lhatbc(\pi)\ldef{}\sum_{i=1}^{n}\frac{1}{H}\sum_{h=1}^{H}\indic\crl*{\pi_h(x_h\ind{i})\neq{}a_h\ind{i}}$,
can have suboptimal horizon dependence compared to \loglossbc. This
turns out to be a subtle point. Suppose that $\pistar$ is
deterministic, that $\Pi$ exactly satisfies realizability in the
sense that $\pistar\in\Pi$, and that $\Pi$ only contains deterministic
policies. In this case, we observe that
$\Lhatbc(\pihat)=0$ (i.e., $\pihat$ agrees with $\pistar$ on every
instance in the dataset). Consequently, $\pihat$ can also be viewed as minimizing an
empirical version of the trajectory-wise loss in
\cref{eq:traj_metric_body}, i.e.
\begin{align}
  \label{eq:traj_empirical}
  \sum_{i=1}^{n}\indic\crl*{\exists{}h: \pihat_h(x_h\ind{i})
   \neq a_h\ind{i}} = 0.
\end{align}
From here, a standard uniform convergence argument implies that
$\Drho{\pistar}{\pihat}\approxleq{}\frac{\log(\abs{\Pi}\delta^{-1})}{n}$,
and by combining this with \cref{eq:pf1}, we obtain the following result.
\begin{proposition}
  \label{prop:indicator_deterministic}
  For any deterministic expert $\pistar\in\Pi$, the indicator loss
  behavior cloning policy $\pihat=\argmin_{\pi\in\Pi}\Lhatbc(\pi)$
  ensures that with probability at least $1-\delta$,
  \[
    J(\pistar)-J(\pihat) \leq{} \bigoh(R)\cdot{}\frac{\log(\abs{\Pi}\delta^{-1})}{n}.
  \]
\end{proposition}
This result, which shows that indicator-loss BC attains a similar
horizon-independent rate
to\loglossbc under the conditions of
\cref{cor:bc_deterministic_finite}, is novel to our knowledge. While
this would seem to suggest that indicator-loss BC can match the
performance of \loglossbc, there are number of important caveats.

First, \cref{prop:indicator_deterministic} is not robust to
optimization errors or misspecification errors. For example, if
$\pihat$ only minimizes the indicator loss $\Lhatbc(\pi)$ up to error
$\vepsopt$, i.e.
\[
\Lhatbc(\pihat) \leq \inf_{\pi\in\Pi}\Lhatbc(\pi) + \vepsopt{}\cdot{}n,
\]
then by adapting the construction of \citet{ross2010efficient}, one
can show that in general the algorithm can have $
J(\pistar)-J(\pihat)\geq{}RH\cdot\vepsopt$, meaning it no longer
achieves horizon dependence. Indeed, in this case, it is no longer
possible to translate the bound on $\Lhatbc(\pihat)$ to a bound on the
trajectory-level loss in  \cref{eq:traj_empirical} without incurring
an $H$ factor. On the other hand, as we show in \cref{sec:additional},
if the \loglossbc objective is solved only approximately, i.e.
\[
  \sum_{i=1}^{n}\sum_{h=1}^{H}\log\prn*{\frac{1}{\pihat_h(a_h\ind{i}\mid{}x_h\ind{i})}}
\leq{}
\inf_{\pi\in\Pi}\sum_{i=1}^{n}\sum_{h=1}^{H}\log\prn*{\frac{1}{\pi_h(a_h\ind{i}\mid{}x_h\ind{i})}}
+\vepsopt\cdot{}n,
\]
the regret of the algorithm degrades only to $    J(\pistar)-J(\pihat)
\approxleq{}
R\cdot{}\prn*{\frac{\log(\abs{\Pi}\delta^{-1})}{n}+\vepsopt}$, and
thus remains horizon-independent. Similar remarks apply to the case of
misspecification. Of course, perhaps the greatest advantage of
\loglossbc is that it readily supports stochastic policies, and is far
more practical to implement.

\subsection{Conclusion and Future Work}}

\arxiv{Our results clarify the role of horizon in offline and online
imitation learning, and show that---at least under standard
assumptions in theoretical research into imitation learning---the gap
between online and offline IL is smaller than previously thought. Instabilities of
  offline IL \citep{muller2005off,de2019causal,block2023butterfly} and
  benefits of online IL \citep{ross2013learning,kim2013learning,gupta2017cognitive,bansal2018chauffeurnet,choudhury2018data,kelly2019hg,barnes2023world,zhuang2023robot,lum2024dextrah}
  likely arise in practice,
  but existing assumptions in theoretical research on imitation learning
  appear be too coarse to give insights into the true nature of these
  phenomena, highlighting the need to
  develop a fine-grained, problem-dependent understanding of
  algorithms and complexity for IL. To this end, natural directions
  include (i) Building upon the initial results in \cref{sec:online}, and
    investigating new mechanisms such as exploration and
    representational benefits through which online IL can improve over
    offline IL;  (ii) Developing and analyzing imitation learning algorithms under
    control-theoretic assumptions that more directly capture practical
    notions of instability
    \citep{pfrommer2022tasil,tu2022sample,block2023butterfly,block2024provable};
    (iii) Developing a more refined theory in the
  context of language models, via the connection in
  \cref{sec:autoregressive}. For the latter two directions, an important question
  is to understand whether the notion of supervised learning error
  $\Dhels{\bbP^{\pihat}}{\bbP^{\pistar}}$ we consider is a suitable
  proxy for real-world performance, or whether more refined notions are required.\loose

 Our results also highlight the importance of
developing \emph{learning-theoretic} foundations for imitation
learning that support general, potentially neural function classes. To
this end, a natural question left open by our work is to develop
complexity measures (analogous to VC dimension or Rademacher complexity) that characterize the minimax sample complexity of
online and offline IL for \emph{any policy class}.\loose
}

\arxiv{\subsubsection*{Additional Results}}
Secondary results deferred to the appendix for space include (i)
examples and additional guarantees for \loglossbc and \loglossdagger
(\cref{sec:additional}); and (ii) additional lower bounds and results
concerning the tightness of
\cref{prop:lb_deterministic,thm:bc_stochastic} (\cref{sec:additional_lower}).

\subsubsection*{Acknowledgements}
We thank Jordan Ash, Audrey Huang, Akshay Krishnamurthy, Max
Simchowitz, and Cyril Zhang for many helpful discussions. We thank
Drew Bagnell for valuable comments and pointers to related work.\loose

\arxiv{
\clearpage
}

\bibliography{refs}

\clearpage

\appendix

\renewcommand{\contentsname}{Contents of Appendix}
\addtocontents{toc}{\protect\setcounter{tocdepth}{2}}
{
  \hypersetup{hidelinks}
  \tableofcontents
}
\clearpage

\section{Additional Related Work}
\label{sec:related}

\subsection{Theory of Imitation Learning and Reinforcement Learning}

Classical theoretical works in imitation learning, beginning from the
work of \citet{ross2010efficient} observes that behavior cloning (for
the specific indicator loss in \arxiv{\cref{eq:bc}}) can incur quadratic
dependence on horizon, and shows that online interaction, via
algorithms like \dagger and \aggrevate, can obtain improved sample
complexity under recoverability-type conditions
\citep{ross2010efficient,ross2011reduction,ross2014reinforcement,sun2017deeply}. Further
works along this line include
\citet{cheng2018convergence,cheng2020policy,cheng2019accelerating,yan2021explaining,spencer2021feedback}.

These papers can be thought of as \emph{supervised learning
    reduction}, in the sense that---in the vein of
  \cref{eq:lbc_rollout}---they guarantee that the imitation learning performance is controlled by
  an appropriate notion of supervised learning performance. Notably, this holds for
  \emph{any policy $\pihat$}, which means that in practice, the rollout performance
  is good whenever supervised learning succeeds, even if we do not
  necessarily have a provable guarantee for the generalization of
  $\pihat$ (e.g., for neural networks, where understanding
  generalization is an active area of research). However, as
noted throughout this paper and elsewhere
\citep{rajaraman2020toward,rajaraman2021value,rajaraman2021provably}, these
works typically state regret guarantees in terms of different, often
incomparable notions of supervised learning performance, and avoid
giving concrete, end-to-end guarantees for specific policy classes of
interest. This can make it challenging to objectively evaluate
optimality, and to understand whether limitations of specific
algorithms are due to suboptimal design choices versus
information-theoretic limitations. For example, \citet{li2022efficient} show that in some cases,
supervised learning oracles that satisfy assumptions required by prior
work do not actually exist.

\paragraph{Minimax sample complexity of imitation learning}
More recently, a line of work beginning with
\citet{rajaraman2020toward} revisits the minimax sample complexity of
imitation learning, aiming to provide end-to-end sample complexity
guarantees and lower bounds, but primarily focused on tabular MDPs and
policies
\citep{rajaraman2020toward,rajaraman2021value,rajaraman2021provably,swamy2022minimax}. Notably,
\citet{rajaraman2020toward} show that when $\Pi$ is the set of all
non-stationary policies in a tabular MDP and $R=H$, online IL methods can
achieve $\bigoh(\murec{}H)$ sample complexity, while offline IL
methods must pay $\bigom(H^2)$; this is consistent with our findings
in \cref{sec:main}, as $\log\abs{\Pi}=\bigom(H)$ for this
setting. Other interesting findings from this line of work include the
observation that when the MDP dynamics are \emph{known}, the sample
complexity for offline IL with non-stationary tabular policies can be
brought down to $\bigoh(H^{3/2})$. As noted in \cref{sec:main},
\citet{rajaraman2021value} show that offline IL methods can obtain
$\bigoh(H)$ sample complexity for \emph{linearly parameterized} policies under
parameter sharing; our analysis of \loglossbc for the special case of
deterministic policies shows that it can be viewed as
\emph{implicitly} minimizing the objective they consider.

\citet{xu2022understanding} also consider the problem of horizon
independence in IL. Their work focuses on tabular MDPs and policies,
and shows with knowledge of the dynamics, it is possible to achieve
horizon dependence for a restricted class of MDPs termed
\emph{RBAS-MDPs}. In contrast, our work achieves horizon independence
for general MDPs, without knowledge of dynamics.

Compared to the works above, we focus on general finite classes
$\Pi$. Various works on theoretical reinforcement learning
\citep{agarwal2019reinforcement,foster2023foundations} have observed
that finite classes are a useful test case for general function
approximation, because
  they are arguably the simplest type of policy class from a generalization
  perspective, yet do not have any additional structure (e.g.,
  linearity) that could lead to spurious conclusions that do not extend
  to rich function classes like neural networks.

Recent work of \citet{tiapkin2023regularized}
provides generalization guarantees for behavior cloning with the
logarithmic loss, but their results scale linearly with the horizon,
and thus cannot give tight guarantees for policy classes with
parameter sharing. In addition, their results are stated in terms of
KL-divergence and, as a consequence, require a lower bound on the
action densities for the policy class under consideration. We expect
that both of these limitations are inherent to KL
divergence. \citet{tiapkin2023regularized} also give
variance-dependent bounds on rollout performance similar to
\cref{thm:bc_stochastic}, but their results require a bound on KL
divergence (which is stronger than a bound on Hellinger distance), and
thus are unlikely to meaningfully capture optimal horizon
dependence. These bounds on rollout performance also do not recover
the notion of variance in \cref{thm:bc_stochastic}.

We also mention in passing \citet{sekhari2024selective}, who consider \emph{active} imitation learning
algorithms, and focus on obtaining improved sample
complexity with respect to dependence on the accuracy $\veps$ (as
opposed to $H$), under strong distributional assumptions in the vein of
active learning \citep{hanneke2014theory}.

  \paragraph{Inverse reinforcement learning}
  A long line of research on \emph{inverse reinforcement learning} and
  related techniques considers a setting in which either a) the
  dynamics of the MDP $\Mstar$ are known, or b) it is possible to
  interact with $\Mstar$ online (without expert
  feedback), with empirical
  \citep{abbeel2004apprenticeship,ziebart2008maximum} and theoretical results
  \citep{syed2007game,syed2008apprenticeship,syed2010reduction,brantley2019disagreement,chang2021mitigating}. This
  setting encompasses generative adversarial imitation learning
  and related moment matching methods
  \citep{ho2016generative,li2017infogail,ke2021imitation,swamy2021moments}. A
  detailed discussion is out of scope for the present work, but we
  believe this framework can improve over the sample complexity of
  offline IL in some but not all situations (e.g., \citet{rajaraman2020toward}).

\paragraph{Benefits of logarithmic loss}
Our work draws inspiration from \citet{foster2021efficient}, who
observed that the logarithmic loss can have benefits over square loss
when outcomes are heteroskedastic, and used this observation to derive
first-order regret bounds for contextual bandits. Subsequent works
have extended their analysis technicals to derive first-order regret bounds in
various reinforcement learning settings
\citep{wang2023benefits,wang2024more,ayoub2024switching}.\footnote{We
  also mention in passing the work of \citet{farebrother2024stop}, which
  observes that switching to the log-loss is beneficial empirically for
  approximate value iteration methods in offline reinforcement learning.}
To the best
of our knowledge, our work is the first to uncover a decision making
setting in which switching to the logarithmic loss is beneficial even
in a minimax sense. We emphasize that while our analysis uses the information-theoretic
machinery introduced in \citet{foster2021efficient} and related work
\citep{foster2021statistical,foster2022complexity}, our results are
quite specialized to structure of the imitation learning setting, and
cannot directly be derived from any of the results in \citet{foster2021efficient, wang2023benefits,wang2024more,ayoub2024switching}.

  \paragraph{Horizon-free reinforcement learning}
Our results also take inspiration from the line of research on
\emph{horizon-independent} sample complexity bounds for reinforcement
learning
\citep{jiang2018open,zanette2019tighter,wang2020long,zhang2021reinforcement,zhang2022horizon},
as well as a closely related line of research on variance-dependent
regret bounds
\citep{zhou2023sharp,zhao2023variance,wang2024more}.\footnote{Compared
  to variance-dependent bounds for RL in
  \citet{zhou2023sharp,zhao2023variance,wang2024more} an interesting
  feature of \cref{thm:bc_stochastic} is that it only depends on
  variance for $\pistar$, whereas these works typically depend on
  worst-case variance over all policies or similar quantities.}
These
papers provide sample complexity bounds for reinforcement learning that have
little or no explicit dependence on horizon whenever rewards are
normalized such that $\sum_{h=1}^{H}r_h\in\brk{0,1}$. We consider a simpler setting (imitation learning), but
provide guarantees that hold under \emph{general function
  approximation}, while the works above are restricted to either
tabular MDPs or MDPs with linear/low-rank structure. Nonetheless, our
proof of \cref{thm:bc_stochastic} makes use of concentration arguments
inspired by \citet{zhang2021reinforcement,zhang2022horizon}.

\subsection{Empirical Research on Imitation Learning}

Many empirical works have observed compounding error in behavior
cloning. Outside of online imitation learning, mitigations include
noise injection at data collection time
\citep{laskey2017dart,ke2021grasping} or inverse RL methods that
assume knowledge of system dynamics \citep{ziebart2008maximum}. Other
works take a control-theoretic perspective
\citep{tu2022sample,havens2021imitation,pfrommer2022tasil,block2024provable},
and augment behavior cloning with techniques designed to ensure
incremental stability (or other control-theoretic notions of
stability) of system.

\paragraph{Online imitation learning}
Many empirical works have noted benefits of online imitation learning
methods like \dagger over classical behavior cloning \citep{ross2013learning,kim2013learning,gupta2017cognitive,bansal2018chauffeurnet,choudhury2018data,kelly2019hg,barnes2023world,zhuang2023robot,lum2024dextrah}. These
results are not in contradiction to our findings, as they typically do
not ablate the effect of the loss function (e.g.,
\cite{ross2010efficient} uses the squared hinge loss,
\citet{ross2011reduction} uses the hinge loss, and
\citet{ross2013learning} uses the square loss). It is also possible
that the perceived benefits arise from factors beyond horizon (e.g.,
representational benefits), as discussed in \cref{sec:online}.

\subsection{Autoregressive Language Modeling}
\label{sec:autoregressive}
\newcommand{\prompt}{z}

Autoregressive language modeling with the standard next-token prediction objective
\citep{radford2019language} can be viewed as an instance of behavior
cloning with the logarithmic loss. In this setting, $\Mstar$ corresponds to a \emph{token-level MDP}. Here $\cA$ is a space or vocabulary of \emph{tokens} The initial state is $x_1=\prompt\sim{}P_0$, where $\prompt$ is a \emph{prompt} or \emph{context}. Given the prompt, for each $h=1,\ldots,H$ the action $a_h\in\cA$ is a new token, which is concatenated to the state via the deterministic dynamics $x_{h+1} \gets (z,a_{1:h})$. Via Bayes' rule, an expert policy
\[
\pistar(a_{1:H}\mid{}z) = \prod_{h=1}^{H}\pistar_h(a_h\mid{}z,a_{1:h-1}) = \prod_{h=1}^{H}\pistar_h(a_h\mid{}x_h)
\]
can represent an arbitrary conditional distribution over sequences, from which a training set $\cD=\crl*{o\ind{i}}$ with $o\ind{i}=(z\ind{i},a\ind{i}_1,\ldots,a\ind{i}_H)$ is generated. With this setup, log-loss behavior cloning
\[
  \pihat=\argmax_{\pi\in\Pi}
  \sum_{i=1}^{n}\sum_{h=1}^{H}\log(\pi_h(a_h\ind{i}\mid{}z\ind{i},a_{1:h-1}\ind{i}))
\]
is equivalent to the standard next-token prediction objective for unsupervised language model pre-training \citep{radford2019language}, with the class $\Pi$ parameterized by a transformer or a similar neural net architecture. In this context, long-range error amplification arising from the next-token prediction objective (often referred to as \emph{exposure bias}) has been widely observed by prior work \citep{holtzman2019curious,braverman2020calibration,block2023butterfly}, and in some cases speculated to be a fundamental limitation \citep{lecun2023large,bachmann2024pitfalls}.

\paragraph{Applying our results}
To apply our results, consider a fixed reward function $r=\crl*{r_h}_{h=1}^{H}$, which might measure performance for a particular task of interest (e.g., question answering or commonsense reasoning). Then, for a model $\pi$, $J(\pi)$ corresponds to rollout performance at the task for an autoregressively generated sequence (i.e., given $z\sim{}P_0$, we sample $a_h\sim{}\pi_h(\cdot\mid{}z,a_{1:h-1})$ for all $h\in\brk{H}$). For this setting, \cref{thm:bc_stochastic} states that
    \begin{align}
    \label{eq:bc_stochastic}
    J(\pistar)-J(\pihat)
    \leq \bigoht\prn*{\sqrt{\sigmastar^2\cdot\DhelsX{\big}{\bbP^{\pihat}}{\bbP^{\pistar}}}
      + R\cdot\DhelsX{\big}{\bbP^{\pihat}}{\bbP^{\pistar}}
      },
    \end{align}
    where $\sigmastar^2=\sum_{h=1}^{H}\En^{\pistar}\brk*{(\Qstar_h(z,a_{1:h})-\Vstar_h(z,a_{1:h-1}))^2}$.
    In particular, as long as the cumulative reward for the task is bounded by $R=\bigoh(1)$ (e.g., if we receive an episode-level reward $r_H=1$ if a question is answered correctly, and receive zero reward otherwise), the rollout performance has \emph{no explicit dependence on the sequence length}, except through the generalization error $\DhelsX{\big}{\bbP^{\pihat}}{\bbP^{\pistar}}$. In light of this result, we expect that error amplification observed in practice may arise from challenges in minimizing the generalization error $\DhelsX{\big}{\bbP^{\pihat}}{\bbP^{\pistar}}$ itself (e.g., architecture, data generation process, optimization \citep{braverman2020calibration,block2023butterfly}), rather than fundamental limits of next-token prediction.

\section{Technical Tools}
\label{sec:technical}

\subsection{Tail Bounds}

      \begin{lemma}[e.g., \citet{foster2021statistical}]
    \label{lem:martingale_chernoff}
    For any sequence of real-valued random variables $\prn{X_t}_{t\leq{}T}$ adapted to a
    filtration $\prn{\filt_t}_{t\leq{}T}$, it holds that with probability at least
    $1-\delta$, for all $T'\leq{}T$,
    \begin{equation}
      \label{eq:martingale_chernoff}
\sum_{t=1}^{T'}-\log\prn*{\En_{t-1}\brk*{e^{-X_t}}}      \leq \sum_{t=1}^{T'}X_t
      + \log(\delta^{-1}).
    \end{equation}
  \end{lemma}

  \begin{lemma}[Time-uniform Freedman-type inequality]
  \label{lem:freedman}
  Let $(X_t)_{t\leq{T}}$ be a real-valued martingale difference
  sequence adapted to a filtration $\prn{\filt_t}_{t\leq{}T}$. If
  $\abs*{X_t}\leq{}R$ almost surely, then for any $\eta\in(0,1/R)$,
  with probability at least $1-\delta$, for all $T'\leq{}T$.
    \[
      \sum_{t=1}^{T'}X_t \leq{} \eta\sum_{t=1}^{T'}\En_{t-1}\brk*{X_t^{2}} + \frac{\log(\delta^{-1})}{\eta}.
    \]
  \end{lemma}
  \begin{proof}[\pfref{lem:freedman}]
    Let $S_t=\sum_{s=1}^{t}X_t$ and
    $V_t=\sum_{s=1}^{t}\En_{t=1}\brk*{X_t^2}$. 
    Let $Z_t = \exp(\eta{}S_t-\eta^2V_t)$. As shown in
    \citet{beygelzimer2011contextual} (see proof of Theorem 1), as long as $\eta\leq{}1/R$,
    \begin{align}
      \En_{t-1}\brk*{\exp\prn*{\eta{}X_t}}
      \leq{} \exp(\eta{}^2\En_{t-1}\brk*{X_t^2}),
    \end{align}
    and so
    \begin{align}
      \En_{t-1}\brk*{Z_t}
      =\En_{t-1}\brk*{\exp\prn*{\eta{}X_t-\eta^2\En_{t-1}\brk*{X_t^2}}}\cdot{}Z_{t-1}
      \leq{}Z_{t-1}.
    \end{align}
    It follows that $(Z_t)$ is a non-negative supermartingale. Hence,
    by Ville's inequality, for any
    $\eta\in(0,1/R)$, we have that for any $\tau>0$,
    \begin{align}
      \bbP\brk*{\exists t : S_t-\eta{}V_t\geq{}\tau}
      =       \bbP\brk*{\exists t : Z_t\geq{}e^{\eta{}\tau}}
      \leq e^{-\eta\tau}\En\brk*{Z_T}\leq{}e^{-\eta\tau}.
    \end{align}
    We conclude by setting $\tau=\log(\delta^{-1})/\eta$.
  \end{proof}

  The following result is a standard consequence of \cref{lem:freedman}.
        \begin{lemma}
      \label{lem:multiplicative_freedman}
            Let $(X_t)_{t\leq{T}}$ be a sequence of random
      variables adapted to a filtration $\prn{\filt_{t}}_{t\leq{}T}$. If
  $0\leq{}X_t\leq{}R$ almost surely, then with probability at least
  $1-\delta$, for all $T'\leq{}T$,
  \begin{align}
    &\sum_{t=1}^{T'}X_t \leq{}
                        \frac{3}{2}\sum_{t=1}^{T'}\En_{t-1}\brk*{X_t} +
                        4R\log(2\delta^{-1}),
    \intertext{and}
      &\sum_{t=1}^{T'}\En_{t-1}\brk*{X_t} \leq{} 2\sum_{t=1}^{T'}X_t + 8R\log(2\delta^{-1}).
  \end{align}
    \end{lemma}

    \subsection{Information Theory}
    For a pair of probability measures $\bbP$ and $\bbQ$, we define
    the total variation distance as
    $\Dtv{\bbP}{\bbQ}=\frac{1}{2}\int\abs{\mathrm{d}\bbP-\mathrm{d}\bbQ}$, and define
    the $\chi^2$-divergence by
    $\Dchis{\bbP}{\bbQ}=\int\frac{(\mathrm{d}\bbQ-\mathrm{d}\bbQ)^2}{\mathrm{d}\bbQ}$ if
    $\bbP\ll\bbQ$ and $\Dchis{\bbP}{\bbQ}=+\infty$ otherwise. We
    define KL divergence by
    $\Dkl{\bbP}{\bbQ}=\int{}\mathrm{d}\bbP\log\prn[\big]{\frac{\mathrm{d}\bbP}{\mathrm{d}\bbQ}}$ if
    $\bbP\ll\bbQ$ and $\Dkl{\bbP}{\bbQ}=+\infty$ otherwise.

The following lemma states some basic inequalities between divergences.
  \begin{lemma}[e.g., \cite{polyanskiy2014lecture}]
    \label{lem:pinsker}
    The following inequalities hold:
    \begin{itemize}
    \item $\Dtvs{\bbP}{\bbQ}\leq\Dhels{\bbP}{\bbQ}\leq{}2\Dtv{\bbP}{\bbQ}$.%
    \item $\frac{1}{6}\Dhels{\bbP}{\bbQ}
      \leq{} \Dchis{\bbP}{\frac{1}{2}(\bbP+\bbQ)}
      \leq{} \Dhels{\bbP}{\bbQ}
      $.
    \end{itemize}
  \end{lemma}

\subsection{Reinforcement Learning}

The following lemma is a somewhat standard result; see, e.g., Lemma 15 in \citet{zanette2019tighter}. We include a proof for completeness.
\begin{lemma}[Law of total variance]
  \label{lem:ltv}
  For any (potentially stochastic) policy $\pi$, we have
  \begin{align}
    \Var^{\pi}\brk*{\sum_{h=1}^{H}r_h}
    = \En^{\pi}\brk*{\sum_{h=0}^{H}\Var^{\pi}\brk*{r_h + V_{h+1}^{\pi}(x_{h+1})\mid{}x_{h}}},
  \end{align}
  with the convention that $x_0$ is a deterministic dummy state (so
  that $P_0(x_1=\cdot\mid{}x_0,a=\cdot)$ is the initial state distribution)
  and $r_0=0$.
\end{lemma}
\begin{proof}[\pfref{lem:ltv}]
  Let $h\in\crl{0,\ldots,H}$ be fixed. We can expand
  \begin{align}
    \Var^{\pi}\brk*{\sum_{\ell=h}^{H}r_\ell\mid{}x_h}
    &= \En^{\pi}\brk*{\prn*{\sum_{\ell=h}^{H}r_\ell  - V_h^{\pi}(x_h)}^2\mid{}x_h}\\
    &= \En^{\pi}\brk*{\prn*{\sum_{\ell=h+1}^{H}r_\ell  -
      V_{h+1}^{\pi}(x_{h+1}) +
      (r_h+V_{h+1}^{\pi}(x_{h+1})-V^{\pi}_{h}(x_h))}^2\mid{}x_h}\\
        &= \En^{\pi}\brk*{\prn*{\sum_{\ell=h+1}^{H}r_\ell  -
          V_{h+1}^{\pi}(x_{h+1})}^2\mid{}x_h}
          +
          \En^{\pi}\brk*{(r_h+V_{h+1}^{\pi}(x_{h+1})-V^{\pi}_{h}(x_h))^2\mid{}x_h}\\
    &~~~~+2\En^{\pi}\brk*{\prn*{\sum_{\ell=h+1}^{H}r_\ell  -
      V_{h+1}^{\pi}(x_{h+1})}\prn*{r_h+V_{h+1}^{\pi}(x_{h+1})-V^{\pi}_{h}(x_h)}\mid{}x_h}\\
            &= \En^{\pi}\brk*{\prn*{\sum_{\ell=h+1}^{H}r_\ell  -
          V_{h+1}^{\pi}(x_{h+1})}^2\mid{}x_h}
          +
              \En^{\pi}\brk*{(r_h+V_{h+1}^{\pi}(x_{h+1})-V^{\pi}_{h}(x_h))^2\mid{}x_h}\\
    &= \En^{\pi}\brk*{\Var^{\pi}\brk*{\sum_{\ell=h+1}^{H}r_\ell\mid{}x_{h+1}}\mid{}x_h}
          +
              \Var^{\pi}\brk*{(r_h+V_{h+1}^{\pi}(x_{h+1})\mid{}x_h}.
  \end{align}
  We conclude inductively that for all $h\in\crl*{0,\ldots,H}$,
  \begin{align}
    \Var^{\pi}\brk*{\sum_{\ell=h}^{H}r_\ell\mid{}x_h}
    = \sum_{\ell=h}^{H}\En^{\pi}\brk*{\Var^{\pi}\brk*{(r_\ell+V_{\ell+1}^{\pi}(x_{\ell+1})\mid{}x_\ell}\mid{}x_h}.
  \end{align}
To obtain the final expression, we note that
  \begin{align}
    \Var^{\pi}\brk*{\sum_{h=1}^{H}r_h}
    =       \Var^{\pi}\brk*{\sum_{h=0}^{H}r_h\mid{}x_0},
  \end{align}
  under the convention that $x_0$ is a deterministic dummy state (so
  that $P_1(x_1=\cdot\mid{}x_0,a)$ is the initial state distribution)
  and $r_0=0$.
  
\end{proof}

\subsection{Maximum Likelihood Estimation}
\label{sec:mle}
\newcommand{\cNlog}{\cN_{\mathrm{log}}}

This section presents a self-contained analysis of the maximum likelihood estimator (MLE) for density estimation. The results are somewhat standard (e.g., \citet{wong1995probability,Sara00,zhang2006from}), but we include proofs for completeness.

Consider a setting where we receive
$\crl*{z\ind{i}}_{i=1}^{n}$ \iid from $z\sim{}\gstar$,
where $\gstar\in\Delta(\cZ)$. We have a
class $\cG\subseteq\Delta(\cZ)$ that may or may not contain
$\gstar$. We analyze the following maximum likelihood estimator:
\begin{align}
  \label{eq:mle}
  \ghat = \argmax_{g\in\cG}\sum_{i=1}^{n}\log(g(z\ind{i})).
\end{align}
To provide sample complexity guarantees that support infinite classes,
we appeal to the following notion of covering number (e.g., \citet{wong1995probability}), which tailored to the log-loss.
\begin{definition}[Covering number]
  \label{def:log_cover}
  For a class $\cG\subset\Delta(\cZ)$, we set that a class $\cG'\subset\Delta(\cZ)$ is an $\veps$-cover if for all $g\in\cG$, there exists $g'\in\cG'$ such that for all $z\in\cZ$, $\log(g(z)/g'(z))\leq{}\veps$. We denote the size of the smallest such cover by $\cNlog(\cG,\veps)$.
\end{definition}
We also allow for optimization errors, and concretely assume that
$\ghat$ satisfies
\begin{align}
\sum_{i=1}^{n}\log(\ghat(z\ind{i}))  \geq
  \max_{g\in\cG}\sum_{i=1}^{n}\log(g(z\ind{i})) - \vepsopt\cdot{}n
\end{align}
for a parameter $\vepsopt\geq{}0$; the case $\vepsopt=0$ coincides
with \cref{eq:mle}. Our main guarantee for MLE is as follows.
\begin{proposition}
  \label{thm:mle}
  The maximum likelihood estimator in \cref{eq:mle} has that with probability at least $1-\delta$,
  \begin{align}
    \Dhels{\ghat}{\gstar}
    \leq{}
    \inf_{\veps>0}\crl*{\frac{6\log(2\cNlog(\cG,\veps)/\delta^{-1})}{n}
    + 4\veps}+ 2\inf_{g\in\cG}\log(1+\Dchis{\gstar}{g}) + 2\vepsopt.
  \end{align}
  
In particular, if $\cG$ is finite, the maximum likelihood estimator satisfies
  \begin{align}
    \Dhels{\ghat}{\gstar}
    \leq{} \frac{6\log(2\abs{\cG}/\delta^{-1})}{n} +
    2\inf_{g\in\cG}\log(1+\Dchis{\gstar}{g}) + 2\vepsopt.
  \end{align}
\end{proposition}
Note that the term $\inf_{g\in\cG}\log(1+\Dchis{\gstar}{g})$ corresponds to misspecification error, and is zero if $\gstar\in\cG$.

\begin{proof}[\pfref{thm:mle}]%
  \newcommand{\Lhat}{\wh{L}}%
  \newcommand{\cGveps}{\cG_{\veps}}%
  \newcommand{\gtil}{\wt{g}}%
  Let $\cGveps$ denote a minimal $\veps$-cover for $\cG$, and let $\gtil\in\cGveps$ denote any element that covers $\ghat$ in the sense of \cref{def:log_cover}. Going forward, we will use that $\gtil$ satisfies
  \begin{align}
    \label{eq:gtil_hellinger}
    \Dhels{\gstar}{\gtil}\leq\Dkl{\gstar}{\gtil}\leq{}\veps.
  \end{align}

  Let $\ell\ind{i}(g) = -\log(g(z\ind{i}))$, and set $\Lhat(g) =
  -\sum_{i=1}^{n}\log(g(z\ind{i}))$. Set
  $X_i(g)=\frac{1}{2}(\ell\ind{i}(g)-\ell\ind{i}(\gstar))$. By
  applying \cref{lem:martingale_chernoff} with the sequence $(X_i(g))_{i=1}^{n}$ for each $g\in\cGveps$ and taking a union bound, we have
  that with probability at least $1-\delta$, for all $g\in\cGveps$
  \begin{align}
    -n\cdot{}\log\prn*{\En_{z\sim\gstar}\brk*{e^{\frac{1}{2}\log(g(z)/\gstar(z))}}}
    \leq{} \frac{1}{2}\prn*{\Lhat(g)-\Lhat(\gstar)}
    + \log(\abs{\cG_\veps}\delta^{-1}).
  \end{align}
  Using a standard argument \citep{zhang2006from}, we have that
  \begin{align}
    -\log\prn*{\En_{z\sim\gstar}\brk*{e^{\frac{1}{2}\log(g(z)/\gstar(z))}}}
    = -\log\prn*{1-
\frac{1}{2}\Dhels{g}{\gstar}
    }
    \geq{} \frac{1}{2}\Dhels{g}{\gstar}.
  \end{align}
In particular, this implies that 
  \begin{align}
    \Dhels{\gtil}{\gstar}
    \leq{} \frac{2\log(\abs{\cG}/\delta^{-1})}{n}
    + \frac{1}{n}\prn*{\Lhat(\gtil)-\Lhat(\gstar)},
  \end{align}
  and so
  \begin{align}
    \Dhels{\ghat}{\gstar}
    \leq{} 2     \Dhels{\ghat}{\gtil}
    + 2\Dhels{\gtil}{\gstar}
    \leq{} \frac{4\log(\abs{\cG}/\delta^{-1})}{n} 
    + \frac{2}{n}\prn*{\Lhat(\gtil)-\Lhat(\gstar)} + 2\veps,
  \end{align}
  by the triangle inequality for Hellinger distance and \cref{eq:gtil_hellinger}.

  It remains to bound the right-hand-side. Let $\gbar\in\cG$ be arbitrary. We can bound
  \begin{align}
    \label{eq:mle1}
    \Lhat(\gtil)-\Lhat(\gstar)
    \leq{}     \Lhat(\gtil)-\Lhat(\ghat)
    + \Lhat(\ghat)-\Lhat(\gstar)
    \leq{} \Lhat(\gtil)-\Lhat(\ghat) + \Lhat(\gbar)-\Lhat(\gstar) + \vepsopt{}n,
  \end{align}
  by the definition of the maximum likelihood estimator. For the first term in \cref{eq:mle1}, we observe that
  \begin{align}
    \Lhat(\gtil)-\Lhat(\gstar)
    = \sum_{i=1}^{n}\log(\gstar(z\ind{i})/\gtil(z\ind{i}))\leq{}\veps{}n,
  \end{align}
  by \cref{def:log_cover}.
  
  To bound the second term in \cref{eq:mle1}, set $Y_i = -(\ell\ind{t}(\gbar)-\ell\ind{t}(\gstar))$. Applying
  \cref{lem:martingale_chernoff} with the sequence $(Y_i)_{i=1}^{n}$, we have that with probability at
  least $1-\delta$,
  \begin{align}
    \Lhat(\gbar)-\Lhat(\gstar)
    \leq{} n\cdot\log\prn*{\En_{z\sim\gstar}\brk*{e^{\log(\gstar(z)/\gbar(z))}}} + \log(\delta^{-1}).
  \end{align}
  Finally, note that
  \begin{align}
    \log\prn*{\En_{z\sim\gstar}\brk*{e^{\log(\gstar(z)/\gbar(z))}}}
    = \log\prn*{\En_{z\sim\gstar}\brk*{\frac{\gstar(z)}{\gbar(z)}}}
    = \log(1+\Dchis{\gstar}{\gbar}).
  \end{align}
  The result follows by choosing $\gbar\in\cG$ to minimize this quantity.
  
\end{proof}

\part{Proofs and Supporting Results}

\section{Examples and Supporting Results from \creftitle{sec:main} and
  \creftitle{sec:stochastic}}
\label{sec:additional}

This section contains supporting results from \cref{sec:main,sec:stochastic}:
\begin{itemize}
\item \cref{sec:bc_examples} presents general sample complexity
  guarantees for log-loss behavior cloning that support infinite
  policy classes and misspecification, as well as concrete examples.
\item \cref{sec:dagger} formally introduces the online imitation
  learning framework, and gives sample complexity guarantees for a
  log-loss variant of \dagger.
\end{itemize}

\subsection{General Guarantees and Examples for Log-Loss Behavior Cloning}
\label{sec:bc_examples}
\newcommand{\cNpol}{\cN_{\mathrm{pol}}}
\newcommand{\cNval}{\cN_{\mathrm{val}}}
\newcommand{\rstar}{r^{\star}}

In this section, we give bounds on the generalization error
$\DhelsX{\big}{\bbP^{\pihat}}{\bbP^{\pistar}}$ for log-loss behavior
cloning for concrete classes $\Pi$ of interest. To do so, we observe
that the log-loss behavior cloning objective
\[
\pihat=\argmax_{\pi\in\Pi}\sum_{i=1}^{n}\sum_{h=1}^{H}\log\prn*{\pi_h(a_h\ind{i}\mid{}x_h\ind{i})}.
\]
is equivalent to performing maximum likelihood estimation over the
\emph{density class} $\cP=\crl*{\bbP^{\pi}}_{\pi\in\Pi}$. Indeed, for
any $\pi\in\Pi$, we have
\begin{align}
  \sum_{i=1}^{n}\log(\bbP^{\pi}(o\ind{i}))
  &=
\sum_{i=1}^{n}\log\prn*{P_0(x_1\ind{i})\prod_{h=1}^{H}P_h(x_{h+1}\ind{i}\mid{}x_h\ind{i},a_h\ind{i})\pi_h(a_h\ind{i}\mid{}x_h\ind{i})}\\
  &=
  \sum_{i=1}^{n}\sum_{h=1}^{H}\log\prn*{\pi_h(a_h\ind{i}\mid{}x_h\ind{i})}
  + C(\cD),
\end{align}
where $C(\cD)$ is a constant that depends on the dataset $\cD$ but not
on $\pi$. It follows that both objectives have the same
maximizer. Consequently, we can prove sample complexity bounds for
log-loss behavior cloning by specializing sample complexity bounds for
MLE given in \cref{sec:mle}.

To give guarantees that support infinite policy classes, we appeal to
the following notion of covering number.
\begin{definition}[Policy covering number]
  \label{def:policy_cover}
  For a class $\Pi\subset\crl*{\pi_h:\cX\to\Delta(\cA)}$, we set that $\Pi'\subset\crl*{\pi_h:\cX\to\Delta(\cA)}$
  is an $\veps$-cover if for all $\pi\in\Pi$, there exists
  $\pi'\in\Pi'$ such that for all $x\in\cX$, $a\in\cA$, and $h\in\brk{H}$, $\log(\pi_h(a\mid{}x)/\pi_h'(a\mid{}x))\leq{}\veps$. We denote the size of the smallest such cover by $\cNpol(\Pi,\veps)$.
\end{definition}
In addition, to allow for optimization errors, we replace
\cref{eq:log_loss_bc} with the assumption that $\pihat$ satisfies
\begin{equation}
  \label{eq:loglossbc_general}
  \sum_{i=1}^{n}\sum_{h=1}^{H}\log\prn*{\pihat_h(a_h\ind{i}\mid{}x_h\ind{i})}
\geq{}
\max_{\pi\in\Pi}\sum_{i=1}^{n}\sum_{h=1}^{H}\log\prn*{\pi_h(a_h\ind{i}\mid{}x_h\ind{i})}
-\vepsopt\cdot{}n
\end{equation}
for a parameter $\vepsopt>0$; \cref{eq:log_loss_bc} is the special
case in which $\vepsopt=0$. With these definitions, specializing \cref{thm:mle} leads to the
following result.
\begin{theorem}[Generalization bound for \loglossbc]
  \label{thm:bc_generalization}
  The \loglossbc policy in \cref{eq:loglossbc_general} has that with probability at least $1-\delta$,
  \begin{align}
    \label{eq:bc_cover}
    \Dhels{\bbP^{\pihat}}{\bbP^{\pistar}}
    \leq{}
    \inf_{\veps>0}\crl*{\frac{6\log(2\cNpol(\Pi,\veps{}/H)\delta^{-1})}{n}
    + 4\veps}+
    2\inf_{\pi\in\Pi}\log\prn*{1+\DchisX{\big}{\bbP^{\pistar}}{\bbP^{\pi}}}
    + 2\vepsopt.
  \end{align}
In particular, if $\Pi$ is finite, the log-loss behavior cloning
policy satisfies
\begin{align}
      \Dhels{\bbP^{\pihat}}{\bbP^{\pistar}}
  \leq{} \frac{6\log(2\abs{\Pi}\delta^{-1})}{n} +
  2\inf_{\pi\in\Pi}\log\prn*{1+\DchisX{\big}{\bbP^{\pistar}}{\bbP^{\pi}}}
  + 2\vepsopt.
  \end{align}
\end{theorem}
Let us make two remarks.
\begin{itemize}
\item First, the only explicit dependence on the horizon $H$ is
  through the precision $\veps/H$ through which we evaluate the
  covering number: $\cNpol(\Pi,\veps{}/H)$. As a result, for
  \emph{parametric} classes where
  $\cNpol(\Pi,\veps{})\asymp\log(\veps^{-1})$ (we will give examples
  in the sequel), the sample complexity will scale at most logarithmically in
  $H$, but for nonparametric classes the dependence can be
  polynomial. We leave a detailed understanding of optimal dependence
  on $H$ for nonparametric classes for future work.
\item Second, the remainder term
$\inf_{\pi\in\Pi}\log(1+\Dchis{\bbP^{\pistar}}{\bbP^{\pi}})$
corresponds to misspecification error, and is zero if
$\pistar\in\Pi$.  We remark that when $\pistar$ is deterministic, this
expression can be simplified to
$\inf_{\pi\in\Pi}\log\prn*{\En^{\pistar}\brk*{\frac{1}{\prod_{h=1}^{H}\pi_h(a_h\mid{}x_h)}}}$.
\end{itemize}

\begin{proof}[\pfref{thm:bc_generalization}]
  This follows by applying \cref{thm:mle} with the class
  $\crl*{\bbP^{\pi}}_{\pi\in\Pi}$, and noting that if $\pi'$ covers
  $\pi$ in the sense of \cref{def:policy_cover}, then for all
  $o\in(\cX\times\cA)^{H}$, we have
  $\log(\bbP^{\pi}(o)/\bbP^{\pi'}(o))\leq{}\veps{}H$, meaning that an
  $\veps$-cover in the sense of \cref{def:policy_cover} yields an
  $\veps{}H$-cover in the sense of \cref{def:log_cover}.
  
\end{proof}

\subsubsection{Example: Tabular Policies}%
We now instantiate \cref{thm:bc_generalization} to
give generalization bounds for specific policy classes of interest.

Consider a tabular MDP in which $\abs*{\cX},\abs{\cA}<\infty$ are
small and finite. Here, choosing $\Pi$ to be the set of all stationary
policies leads to a bound independent of $H$.
\begin{corollary}[Stationary tabular policies]
  \label{cor:tabular_stationary}
When $\Pi$ is the set of all deterministic stationary policies, the log-loss behavior cloning policy \cref{eq:log_loss_bc} has that with probability at least $1-\delta$,
\begin{align}
      \DhelsX{\big}{\bbP^{\pihat}}{\bbP^{\pistar}}
  \leq{} \bigoh\prn*{\frac{\abs{\cX}\log(\abs{\cA}\delta^{-1})}{n}}.
\end{align}
Meanwhile, if $\Pi$ is the set of all \emph{stochastic} stationary policies, the log-loss behavior cloning policy \cref{eq:log_loss_bc} has that with probability at least $1-\delta$,
\begin{align}
      \DhelsX{\big}{\bbP^{\pihat}}{\bbP^{\pistar}}
  \leq{} \bigoht\prn*{\frac{\abs{\cX}\abs{\cA}\log(Hn\delta^{-1})}{n}}.
  \end{align}
\end{corollary}
\begin{proof}[\pfref{cor:tabular_stationary}]
This follows by noting that we have
$\log\abs{\Pi}\leq{}\abs{\cX}\log\abs{\cA}$ in the deterministic case
and $\log\cNpol(\Pi,\veps)\leq{}\bigoht\prn*{\abs{\cX}\abs{\cA}\log(\veps^{-1})}$ in
the stochastic case (this follows from a standard discretization
argument, e.g., \citet{wainwright2019high}). 
\end{proof}

Naturally, we can also give generalization guarantees for
non-stationary tabular policies, though the sample complexity will
scale with $H$ in this case.
\begin{corollary}[Non-stationary tabular policies]
  \label{cor:tabular_nonstationary}
When $\Pi$ is the set of all deterministic non-stationary policies, the log-loss behavior cloning policy \cref{eq:log_loss_bc} has that with probability at least $1-\delta$,
\begin{align}
      \DhelsX{\big}{\bbP^{\pihat}}{\bbP^{\pistar}}
  \leq{} \bigoh\prn*{\frac{H\abs{\cX}\log(\abs{\cA}\delta^{-1})}{n}}.
\end{align}
Meanwhile, if $\Pi$ is the set of all \emph{stochastic} non-stationary policies, the log-loss behavior cloning policy \cref{eq:log_loss_bc} has that with probability at least $1-\delta$,
\begin{align}
      \DhelsX{\big}{\bbP^{\pihat}}{\bbP^{\pistar}}
  \leq{} \bigoht\prn*{\frac{H\abs{\cX}\abs{\cA}\log(Hn\delta^{-1})}{n}}.
  \end{align}
\end{corollary}
\begin{proof}[\pfref{cor:tabular_nonstationary}]
This follows because we have
$\log\abs{\Pi}\leq{}H\abs{\cX}\log\abs{\cA}$ in the deterministic case
and $\log\cNpol(\Pi,\veps)\leq{}\bigoht\prn*{H\abs{\cX}\abs{\cA}\log(\veps^{-1})}$ in
the stochastic case.
\end{proof}

\subsubsection{Example: Softmax Policies}
Next, we give an example of a general family of policy classes based on function
approximation for which the sample complexity is at most
polylogarithmic in $H$.

For a vector $v\in\bbR^{\cA}$, let $\sigma:\bbR^{\cA}\to\Delta(\cA)$
be the softmax function, which is given by
\[
  \sigma_a(v) = \frac{\exp(v_a)}{\sum_{a'\in\cA}\exp(v_{a'})}.
\]
Let $\cF\subset\crl*{f_h:\cX\times\cA\to\bbR}_{h=1}^{H}$ be a class of
value functions, and define the induced class of \emph{softmax
  policies} via
\begin{align}
  \Pi_\cF=\crl*{\pi_f\mid{}f\in\cF},  
\end{align}
where
\begin{align}
  \pi_{f,h}(x)\ldef{}\sigma_a(f_h(x,a)).
\end{align}

We give sample complexity guarantees based on covering numbers for the
value function class $\cF$.
\begin{definition}[Value function covering number]
  \label{def:value_cover}
  For a class $\cF\subset\crl*{f_h:\cX\times\cA\to\bbR}$, we set that $\cF'\subset\crl*{f_h:\cX\times\cA\to\bbR}$
  is an $\veps$-cover if for all $f\in\cF$, there exists
  $f'\in\cF'$ such that for all $x\in\cX$, $a\in\cA$, and $h\in\brk{H}$, $\abs{f_h(x,a)-f'_h(x,a)}\leq{}\veps$. We denote the size of the smallest such cover by $\cNval(\Pi,\veps)$.
\end{definition}

\begin{corollary}[Softmax policies]
  \label{cor:softmax}
  When $\Pi=\Pi_\cF$ is the softmax policy class for a value function
  class $\cF$,
  the log-loss behavior cloning policy \cref{eq:log_loss_bc} has that
  with probability at least $1-\delta$,
  \begin{align}
    \Dhels{\bbP^{\pihat}}{\bbP^{\pistar}}
    \leq{} \bigoh(1)\cdot{}\inf_{\veps>0}\crl*{\frac{\log(\cNval(\cF,\veps{}/H)\delta^{-1})}{n} + \veps}+ 2\inf_{\pi\in\Pi_\cF}\log\prn*{1+\DchisX{\big}{\bbP^{\pistar}}{\bbP^{\pi}}}.
  \end{align}
\end{corollary}
\begin{proof}[\pfref{cor:softmax}]
  Consider a pair of functions $f,f'$ with
  $\abs{f_h(x,a)-f'_h(x,a)}\leq{}\veps$ for all $x\in\cX$, $a\in\cA$,
  and $h\in\brk{H}$. The induced softmax policies satisfy
  \begin{align}
    \log(\pi_{f,h}(a\mid{}x)/\pi_{f',h}(a\mid{}x))
    = f_h(x,a)-f'_h(x,a)+ \log\prn*{\frac{\sum_{a'\in\cA}\exp(f'_h(x,a'))}{\sum_{a\in\cA}\exp(f'_h(x,a'))}}.
  \end{align}
  Clearly we have $f_h(x,a)-f'_h(x,a)\leq{}\veps$, and we can bound
  \begin{align}
    \log\prn*{\frac{\sum_{a'\in\cA}\exp(f'_h(x,a'))}{\sum_{a\in\cA}\exp(f_h(x,a'))}}
    &=
      \log\prn*{\frac{\sum_{a'\in\cA}\exp(f_h(x,a'))\cdot{}\exp(f'_h(x,a')-f_h(x,a'))}{\sum_{a\in\cA}\exp(f_h(x,a'))}}\\
    &\leq{}
      \log\prn*{\frac{\sum_{a'\in\cA}\exp(f_h(x,a'))\cdot{}\max_{a''\in\cA}\exp(f'_h(x,a'')-f_h(x,a''))}{\sum_{a\in\cA}\exp(f_h(x,a'))}}\\
    &\leq{}\max_{a''\in\cA}\crl*{f'_h(x,a'')-f_h(x,a'')}\leq{}\veps.
  \end{align}
  Hence, an $\veps$-cover in the sense of \cref{def:value_cover}
  implies a $2\veps$-cover in the sense of \cref{def:policy_cover}.
\end{proof}

Whenever $\cF$ is parametric in the sense that
$\log\cNval(\cF,\veps)\propto{}\log(\veps^{-1})$, \cref{cor:softmax}
leads to polylogarithmic dependence on $H$. The following result gives
such an example.
\paragraph{Linear softmax policies}
Consider the set of stationary linear softmax policies induced by the
value function class
\begin{align}
  \cF=\crl*{(x,a,h)\mapsto{}\tri*{\phi_h(x,a),\theta}\mid{}\nrm*{\theta}_2\leq{}B},
\end{align}
where $\phi_h(x,a)\in\bbR^{d}$ is a known feature map with
$\nrm*{\phi_h(x,a)}\leq{}B$. Here, we have
$\log\cNval(\cF,\veps)\propto{}d\log(B\veps^{-1})$ (e.g.,
\citet{wainwright2019high}), which yields the following generalization guarantee.
\begin{corollary}
  When $\Pi$ is the set of stationary linear softmax policies and
  $\pistar\in\Pi$, the log-loss behavior cloning policy \cref{eq:log_loss_bc} has that
  with probability at least $1-\delta$,
  \begin{align}
      \DhelsX{\big}{\bbP^{\pihat}}{\bbP^{\pistar}}
  \leq{} \bigoh\prn*{\frac{d\log(BHn\delta^{-1})}{n}}.
\end{align}
\end{corollary}

  \subsection{Online IL Framework and Sample Complexity Bounds for Log-Loss Dagger}
  \label{sec:dagger}
  
In this section, we give sample complexity bounds for a variant of the \dagger algorithm for online IL \citep{ross2011reduction} that uses the logarithmic loss. The main purpose of including this result is to give end-to-end sample complexity guarantees for general policy classes, which we use in \cref{sec:main,sec:stochastic} to compare the optimal rates for online and offline IL. For this comparison, we are be mainly interested in the case of deterministic expert policies, but our analysis supports stochastic policies, which may be of independent interest.

\paragraph{Online imitation learning framework}
In the online imitation learning framework, learning proceeds in $n$ episodes in which the
learner can directly interact with the underlying MDP $\Mstar$ and
query the expert advice. Concretely, for each episode $i\in\brk{n}$,
the learner executes a policy $\pi\ind{i}=\crl*{\pi\ind{i}_h:\cX\to\Delta(\cA)}_{h=1}^{H}$ and
      receives a trajectory $o\ind{t} =
      (x\ind{i}_1,a\ind{i}_1,\astari_1),\ldots,(x\ind{i}_H,a\ind{i}_H,\astari_H)$,
      in which $a_h\ind{i}\sim{}\pi_h\ind{i}(x_h\ind{i})$,
      $\astari_h\sim\pistar(x\ind{t}_h)$, and
      $x_{h+1}\ind{i}\sim{}P_h(x_h\ind{i},a_h\ind{i})$; in other
      words, the trajectory induced by the learner's policy is
      annotated by the expert's action $\astar_h\sim{}\pistar_h(x_h)$
      at each state $x_h$ encountered.
      After all $n$ episodes conclude, they can use all of the data
      collected to produce a policy $\pihat$ such that
      $J(\pistar)-J(\pihat)$ is small.

\paragraph{\dagger algorithm}
We consider a general version of the \dagger algorithm. The algorithm is parameterized by an online learning algorithm $\AlgEst$, which attempts to estimate the expert policy in a sequential fashion based on trajectories.

Set $\cD\ind{1}=\emptyset$. For $i=1,\ldots,n$:
\begin{itemize}
\item Query online learning algorithm $\AlgEst$ with $\cD\ind{i}$ and receive policy $\pihat$.
\item Execute $\pihat$ and observe $o\ind{i} =
  (x\ind{i}_1,a\ind{i}_1,\astari_1),\ldots,(x\ind{i}_H,a\ind{i}_H,\astari_H)$.
  \item Update $\cD\ind{i+1}\gets{}\cD\ind{i}\cup\crl{o\ind{i}}$.
\end{itemize}
At the end, we output $\pihat=\unif(\pi\ind{1},\ldots,\pi\ind{n})$ as
the final policy.

To measure the performance of the estimation oracle, we define the
online estimation error as:
\begin{align}
  \EstOnHel(n) = 
  \frac{1}{n}\sum_{i=1}^{n}\sum_{h=1}^{H}\En^{\pihat\ind{i}}\brk*{\Dhels{\pihat\ind{i}_h(x_h)}{\pistar(x_h)}}.
\end{align}
As we will show in a moment, this notion of estimation error is well-suited for online learning algorithms that estimate $\pistar$ using the logarithmic loss.

Our following result gives a general guarantee for \dagger that holds for any choice of online learning algorithm. To state the result, let $\bbP^{\pistar\mid{}\pi}$ denote the law of $o=(x_1,a_1,\astar_1),\ldots,(x_H,a_H,\astar_H)$ when $\pistar$ is the expert policy and we execute $\pi$. Let
  \[
    \sigma^2_{\pistar\mid{}\pi}=\sum_{h=1}^{H}\En^{\pi\circ_h\pistar}\brk*{(\Qstar_h(x_h,a_h)-\Vstar_h(x_h))^2},
  \]
  so that $\sigma^2_{\pistar}=\sigma^2_{\pistar\mid\pistar}$ and define $\sigmab_{\pistar}^2=\sup_{\pi}\sigma^2_{\pistar\mid\pi}$. Note that $\sigmabs^2=0$ whenever $\pistar$ is deterministic, but in general, $\sigmabs^2\geq{}\sigmastar^2$.

\begin{proposition}[Regret for \dagger]
  \label{prop:dagger}
  For any MDP $\Mstar$ with signed recoverability parameter $\mutil$ \arxiv{(\cref{eq:signed_rec})}and any online learning algorithm $\AlgEst$, \dagger ensures that %
  \begin{align}
    \label{eq:dagger_stoch}
    J(\pistar)-J(\pihat) \approxleq{} \sqrt{\sigmabs^2\cdot\EstOnHel(n)}
    + \mutil\cdot{}\EstOnHel(n).
  \end{align}
  Furthermore, whenever $\pistar$ is deterministic, \dagger ensures that
  \begin{align}
    \label{eq:dagger_det}
    J(\pistar)-J(\pihat) \approxleq{}  \mu\cdot{}\EstOnHel(n).
  \end{align}

\end{proposition}

To instantiate the bound above, we choose $\AlgEst$ by applying the exponential weights algorithm (e.g., \citet{cesa2006prediction}) with the logarithmic loss. Let $\Pi_h\ldef{}\crl*{\pi_h\mid{}\pi\in\Pi}$ denote the projection of $\Pi$ onto step $h$. The algorithm proceeds as follows. At step $i\in\brk{n}$, given the dataset $\cD\ind{i}$, for each layer $h\in\brk{H}$ we define a distribution $\mu_h\ind{i}\in\Delta(\Pi_h)$ via
\begin{align}
  \mu_h\ind{i}(\pi) \propto\exp\prn*{\sum_{j<j}\log(\pi_h(a_h^{\star,j}\mid{}x_h\ind{j}))}
  = \prod_{j<j}\pi_h(a_h^{\star,j}\mid{}x_h\ind{j}).
\end{align}
We then set
\begin{align}
\pihat\ind{i}_h(a\mid{}x)=\En_{\pi_h\sim\mu_h\ind{i}}\brk{\pi_h(a\mid{}x)}.
\end{align}
We refer to the resulting algorithm as \loglossdagger. This leads to the following guarantee for finite classes.
\begin{proposition}[Regret for \loglossdagger]
  \label{prop:dagger_finite}
When $\pistar\in\Pi$, the log-loss exponential weights algorithm ensures that with probability at least $1-\delta$,
\begin{align}
  \EstOnHel(n)
  \leq{} \frac{2}{n}\sum_{h=1}^{H}\log(\abs{\Pi_h}{}H\delta^{-1}).
\end{align}
Consequently, \loglossdagger ensures that with probability at least $1-\delta$,
  \begin{align}
    \label{eq:dagger_stoch_finite}
    J(\pistar)-J(\pihat) \approxleq{} \sqrt{\sigmabs^2\cdot \sum_{h=1}^{H}\frac{\log(\abs{\Pi_h}{}H\delta^{-1})}{n}}
    + \mutil\cdot{}\sum_{h=1}^{H}\frac{\log(\abs{\Pi_h}{}H\delta^{-1})}{n},
  \end{align}
  and when $\pistar$ is deterministic, 
  \begin{align}
    \label{eq:dagger_det_finite}
    J(\pistar)-J(\pihat) \approxleq{}  \mu\cdot{}\sum_{h=1}^{H}\frac{\log(\abs{\Pi_h}{}H\delta^{-1})}{n}.
  \end{align}

\end{proposition}

We note that for many parameter regimes, the sample complexity bound in \cref{prop:dagger} can be worse than that of \loglossbc in \cref{thm:bc_stochastic} (for stationary policies, \cref{prop:dagger} has spurious dependence on $H$, and the variance-like quantity in the leading order term is weaker). It would be interesting to get the best of both worlds, though this may require changing the algorithm.

\begin{proof}[\pfref{prop:dagger}]
  Consider an arbitrary policy $\pihat$. Begin by writing
  \begin{align}
    J(\pistar)-J(\pihat)
    = \sum_{h=1}^{H}\En^{\pihat\mid\pihat}\brk*{\Qstar_h(x_h,\pistar_h(x_h))
    -\Qstar_h(x_h,a_h)}.
  \end{align} 
Fix a layer $h$. By \cref{lem:hellinger_com}, we have
\begin{small}
  \begin{align}
    & \En^{\pihat\mid{}\pihat}\brk*{\Qstar_h(x_h,\pistar_h(x_h))
      -\Qstar_h(x_h,a_h)}\\
    &\leq{}
      \En^{\pistar\mid{}\pihat}\brk*{\Qstar_h(x_h,\pistar_h(x_h))
      -\Qstar_h(x_h,a_h)}\\
&~~~~+\sqrt{\prn*{\En^{\pihat\mid{}\pihat}\brk*{(\Qstar_h(x_h,\pistar_h(x_h))
      -\Qstar_h(x_h,a_h))^2}
      + \En^{\pistar\mid{}\pihat}\brk*{(\Qstar_h(x_h,\pistar_h(x_h))
      -\Qstar_h(x_h,a_h))^2}
      }
                              \En^{\pihat}\brk*{\Dhels{\pihat_h(x_h)}{\pistar_h(x_h)}}
                              }\\
    &=\sqrt{\prn*{\En^{\pihat\mid{}\pihat}\brk*{(\Qstar_h(x_h,\pistar_h(x_h))
      -\Qstar_h(x_h,a_h))^2}
      + \En^{\pistar\mid{}\pihat}\brk*{(\Qstar_h(x_h,\pistar_h(x_h))
      -\Qstar_h(x_h,a_h))^2}
      }
                              \En^{\pihat}\brk*{\Dhels{\pihat_h(x_h)}{\pistar_h(x_h)}}
      }.
  \end{align}
\end{small}
  Furthermore, using \cref{lem:hellinger_com}, we have
  \begin{align}
    &\En^{\pihat\mid{}\pihat}\brk*{(\Qstar_h(x_h,\pistar_h(x_h))
      -\Qstar_h(x_h,a_h))^2}\\
    &\approxleq{}
      \sum_{h=1}^{H}
    \En^{\pistar\mid{}\pihat}\brk*{(\Qstar_h(x_h,\pistar_h(x_h))
      -\Qstar_h(x_h,a_h))^2}
      +
      \mutil^2\sum_{h=1}^{H}\En^{\pihat}\brk*{\Dhels{\pihat_h(x_h)}{\pistar_h(x_h)}},
  \end{align}
  so that
  \begin{align}
    \label{eq:dagger1}
    &\En^{\pihat\mid{}\pihat}\brk*{\Qstar_h(x_h,\pistar_h(x_h))
      -\Qstar_h(x_h,a_h)}\\
    &\approxleq{}
    \sqrt{
       \En^{\pistar\mid{}\pihat}\brk*{(\Qstar_h(x_h,\pistar_h(x_h))
      -\Qstar_h(x_h,a_h))^2
      }\cdot
                              \En^{\pihat}\brk*{\Dhels{\pihat_h(x_h)}{\pistar_h(x_h)}}
    }
    + \mutil\cdot \En^{\pihat}\brk*{\Dhels{\pihat_h(x_h)}{\pistar_h(x_h)}}.
  \end{align}
  Recall that the \dagger policy satisfies
  \begin{align}
    J(\pistar)-J(\pihat)
    = \frac{1}{n}\sum_{i=1}^{n}    J(\pistar)-J(\pihat\ind{i}).
  \end{align}
  Applying \cref{eq:dagger1} to each policy $\pihat\ind{i}$, summing over all layer $h$, and applying Cauchy-Schwarz yields
  \begin{align}
    J(\pistar)-J(\pihat)
    \approxleq{} \sqrt{\frac{1}{n}\sum_{i=1}^{n}\sigma^2_{\pistar\mid{}\pi\ind{i}}
    \cdot  \EstOnHel(n)}
    + \mutil\cdot{}\EstOnHel(n)\\
        \approxleq{} \sqrt{\sigmabs^2
    \cdot  \EstOnHel(n)}
    + \mutil\cdot{}\EstOnHel(n).
  \end{align}

  In the deterministic case, we tighten the argument above by applying the following improved change-of-measure argument based on \cref{lem:hellinger_com}:
      \begin{align}
    & \En^{\pistar\mid\pihat}\brk*{\Qstar_h(x_h,\pistar_h(x_h))
      -\Qstar_h(x_h,a_h)}\\
        &\leq{} \En^{\pistar\mid\pihat}\brk*{(\Qstar_h(x_h,\pistar_h(x_h))
      -\Qstar_h(x_h,a_h))_+}\\
    &\leq{} 2 \En^{\pistar\mid\pihat}\brk*{(\Qstar_h(x_h,\pistar_h(x_h))
      -\Qstar_h(x_h,a_h))_+} + \mu\cdot \En^{\pihat}\brk*{\Dhels{\pihat_h(x_h)}{\pistar_h(x_h)}}\\
        &=\mu \cdot\En^{\pihat}\brk*{\Dhels{\pihat_h(x_h)}{\pistar_h(x_h)}},
      \end{align}
      This leads to \cref{eq:dagger_det}.

\end{proof}

\begin{proof}[\pfref{prop:dagger_finite}]
  Since $\pistar\in\Pi$, a standard guarantee for exponential weights with the log-loss (e.g., \citet{cesa2006prediction}) ensures that for all $h\in\brk{H}$, the following bound holds almost surely:
  \begin{align}
    \sum_{i=1}^{n}\log(1/\pihat_h\ind{i}(\astari_h\mid{}x_h\ind{i}))
    \leq{}     \sum_{i=1}^{n}\log(1/\pistar_h(\astari_h\mid{}x_h\ind{i}))
    + \log\abs{\Pi_h}.
  \end{align}
   From here, for each $h\in\brk{H}$, Lemma A.14 of \citet{foster2021statistical} implies that with probability at least $1-\delta$,
  \begin{align}
    \sum_{i=1}^{n}\En^{\pihat\ind{i}}\brk*{\Dhels{\pihat\ind{i}_h(x_h)}{\pistar(x_h)}}
    \leq{} \log\abs{\Pi_h} + 2\log(\delta^{-1}).
  \end{align}
  The result now follows by taking a union bound.
  
\end{proof}

\section{Proofs from \creftitle{sec:main}}
\label{sec:proofs_main}

\subsection{Proof of \creftitle{thm:bc_deterministic}}

\begin{proof}[\pfref{thm:bc_deterministic}]
We begin by defining the following
\emph{trajectory-wise} semi-metric between policies. For a pair of potentially stochastic policies $\pi$ and $\pi'$, define
\begin{align}
  \label{eq:traj_metric}
    \Drho{\pi}{\pi'}\ldef \En^{\pi}\En_{a'_{1:H}\sim\pi'(x_{1:H})}\brk*{\indic\crl*{\exists{}h:\,a_h\neq{}a'_h}},
  \end{align}
  where we use the shorthand $a'_{1:H}\sim\pi'(x_{1:H})$ to indicate
  that $a'_1\sim{}\pi'_1(x_1),\ldots,a'_H\sim{}\pi'_H(x_H)$. Despite being
  defined in an asymmetric fashion, the following lemma shows that the
  trajectory-wise distance
  $\Drho{\cdot}{\cdot}$ is symmetric, from which it follows that it is
  indeed a semi-metric.
  \begin{lemma}
  \label{lem:lmax_swap}
  For all (potentially stochastic) policies $\pi$ and $\pi'$, it holds that
  \begin{align}
\Drho{\pi}{\pi'} = \Drho{\pi'}{\pi}.
  \end{align}
\end{lemma}
Next, we show that it is possible to bound the difference in reward
for any pair of policies in terms of the trajectory-wise distance $\Drho{\cdot}{\cdot}$.
\begin{lemma}
  \label{lem:regret_rho}
    For all (potentially stochastic) policies $\pi$ and $\pi'$, it holds that
  \begin{align}
    J(\pi)-J(\pi') \leq R\cdot{}\Drho{\pi}{\pi'}.
  \end{align}
\end{lemma}
Finally, using \cref{lem:lmax_swap}, we show that when one of the policies is deterministic, the
trajectory-wise distance is equivalent to Hellinger distance up to an
absolute constant.
\begin{lemma}
  \label{lem:hellinger_policy}
  Let $\pistar$ be a deterministic policy and $\pi$ be an arbitrary
  stochastic policy. Then we have that
  \begin{align}
\frac{1}{4}\cdot\Drho{\pistar}{\pi} \leq{}   \Dhels{\bbP^{\pi}}{\bbP^{\pistar}}
    &\leq{} 2\cdot\Drho{\pistar}{\pi}.
  \end{align}
\end{lemma}
Combining \cref{lem:regret_rho,lem:hellinger_policy}, we
conclude that for any deterministic policy $\pistar$ and stochastic
policy $\pihat$,
\begin{align}
        J(\pistar)-J(\pihat)
    \leq
    4R\cdot{}\DhelsX{\big}{\bbP^{\pihat}}{\bbP^{\pistar}}.
\end{align}

\end{proof}

\begin{proof}[\pfref{lem:lmax_swap}]
This follows by noting that we can write
  \begin{align}
    \Drho{\pi}{\pi'}
    &= 1-\En^{\pi}\En_{a'_{1:H}\sim\pi'(x_{1:H})}\brk*{\indic\crl*{a_h=a'_h\;\forall{}h}}\\
    &=1-\sum_{x_{1:H},a_{1:H},a'_{1:H}}
      P_0(x_1)\prod_{h=1}^{H}P_h(x_{h+1}\mid{}x_h,a_h)\pi_h(a_h\mid{}x_h)\pi'_h(a'_h\mid{}x_h)\indic\crl*{a_h=a'_h}\\
    &=1-\sum_{x_{1:H},a_{1:H},a'_{1:H}}
      P_0(x_1)\prod_{h=1}^{H}P_h(x_{h+1}\mid{}x_h,a'_h)\pi_h(a_h\mid{}x_h)\pi'_h(a'_h\mid{}x_h)\indic\crl*{a_h=a'_h}\\
    &=
      1-\En^{\pi'}\En_{a'_{1:H}\sim\pi(x_{1:H})}\brk*{\indic\crl*{a_h=a'_h\;\forall{}h}}=\Drho{\pi'}{\pi}.
  \end{align}

\end{proof}

\begin{proof}[\pfref{lem:regret_rho}]
  Observe that since $\sum_{h=1}^{H}r_h\in\brk{0,R}$, we can bound the
  reward for $\pi$ as 
\begin{align}
  J(\pi)
  &\leq{}
  \En^{\pi}\brk*{\prn*{\sum_{h=1}^{H}r_h}\En_{a'_{1:H}\sim{}\pi'(x_{1:H})}\brk*{\indic\crl*{a'_h=a_h\;\forall{}h}}}
  + R\cdot{}
    \En^{\pi}\En_{a'_{1:H}\sim{}\pi'(x_{1:H})}\brk*{\indic\crl*{\exists{}h:\;a'_h\neq{}a_h}}\\
  &=\En^{\pi}\brk*{\prn*{\sum_{h=1}^{H}r_h}\En_{a'_{1:H}\sim{}\pi'(x_{1:H})}\brk*{\indic\crl*{a'_h=a_h\;\forall{}h}}}
  + R\cdot{}
   \Drho{\pi}{\pi'}.
\end{align}
We can bound the first term as
\begin{align}
&\En^{\pi}\brk*{\prn*{\sum_{h=1}^{H}r_h}\En_{a'_{1:H}\sim{}\pi'(x_{1:H})}\brk*{\indic\crl*{a'_h=a_h\;\forall{}h}}}\\
  &=\En^{\pi}\brk*{f(x_{1:H},a_{1:H}) \En_{a'_{1:H}\sim{}\pi'(x_{1:H})}\brk*{\indic\crl*{a'_h=a_h\;\forall{}h}}},
\end{align}
where
$f(x_{1:H},a_{1:H})\ldef{}\sum_{h=1}^{H}\En\brk*{r_h\mid{}x_h,a_h}$. We
now observe
that for any function $f$,
\begin{align}
  &\En^{\pi}\brk*{f(x_{1:H},a_{1:H}) \En_{a'_{1:H}\sim{}\pi'(x_{1:H})}\brk*{\indic\crl*{a'_h=a_h\;\forall{}h}}}\\
  &= \sum_{x_{1:H},a_{1:H},a'_{1:H}}f(x_{1:H},a_{1:H})
    \cdot{}P_0(x_1)\prod_{h=1}^{H}P_h(x_{h+1}\mid{}x_h,a_h)\pi_h(a_h\mid{}x_h)\pi'_h(a'_h\mid{}x_h)\indic\crl*{a_h=a'_h}\\
    &= \sum_{x_{1:H},a_{1:H},a'_{1:H}}f(x_{1:H},a'_{1:H})
      \cdot{}P_0(x_1)\prod_{h=1}^{H}P_h(x_{h+1}\mid{}x_h,a'_h)\pi_h(a_h\mid{}x_h)\pi'_h(a'_h\mid{}x_h)\indic\crl*{a_h=a'_h}\\
  &\leq \sum_{x_{1:H},a'_{1:H}}f(x_{1:H},a'_{1:H})
    \cdot{}P_0(x_1)\prod_{h=1}^{H}P_h(x_{h+1}\mid{}x_h,a'_h)\pi'_h(a'_h\mid{}x_h)\\
    &=
      \En^{\pi'}\brk*{f(x_{1:H},a_{1:H})}.
\end{align}
We conclude that
\[
  \En^{\pi}\brk*{\prn*{\sum_{h=1}^{H}r_h}\En_{a'_{1:H}\sim{}\pi'(x_{1:H})}\brk*{\indic\crl*{a'_h=a_h\;\forall{}h}}}
  \leq{}J(\pi'),
\]
so that
\begin{align}
  \label{eq:regret_lmax}
    J(\pi) - J(\pi')
\leq{}
  R\cdot{}\Drho{\pi}{\pi'}.
\end{align}
  
\end{proof}

\begin{proof}[\pfref{lem:hellinger_policy}]
  Define the \emph{triangular discrimination} via $\Ddel{\bbP}{\bbQ}\ldef{}\int\frac{(d\bbP-d\bbQ)^2}{d\bbP+d\bbQ}$,
  and recall that
  $\frac{1}{2}\Ddel{\bbP}{\bbQ}\leq\Dhels{\bbP}{\bbQ}\leq{}\Ddel{\bbP}{\bbQ}$
  (e.g., \citet{foster2021efficient}). Next, define the shorthand $P(x_{1:H}\mid{}a_{1:H}) \ldef{}
\prod_{h=0}^{H-1}P(x_{h+1}\mid{}x_h,a_h)$ and
$P^{\pi}(a_{1:H}\mid{}x_{1:H})\ldef\prod_{h=1}^{H}\pi_h(a_h\mid{}x_h)$
(these quantities do not have an interpretation as conditional
probability measures in the way the notation might suggest, but
this will not be relevant to the proof). For any deterministic policy
$\pistar$, we can write
\begin{align}
  &\Ddel{\bbP^{\pi}}{\bbP^{\pistar}} \\
&  =
                                          \sum_{x_{1:H}}\sum_{a_{1:H}}P(x_{1:H}\mid{}a_{1:H-1})\cdot\frac{(P^{\pi}(a_{1:H}\mid{}x_{1:H})-P^{\pistar}(a_{1:H}\mid{}x_{1:H}))^2}{P^{\pi}(a_{1:H}\mid{}x_{1:H})+P^{\pistar}(a_{1:H}\mid{}x_{1:H})}\\
&=\sum_{x_{1:H}}\sum_{a_{1:H}=\pistar(x_{1:H})}P(x_{1:H}\mid{}a_{1:H-1})\cdot\frac{(P^{\pi}(a_{1:H}\mid{}x_{1:H})-P^{\pistar}(a_{1:H}\mid{}x_{1:H}))^2}{P^{\pi}(a_{1:H}\mid{}x_{1:H})+P^{\pistar}(a_{1:H}\mid{}x_{1:H})}
  \\
&~~~~+\sum_{x_{1:H}}\sum_{a_{1:H}\neq\pistar(x_{1:H})}P(x_{1:H}\mid{}a_{1:H-1})\cdot\frac{(P^{\pi}(a_{1:H}\mid{}x_{1:H})-P^{\pistar}(a_{1:H}\mid{}x_{1:H}))^2}{P^{\pi}(a_{1:H}\mid{}x_{1:H})+P^{\pistar}(a_{1:H}\mid{}x_{1:H})}.
\end{align}
Since $\pistar$ is deterministic, $P^{\pistar}(a_{1:H}\mid{}x_{1:H})=1$
if $a_{1:H}=\pistar(x_{1:H})$, and is $P^{\pistar}(a_{1:H}\mid{}x_{1:H})=0$ otherwise.
Using this, we can write the second
term above as
\begin{align}
&\sum_{x_{1:H}}\sum_{a_{1:H}\neq\pistar(x_{1:H})}P(x_{1:H}\mid{}a_{1:H-1})\cdot\frac{(P^{\pi}(a_{1:H}\mid{}x_{1:H})-0)^2}{P^{\pi}(a_{1:H}\mid{}x_{1:H})+0}\\
  &=\sum_{x_{1:H}}\sum_{a_{1:H}\neq\pistar(x_{1:H})}P(x_{1:H}\mid{}a_{1:H-1})P^{\pi}(a_{1:H}\mid{}x_{1:H})\\
  &=\bbP^{\pi}\brk*{\exists{}h:\,a_h\neq\pistar(x_H)} = \Drho{\pi}{\pistar}.
\end{align}
This proves that
$\Ddel{\bbP^{\pi}}{\bbP^{\pistar}}\geq{}\Drho{\pi}{\pistar}$. For the
upper bound, we use that $\pistar$ is deterministic once more to write the first term above as
\begin{align}
&\sum_{x_{1:H}}\sum_{a_{1:H}=\pistar(x_{1:H})}P(x_{1:H}\mid{}a_{1:H-1})\cdot\frac{(P^{\pi}(a_{1:H}\mid{}x_{1:H})-1)^2}{P^{\pi}(a_{1:H}\mid{}x_{1:H})+1}\\
  &
    =\En^{\pistar}\brk*{\frac{(P^{\pi}(a_{1:H}\mid{}x_{1:H})-1)^2}{P^{\pi}(a_{1:H}\mid{}x_{1:H})+1}}
  \leq{} \En^{\pistar}\brk*{(P^{\pi}(a_{1:H}\mid{}x_{1:H})-1)^2}.
\end{align}
We further note that
\begin{align}
  &\En^{\pistar}\brk*{(P^{\pi}(a_{1:H}\mid{}x_{1:H})-1)^2} \\
  &=\En^{\pistar}\En_{a'_{1:H}\sim\pi(x_{1:H})}\brk*{1 +
    (P^{\pi}(a'_{1:H}\mid{}x_{1:H})-2)\indic\crl*{a'_{1:H}=a_{1:H}}}\\
  &\leq{}\En^{\pistar}\En_{a'_{1:H}\sim\pi(x_{1:H})}\brk*{1 -
    \indic\crl*{a'_{1:H}=a_{1:H}}}\\
  &=\En^{\pistar}\En_{a'_{1:H}\sim\pi(x_{1:H})}\brk*{\indic\crl*{\exists{}h:\,a'_{1:H}\neq{}a_{1:H}}}
    = \Drho{\pistar}{\pi}.
\end{align}
By \cref{lem:lmax_swap}, we conclude that
$\Ddel{\bbP^{\pi}}{\bbP^{\pistar}}\leq{}
\Drho{\pi}{\pistar} + \Drho{\pistar}{\pi}=2\Drho{\pistar}{\pi}$.
\end{proof}

\subsection{Proof of \creftitle{prop:lb_deterministic}}

\begin{proof}[\pfref{prop:lb_deterministic}]
    \newcommand{\pia}{\pi\ind{\afrak}}%
    \newcommand{\pib}{\pi\ind{\bfrak}}%
      For this proof, we consider a slightly more general online imitation
      learning model in which the learner is
        allowed to select $a_h\ind{i}$ based on the sequence
        $(x\ind{i}_1,a\ind{i}_1,\astari_1),\ldots,(x\ind{i}_{h-1},a\ind{i}_{h-1},\astari_{h-1}),
        (x_h\ind{i}, \astari_h)$ at training time; this subsumes the offline imitation learning model.
  Let $n\in\bbN$ and $H\in\bbN$ be fixed. Let $\Delta\in(0,1/3)$ be a parameter
  whose value will be chosen later.

  We first specify the dynamics for the reward-free MDP $\Mstar$ and
  the policy class $\Pi$. Set $\cX=\crl*{\xfrak,\yfrak}$
  and $\cA=\crl*{\afrak,\bfrak}$. The initial state distribution sets
  $P_0(\xfrak)=1-\Delta$ and $P_0(\yfrak)=\Delta$. The transition dynamics are
  $P_h(x'\mid{}x,a)=\indic\crl*{x'=x}$ for all $h$; that is,
  $\xfrak,\yfrak$ are self-looping terminal states. We set
  $\Pi=\crl*{\pi\ind{\afrak},\pi\ind{\bfrak}}$, where
  the
  expert policies are $\pi\ind{\afrak}$, which sets
  $\pi_h\ind{\afrak}(x)=\afrak$ for all $h$ and $x$, and
  $\pi\ind{\bfrak}$, which sets $\pi_h\ind{\bfrak}(\xfrak)=\afrak$ and sets
  $\pi_h\ind{\bfrak}(\yfrak)=\bfrak$. 

Let a \emph{problem instance} $\cI=(\Mstar,r,\pistar)$ refer to a tuple
consisting of the reward-free MDP $\Mstar$, a reward function
$r=\crl*{r_h}_{h=1}^{H}$, and an expert policy $\pistar$. We consider
two problem instances, $\cI\ind{\afrak}=(\Mstar,r\ind{\afrak},\pia)$ and
$\cI\ind{\bfrak}=(\Mstar,r\ind{\bfrak},\pib)$:
\begin{itemize}
\item For problem instance $\cI\ind{\afrak}$,
  the expert policy is $\pia$. We set
  $r_h\ind{\afrak}(\xfrak,\cdot)=0$,
  $r_h\ind{\afrak}(\yfrak,a)=\indic\crl*{a=\afrak}$ for all $h$.
\item For problem instance $\cI\ind{\bfrak}$,
  the expert policy is $\pib$. We set
  $r_h\ind{\bfrak}(\xfrak,\cdot)=0$,
  $r_h\ind{\bfrak}(\yfrak,a)=\indic\crl*{a=\bfrak}$ for all $h$.
\end{itemize}
Note that both of these instances satisfy $\mu=1$,
  and that $\pia$ and $\pib$ are optimal policies for their respective
  instances. Let $J\ind{\afrak}$ denote the expected reward function
  for instance $\cI\ind{\afrak}$, and likewise for $\cI\ind{\bfrak}$.

  Going forward, we fix the online imitation learning algorithm under consideration and let
  $\bbP\ind{\afrak}$ denote the law of $o\ind{1},\ldots,o\ind{n}$ when we execute the algorithm on instance $\afrak$,
  and likewise for $\bfrak$; let $\En\ind{\afrak}\brk*{\cdot}$ and
  $\En^{\bfrak}\brk*{\cdot}$ denote the corresponding expectations. In addition, for any policy $\pi$, let
  $\bbP^{\pia\mid{}\pi}$ denote the law of
  $o=(x_1,a_1,\astar_1),\ldots,(x_H,a_H,\astar_H)$ when we execute
  $\pi$ in the online imitation learning framework and the expert
  policy is $\pistar=\pia$, and define $\bbP^{\pib\mid{}\pi}$ analogously.

  We first observe that for any policy $\pihat$,
  \begin{align}
J\ind{\afrak}(\pia) -     J\ind{\afrak}(\pihat)=\Delta\cdot{}\sum_{h=1}^{H}\En_{a_h\sim{}\pihat_h(\yfrak)}\brk*{\indic\crl*{a_h\neq\pia_h(\yfrak)}},
  \end{align}
  and that $J\ind{\bfrak}(\pib) -
  J\ind{\bfrak}(\pihat)=\Delta\cdot{}\sum_{h=1}^{H}\En_{a_h\sim\pihat_h(\yfrak)}\brk*{\indic\crl*{a_h\neq\pib_h(\yfrak)}}$. Defining
  $\rho(\pi,\pi')=\sum_{h=1}^{H}\En_{a_h\sim{}\pi_h(\yfrak),a'_h\sim\pi'_h(\yfrak)}\indic\crl*{a_h\neq{}a'_h}$
  as a metric, we note that $\rho(\pia,\pib)=H$, and hence by the
  standard Le Cam two-point argument (e.g.,. \citet{wainwright2019high}), the algorithm must have
  \begin{align}
    \max\crl*{
    \En\ind{\afrak}\brk*{J\ind{\afrak}(\pia) -     J\ind{\afrak}(\pihat)},
    \En\ind{\bfrak}\brk*{J\ind{\bfrak}(\pib) -     J\ind{\bfrak}(\pihat)}}
    \geq \frac{\Delta{}H}{4}(1-\Dtv{\bbP\ind{\afrak}}{\bbP\ind{\bfrak}}),
  \end{align}
  where $\Dtv{\cdot}{\cdot}$ denotes total variation distance.
  Next, using Lemma D.2 of \citet{foster2024online}, we can bound
  \begin{align}
    \Dtvs{\bbP\ind{\afrak}}{\bbP\ind{\bfrak}}
    \leq{}\Dhels{\bbP\ind{\afrak}}{\bbP\ind{\bfrak}}
    \leq{}7\En\ind{\afrak}\brk*{\sum_{i=1}^{n}
    \Dhels{\bbP^{\pia\mid{}\pi\ind{i}}}{\bbP^{\pib\mid{}\pi\ind{i}}}
    }.
  \end{align}
  Since, the feedback the learner receives for a given episode $i$ is
  identical under instances $\cI\ind{\afrak}$ and $\cI\ind{\bfrak}$
  unless $x_1=\yfrak$ (regardless of how $\pi\ind{i}$ is chosen), we can bound
  \begin{align}
    \Dhels{\bbP^{\pia\mid{}\pi\ind{i}}}{\bbP^{\pib\mid{}\pi\ind{i}}}
    \leq{} 2\Delta,
  \end{align}
  and hence
  \begin{align}
    \Dtvs{\bbP\ind{\afrak}}{\bbP\ind{\bfrak}}
    \leq{} 14\Delta{}n.
  \end{align}
  We set $\Delta=1/56n$, and conclude that any algorithm must have
  \begin{align}
    \max\crl*{
    \En\ind{\afrak}\brk*{J\ind{\afrak}(\pia) -     J\ind{\afrak}(\pihat)},
    \En\ind{\bfrak}\brk*{J\ind{\bfrak}(\pib) -     J\ind{\bfrak}(\pihat)}}
    \geq \frac{\Delta{}H}{8} = c\cdot\frac{H}{n}
  \end{align}
  for an absolute constant $c>0$.
  
\end{proof}

\section{Proofs from \creftitle{sec:stochastic}}

\subsection{Proof of \creftitle{thm:bc_stochastic}}
\label{sec:stochastic_proof}

\begin{proof}[\pfref{thm:bc_stochastic}]%
  Assume without loss of generality that $R=1$. Let
  $o=(x_1,a_1),\ldots,(x_H,a_H)$, and for each $h\in\brk{H}$,
define the sum of advantages up to step $h$ via
  \begin{align}
    \Delta_h(o)
    =
    \sum_{\ell=1}^{h}\prn*{\Qstar_\ell(x_\ell,\pistar_\ell(x_\ell))
    - \Qstar_\ell(x_\ell,a_\ell)},
  \end{align}
  which has $\abs{\Delta(o)}\leq{}H$ almost surely. Consider the
  filtration $\filt_h\ldef{}\sigma(x_1,a_1,\ldots,x_h,a_h)$. Fix a
  parameter $L\geq{}1$ whose value will be chosen later, 
  and define a
  random variable
  \begin{align}
    \Hstar \ldef{} \min\crl*{h\mid{}\abs*{\Delta_h(o)}>L},
  \end{align}
   with $\Hstar\ldef{}H+1$ if there is no
  $h$ such that $\abs*{\Delta_h(o)}>L$; we will adopt the
  convention that $\Qstar_{H+1}=\Vstar_{H+1}=0$.
  \begin{lemma}
    \label{lem:stopping_time}
    $\Hstar$ is a stopping time with respect
    $(\filt_h)_{h\geq{}1}$,\footnote{That is, for all $h$, $\indic\crl*{h=\Hstar}$ is a measurable function of
  $(x_1,a_1),\ldots,(x_h,a_h)$.}
    and has $\abs*{\Delta_{\Hstar}(o)}\leq{}L+1$ almost surely.
  \end{lemma}

  The following lemma, which is one of the central technical
  components of this proof, gives a bound on regret in terms of the
  expected advantage at the stopping time $\Hstar$. We use the
  stopping time to keep the sum of advantages
  $\Delta_{\Hstar}$ bounded, which facilitates a strong
  change-of-measure argument in the sequel.
  \begin{lemma}[Regret decomposition for stopped advantages]
    \label{lem:advantage_stopping}
    If $r_h\geq{}0$ and $\sum_{h=1}^{H}r_h\in\brk{0,R}$, then for all policies $\pihat$, we have that
\begin{align}
  \label{eq:pitil_final}
  J(\pistar) - J(\pihat)
  \leq{} \En^{\pihat}\brk*{
    \Delta_{\Hstar}(o)
    } + R\cdot{}\bbP^{\pihat}\brk*{\Hstar\leq{}H}.
\end{align}   
\end{lemma}
Note that even though we assume $R=1$ throughout this proof, we state
this lemma for general $R$ for the sake of keeping it self-contained.

We proceed to bound the right-hand-side of \cref{eq:pitil_final} using
change-of-measure based on Hellinger distance (\cref{lem:hellinger_com}). For
the second term in \cref{eq:pitil_final}, \cref{lem:hellinger_com} gives
\begin{align}
  \label{eq:stoch0}
    \bbP^{\pihat}\brk*{\Hstar\leq{}H}
    &\leq{} 2 \bbP^{\pistar}\brk*{\Hstar\leq{}H} +
      \Dhels{\bbP^{\pihat}}{\bbP^{\pistar}}\\
    &= 2\bbP^{\pistar}\brk*{\exists{}h: \abs*{\Delta_h(o)}>L} + \Dhels{\bbP^{\pihat}}{\bbP^{\pistar}}.
  \end{align}
  For the first term in \cref{eq:pitil_final}, \cref{lem:hellinger_com}, gives that
  \begin{align}
    \En^{\pihat}\brk*{\Delta_{\Hstar}(o)}
    &\leq{} \En^{\pistar}\brk*{\Delta_{\Hstar}(o)}
    + \sqrt{\tfrac{1}{2}\prn*{\En^{\pihat}\brk*{\Delta_{\Hstar}^2(o)} +
      \En^{\pistar}\brk*{\Delta_{\Hstar}^2(o)}}\cdot\Dhels{\bbP^{\pihat}}{\bbP^{\pistar}}}.
  \end{align}
  To bound the first moment and second moment of $\Delta_{\Hstar}(o)$
  under $\pistar$, we use the following lemma, which follows from
  elementary properties of stopped martingale difference sequences.
    \begin{lemma}
    \label{lem:optional_stopping}
    We have that
    \begin{align}
      \En^{\pistar}\brk*{\Delta_{\Hstar}(o)}\leq{}0,\mathand
      \En^{\pistar}\brk*{\Delta^2_{\Hstar}(o)}\leq{}4\sigmastar^2.
    \end{align}
  \end{lemma}
  It remains to bound the second moment under $\pihat$. Here, since $\abs*{\Delta_{\Hstar}(o)}\leq{}L+1$ almost surely by
  \cref{lem:stopping_time}, we note that \cref{lem:hellinger_com} gives
  \begin{align}
    \En^{\pihat}\brk*{\Delta_{\Hstar}^2(o)}
    \leq{} 2\En^{\pistar}\brk*{\Delta_{\Hstar}^2(o)}
    + (L+1)^2\Dhels{\bbP^{\pihat}}{\bbP^{\pistar}}.
  \end{align}
  Combining these developments, we have that
  \begin{align}
    \En^{\pihat}\brk*{\Delta_{\Hstar}(o)}
    &\leq{} \sqrt{\tfrac{3}{2}
      \En^{\pistar}\brk*{\Delta_{\Hstar}^2(o)}\cdot\Dhels{\bbP^{\pihat}}{\bbP^{\pistar}}}
      + (L+1)\Dhels{\bbP^{\pihat}}{\bbP^{\pistar}}\\
        &\leq{} \sqrt{6\sigmastar^2
          \cdot\Dhels{\bbP^{\pihat}}{\bbP^{\pistar}}}
      + (L+1)\Dhels{\bbP^{\pihat}}{\bbP^{\pistar}},
  \end{align}
  and thus
  \begin{align}
    \label{eq:penultimate}
    J(\pistar) - J(\pihat)
    \leq{} \sqrt{6\sigmastar^2
          \cdot\Dhels{\bbP^{\pihat}}{\bbP^{\pistar}}}
    + (L+2)\Dhels{\bbP^{\pihat}}{\bbP^{\pistar}}
    + 2\bbP^{\pistar}\brk*{\exists{}h: \abs*{\Delta_h(o)}>L}.
  \end{align}
  To wrap up, we appeal to the second of our main technical lemmas, \cref{lem:advantage_concentration}.
  \iftoggle{neurips}{
    \begin{restatable}[Concentration for advantages]{lemma}{advantageconc}
      \label{lem:advantage_concentration}
    Assume that $r_h\geq{}0$ and $\sum_{h=1}^{H}r_h\in\brk{0,R}$ almost
    surely for some $R>0$. Then for any (potentially stochastic) policy $\pi$, it holds that for all
    $\delta\in(0,e^{-1})$,
    \begin{align}
      \bbP^{\pi}\brk*{\exists{}H' : \abs*{\sum_{h=1}^{H'}Q^{\pi}_h(x_h,a_h)-V^{\pi}_h(x_h)}\geq{}c\cdot{}R\log(\delta^{-1})}
      \leq \delta,
    \end{align}
    for an absolute constant $c>0$.
  \end{restatable}
  }{}
  Let
  $\veps\in(0,e^{-1})$ be fixed. If we define
  \[
    L = c\cdot{}\log(\veps^{-1}),
  \]
  where $c>1$ is a sufficiently large absolute constant, then by
  \cref{lem:advantage_concentration}, we have that
  \begin{align}
    \bbP^{\pistar}\brk*{\exists{}h: \abs*{\Delta_h(o)}>L}\leq\veps.
  \end{align}
  This proves the result.

\end{proof}

  \begin{proof}[\pfref{lem:stopping_time}]
To prove that $\Hstar$ is a stopping time, we observe that for all $h\leq{}H$, we have
  \begin{align}
    \indic\crl*{h=\Hstar}=
    \indic\crl*{\abs*{\Delta_h(o)}>L, \abs*{\Delta_{h'}(o)}\leq{}L\;\forall{}h'<h},
  \end{align}
  and $\Delta_h(o)$ is a measurable function of
  $(x_1,a_1),\ldots,(x_h,a_h)$. Likewise, we have
    \begin{align}
      \indic\crl*{h=\Hstar+1}=
    \indic\crl*{\abs*{\Delta_{h}(o)}\leq{}L\;\forall{}h\leq{}H},
    \end{align}
    which is a measurable function of $(x_1,a_1),\ldots,(x_H,a_H)$.

For the second claim, we observe that
  \begin{align}
    \label{eq:delta_bound}
    \abs*{\Delta_{\Hstar}(o)}
    &\leq{}     \abs*{\Delta_{\Hstar-1}(o)}
    + \abs*{\Qstar_{\Hstar}(x_{\Hstar},\pistar_{\Hstar}(x_{\Hstar}))
    - \Qstar_{\Hstar}(x_{\Hstar},a_{\Hstar})}\\
    &\leq{} L+1
  \end{align}
  almost surely. 
    
  \end{proof}

  \begin{proof}[\pfref{lem:optional_stopping}] Define $X_h\ldef{}\Qstar_h(x_h,\pistar_h(x_h))-\Qstar_h(x_h,a_h)$, and
    $\filt_h=\sigma(x_1,a_1,\ldots,x_h,a_h)$, with $X_{H+1}\ldef{}0$. Since $\Hstar$ is a
    stopping time with respect to $(\filt_h)$ and $X_h$ is a
    martingale difference sequence (under $\pistar$), the optional
    stopping theorem (e.g., \cite{williams1991probability}) implies
    that\footnote{
          To give self-contained proof, note that we can write
          $\En^{\pistar}\brk*{\sum_{h=1}^{\Hstar}X_h}
      = \En^{\pistar}\brk*{\sum_{h=1}^{H}X_h\indic\crl*{\Hstar\geq{}h}}$
    We claim that $\indic\crl*{\Hstar\geq{}h}$ is a measurable
    function of $\filt_{h-1}$, since
    $\indic\crl*{\Hstar\geq{}h}=1-\indic\crl*{\Hstar<h}$, and
    $\indic\crl*{\Hstar=h'}$ is a measurable function of
    $(x_1,a_1),\ldots,(x_{h'},a_{h'})\subset
    (x_1,a_1),\ldots,(x_{h-1},a_{h-1})$ for $h'<h$. We conclude that
    $\En^{\pistar}\brk*{X_h\indic\crl*{\Hstar\geq{}h}\mid{}\filt_{h-1}}
      =       \En^{\pistar}\brk*{X_h\mid{}\filt_{h-1}}\indic\crl*{\Hstar\geq{}h}=0$.
    }
    \begin{align}
      \En^{\pistar}\brk*{\Delta_{\Hstar}(o)}
      = \En^{\pistar}\brk*{\sum_{h=1}^{\Hstar}X_h} = 0.
    \end{align}

    We now bound the second moment. Recall Doob's maximal inequality
    (e.g., \citet{williams1991probability}).
    \begin{lemma}
        \label{lem:maximal}
      If $(S_{h})_{h\in\brk{H}}$ is a non-negative submartingale, then
      \begin{align}
        \En\brk*{\max_{h\in\brk{H}}S_h^2} \leq{} 4\En\brk*{S_H^2}.
      \end{align}
    \end{lemma}
    We claim that $\abs*{\Delta_h(o)}$ is a submartingale, since a
    convex function of a martingale is a submartingale.\footnote{For
      completeness, note that
      $\En\brk*{\abs*{\Delta_h(o)}\mid{}\filt_{h-1}}
      =       \En\brk*{\abs*{\Delta_{h-1}(o)+X_h}\mid{}\filt_{h-1}}
      \geq{}       \abs*{\Delta_{h-1}(o)+\En\brk*{ X_h \mid{}\filt_{h-1}}}=\abs*{\Delta_{h-1}(o)}$.}
    As a result, \cref{lem:maximal} gives that
    \begin{align}
      \En\brk*{\Delta_{\Hstar}^2(o)}
      \leq{}      \En\brk*{\max_{h\in\brk{H}}\Delta_{h}^2(o)}
      \leq{}   4\En\brk*{\Delta_{H}^2(o)}.
    \end{align}
        Finally, we note that
    \begin{align}
      \En^{\pistar}\brk*{\Delta_H^2(o)}
      &= \En^{\pistar}\brk*{\prn*{\sum_{h=1}^{H}\prn*{\Qstar_h(x_h,\pistar_h(x_h))
      - \Qstar_h(x_h,a_h)}}^2}\\
      &=
        \sum_{h=1}^{H}\En^{\pistar}\brk*{(\Qstar_h(x_h,\pistar_h(x_h))-\Qstar_h(x_h,a_h))^2}
        = \sigmastar^2,
    \end{align}
where we have once more used that
    $X_h=\Qstar_h(x_h,\pistar_h(x_h))-\Qstar_h(x_h,a_h)$ is a
    martingale difference sequence.
    
  \end{proof}

\subsubsection{Proof of \creftitle{lem:advantage_stopping} (Regret
  Decomposition for Stopped Advantages)}

  \begin{proof}[\pfref{lem:advantage_stopping}]%
    Consider the following non-Markovian policy:
  \begin{align}
    \pitil_h(\cdot\mid{}x_{1:h},a_{1:h-1}) =
    \left\{
    \begin{array}{ll}
      \pihat_h(\cdot\mid{}x_h) & h \leq \Hstar,\\
      \pistar_h(\cdot\mid{}x_h) & h > \Hstar.
    \end{array}
    \right.
  \end{align}
This is a well-defined policy, since we can write
  $\indic\crl*{h>\Hstar}=\max_{h'<h}\indic\crl*{h'=\Hstar}$, and
  $\indic\crl*{h'=\Hstar}$ is a measurable function of
  $(x_1,a_1),\ldots,(x_{h'},a_{h'})\subset
  (x_1,a_1),\ldots,(x_{h-1},a_{h-1})$ for $h'<h$.

  We begin by writing
  \begin{align}
    \label{eq:pitil1}
    J(\pistar) - J(\pihat)
    =     J(\pistar) - J(\pitil)
    + J(\pitil) - J(\pihat).
  \end{align}
  For the second pair of terms in \cref{eq:pitil1}, we use the following lemma.
  \begin{lemma}
    \label{lem:pitil_regret}
    Under the same assumptions as \cref{lem:advantage_stopping}, it holds that
    \begin{align}
      J(\pitil) - J(\pihat)
      \leq{} R\cdot{}\bbP^{\pihat}\brk*{\Hstar\leq{}H}.
    \end{align}
    
  \end{lemma}
For the first pair of terms in \cref{eq:pitil1}, using the performance difference lemma,
we can write\footnote{Since $\pitil$ is non-Markovian, we need to expand the state space to $x'_h=x_{1:h},a_{1:h-1}$ to
  apply the performance difference lemma, but since $\pistar$ itself
  is Markovian, this results in the claimed expression.}
  \begin{align}
    J(\pistar) - J(\pitil)
    &= \En^{\pitil}\brk*{
    \sum_{h=1}^{H}\Qstar_h(x_h,\pistar_h(x_h))
    - \Qstar_h(x_h,a_h)
      }\\
    &= \En^{\pitil}\brk*{
    \sum_{h=1}^{H}\En_{h-1}\brk*{\Qstar_h(x_h,\pistar_h(x_h))
    - \Qstar_h(x_h,a_h)}
      }\\
    &= \En^{\pitil}\brk*{
    \sum_{h=1}^{H}\En_{h-1}\brk*{\Qstar_h(x_h,\pistar_h(x_h))
    - \Qstar_h(x_h,a_h)}\indic\crl*{h\leq{}\Hstar}
      }\\
    &= \En^{\pitil}\brk*{
      \sum_{h=1}^{H}\En_{h-1}\brk*{\prn*{\Qstar_h(x_h,\pistar_h(x_h))
    - \Qstar_h(x_h,a_h)}\indic\crl*{h\leq{}\Hstar}}
      }\\
    &= \En^{\pitil}\brk*{
    \sum_{h=1}^{\Hstar}\Qstar_h(x_h,\pistar_h(x_h))
    - \Qstar_h(x_h,a_h)
    } = \En^{\pitil}\brk*{
    \Delta_{\Hstar}(o)
    },
  \end{align}
  where the third equality uses that $\pitil_h(\cdot\mid{}x_{1:h},a_{1:h-1})=\pistar_h(\cdot\mid{}x_h)$
  for $h>\Hstar$, and the fourth equality uses that
  $\indic\crl*{h\leq{}\Hstar}$ is $\filt_{h-1}$-measurable. We now
  appeal to the following lemma, proven in the sequel.
  \begin{lemma}
    \label{lem:pitil_advantage}
        Under the same assumptions as \cref{lem:advantage_stopping}, it holds that
\begin{align}
  \En^{\pitil}\brk*{
    \Delta_{\Hstar}(o)
    } = \En^{\pihat}\brk*{
    \Delta_{\Hstar}(o)
    }.
\end{align}
\end{lemma}
Altogether, we conclude that
\begin{align}
    J(\pistar) - J(\pihat)
  \leq{} \En^{\pihat}\brk*{
    \Delta_{\Hstar}(o)
    } + R\cdot{}\bbP^{\pihat}\brk*{\Hstar\leq{}H}.
\end{align}
    
\end{proof}

    \begin{proof}[\pfref{lem:pitil_regret}]
    Let us define $f(o)=\sum_{h=1}^{H}\En\brk*{r_h\mid{}x_h,a_h}$
    and $g(o)=\indic\crl*{\Hstar>H}$; note that $g(o)$ is indeed
    a measurable function of $o=(x_1,a_1),\ldots,(x_H,a_H)$, since
    $\indic\crl*{\Hstar>H}=1-\indic\crl*{\Hstar\leq{}H}$,
    $\crl*{\Hstar\leq{}H}=\cup_{h\leq{}H}\crl*{\Hstar=h}$, and
    $\crl*{\Hstar=h}$ is a measurable function of
    $(x_1,a_1),\ldots,(x_h,a_h)$.  We can write
    \begin{align}
      \label{eq:pitil2}
      J(\pitil)
      \leq{}
      \En^{\pitil}\brk*{\prn*{\sum_{h=1}^{H}r_h}\indic\crl*{\Hstar>H}}
      + R\cdot{}\bbP^{\pitil}\brk*{\Hstar\leq{}H}.
    \end{align}
    Let us adopt the shorthand $P(x_{1:H}\mid{}a_{1:H-1})\ldef{}\prod_{h=0}^{H-1}P_h(x_{h+1}\mid{}x_h,a_h)$.
We can bound the first term in \cref{eq:pitil2} via
    \begin{align}
      \En^{\pitil}\brk*{\prn*{\sum_{h=1}^{H}r_h}\indic\crl*{\Hstar>H}}
      &= \sum_{o=x_{1:H},a_{1:H}}f(o)g(o)P(x_{1:H}\mid{}a_{1:H-1})
        \prod_{h=1}^{H}\pitil_h(a_h\mid{}x_{1:h},a_{1:h-1})\\
      &= \sum_{o=x_{1:H},a_{1:H}}f(o)g(o)P(x_{1:H}\mid{}a_{1:H-1})
        \prod_{h=1}^{H}\pihat_h(a_h\mid{}x_{h})\\
      &\leq \sum_{o=x_{1:H},a_{1:H}}f(o)P(x_{1:H}\mid{}a_{1:H-1})
        \prod_{h=1}^{H}\pihat_h(a_h\mid{}x_{h})\\
      &= \En^{\pihat}\brk*{\sum_{h=1}^{H}r_h} = J(\pihat),
    \end{align}
    where the second equality uses that
    $\pitil(\cdot\mid{}x_{1:h},a_{1:h-1})=\pihat(\cdot\mid{}x_h)$ for
    all $h\in\brk{H}$
    whenever $g(o)=1$.

    To bound the second term in \cref{eq:pitil2}, we can write
    \begin{align}
      \bbP^{\pitil}\brk*{\Hstar\leq{}H}
      =\sum_{h=1}^{H}\bbP^{\pitil}\brk*{\Hstar=h}.
    \end{align}
    For each $h$, let $o_h\ldef(x_1,a_1),\ldots,(x_h,a_h)$ and
    $g_h(o_h)\ldef{}\indic\crl*{\Hstar=h}$ (recall that $\indic\crl*{\Hstar=h}$ is a measurable function
    of $(x_1,a_1),\ldots,(x_h,a_h)$).  Note that for each $h$, if we
    define
    $P(x_{1:h}\mid{}a_{1:h-1})\ldef{}\prod_{h=0}^{h-1}P_\ell(x_{\ell+1}\mid{}x_\ell,a_{\ell})$, then
    \begin{align}
      \bbP^{\pitil}\brk*{\Hstar=h}
      &= \sum_{o_h=x_{1:h},a_{1:h}}g_h(o_h)P(x_{1:h}\mid{}a_{1:h-1})
        \prod_{\ell=1}^{h}\pitil_\ell(a_\ell\mid{}x_{1:\ell},a_{1:\ell-1})\\
      &= \sum_{o_h=x_{1:h},a_{1:h}}g_h(o_h)P(x_{1:h}\mid{}a_{1:h-1})
        \prod_{\ell=1}^{h}\pihat_\ell(a_\ell\mid{}x_{\ell})\\
      &= \bbP^{\pihat}\brk*{\Hstar=h},
    \end{align}
    where the second inequality uses that
    $\pitil(\cdot\mid{}x_{1:\ell},a_{1:\ell-1})=\pihat(\cdot\mid{}x_\ell)$
    whenever $\ell\leq{}\Hstar$.
  \end{proof}

  \begin{proof}[\pfref{lem:pitil_advantage}]
We start by writing
  \begin{align}
    \En^{\pitil}\brk*{
    \Delta_{\Hstar}(o)
    }
    &= \sum_{h=1}^{H+1}\En^{\pitil}\brk*{\indic\crl*{\Hstar=h}\Delta_h(o)}.
  \end{align}
    For each $h\leq{}H+1$, let $o_h\ldef(x_1,a_1),\ldots,(x_h,a_h)$ and
    $g_h(o_h)\ldef{}\indic\crl*{\Hstar=h}$ (recall that $\indic\crl*{\Hstar=h}$ is a measurable function
    of $(x_1,a_1),\ldots,(x_h,a_h)$).  For each $h\leq{}H+1$, if we
    define
    $P(x_{1:h}\mid{}a_{1:h-1})\ldef{}\prod_{h=0}^{h-1}P_\ell(x_{\ell+1}\mid{}x_\ell,a_\ell)$, then
  \begin{align}
    \En^{\pitil}\brk*{\indic\crl*{\Hstar=h}\Delta_h(o)}
      &= \sum_{o_h=x_{1:h},a_{1:h}}g_h(o_h)\Delta_h(o_h)P(x_{1:h}\mid{}a_{1:h-1})
        \prod_{\ell=1}^{h}\pitil_\ell(a_\ell\mid{}x_{1:\ell},a_{1:\ell-1})\\
      &= \sum_{o_h=x_{1:h},a_{1:h}}g_h(o_h)\Delta_h(o_h)P(x_{1:h}\mid{}a_{1:h-1})
        \prod_{\ell=1}^{h}\pihat_\ell(a_\ell\mid{}x_{\ell})\\
      &=     \En^{\pihat}\brk*{\indic\crl*{\Hstar=h}\Delta_h(o)},
    \end{align}
    where the second inequality uses that
    $\pitil(\cdot\mid{}x_{1:\ell},a_{1:\ell-1})=\pihat(\cdot\mid{}x_\ell)$
    whenever $\ell\leq{}\Hstar$.
  
\end{proof}

\subsubsection{Proof of \creftitle{lem:advantage_concentration} (Concentration for Advantages)}

\cref{lem:advantage_concentration} is proven using arguments similar to those in
\citet{zhang2021reinforcement,zhang2022horizon}, but requires
non-trivial modifications to accommodate the fact that $\pi$ is an
arbitrary, potentially suboptimal policy.

\begin{proof}[\pfref{lem:advantage_concentration}]
  Let us abbreviate $Q=Q^{\pi}$ and $V=V^{\pi}$. Assume without loss of generality that $R=1$, and note that this
  implies that $r_h\in\brk{0,1}$ and $Q_h,V_h\in\brk{0,1}$, which we
  will use throughout the proof.

Define a filtration
$\filt_{h-1}\ldef{}\sigma((x_1,a_1,r_1),\ldots,(x_{h-1},a_{h-1},r_{h-1}),x_h)$. Since
\[\En_{h-1}\brk*{Q_h(x_h,a_h)-V_h(x_h)}=0,\] two applications of
\cref{lem:freedman} and a union bound imply that with probability at
least $1-\delta$, for all $H'\in\brk{H}$
\begin{align}
  \abs*{\sum_{h=1}^{H'}Q_h(x_h,a_h)-V_h(x_h)}
  \leq{}
\sum_{h=1}^{H'}\En^{\pi}\brk*{(Q_h(x_h,a_h)-V_h(x_h))^2\mid{}x_h}
  + \log(2\delta^{-1}).
\end{align}
Since $\En^{\pi}\brk*{Q_h(x_h,a_h)\mid{}x_h}=V_h(x_h)$, we can write
\begin{align}
  \sum_{h=1}^{H'}\En^{\pi}\brk*{(Q_h(x_h,a_h)-V_h(x_h))^2\mid{}x_h}
  &=\sum_{h=1}^{H'}\En^{\pi}\brk*{(Q_h^2(x_h,a_h)\mid{}x_h} -
    V_h^2(x_h)\\
  &=\sum_{h=1}^{H'}\prn*{\En^{\pi}\brk*{(Q_h^2(x_h,a_h)\mid{}x_h} -
    V_{h+1}^2(x_{h+1})} +V_{H'+1}^2(x_{H'+1})- V_1^2(x_1)\\
  &\leq{}\sum_{h=1}^{H'}\prn*{\En^{\pi}\brk*{(Q_h^2(x_h,a_h)\mid{}x_h} -
    V_{h+1}^2(x_{h+1})} + 1.
\end{align}
Observe that by Jensen's
inequality, we have
\begin{align}
  \En^{\pi}\brk*{(Q_h^2(x_h,a_h)\mid{}x_h}
  &\leq{} \En^{\pi}\brk*{(r_h+V_{h+1}(x_{h+1}))^2\mid{}x_h}\\
  &=\En^{\pi}\brk*{V^2_{h+1}(x_{h+1})\mid{}x_h}
    +
    \En^{\pi}\brk*{r_h^2\mid{}x_h}+2\En^{\pi}\brk*{r_hV_{h+1}(x_{h+1})\mid{}x_h}\\
    &\leq{}\En^{\pi}\brk*{V^2_{h+1}(x_{h+1})\mid{}x_h}
      + 3\En^{\pi}\brk*{r_h\mid{}x_h},
\end{align}
so that
\begin{align}
  \sum_{h=1}^{H'}\En^{\pi}\brk*{(Q_h(x_h,a_h)-V_h(x_h))^2\mid{}x_h}
  \leq{} \sum_{h=1}^{H'}\En^{\pi}\brk*{V^2_{h+1}(x_{h+1})\mid{}x_h}-
  V^2_{h+1}(x_{h+1})
  + 3\sum_{h=1}^{H'}\En^{\pi}\brk*{r_h\mid{}x_h} + 1.
\end{align}
By \cref{lem:multiplicative_freedman}, we have that with probability
at least $1-\delta$, for all $H'\in\brk{H}$,
\begin{align}
  \sum_{h=1}^{H'}\En^{\pi}\brk*{r_h\mid{}x_h}
  \leq{} \frac{3}{2}\sum_{h=1}^{H'}r_h + 4\log(2\delta^{-1})\\
  \leq{} \frac{3}{2} + 4\log(2\delta^{-1}).
\end{align}
Likewise, by \cref{lem:freedman}, we have that with probability at least $1-\delta$, for all $H'\in\brk{H}$,
\begin{align}
  \sum_{h=1}^{H'}\En^{\pi}\brk*{V^2_{h+1}(x_{h+1})\mid{}x_h}-
  V^2_{h+1}(x_{h+1})
  &\leq{} 
  \sum_{h=1}^{H'}\En^{\pi}\brk*{\prn*{V_{h+1}^2(x_{h+1}) -
  \En^{\pi}\brk*{V^2_{h+1}(x_{h+1})\mid{}x_h}}^2\mid{}x_h}
    +\log(\delta^{-1})\\
    &= 
  \sum_{h=1}^{H'}\Var^{\pi}\brk*{V_{h+1}^2(x_{h+1})\mid{}x_h}
      +\log(\delta^{-1})\\
  &\leq{} 4
  \sum_{h=1}^{H'}\Var^{\pi}\brk*{V_{h+1}(x_{h+1})\mid{}x_h}
  +\log(\delta^{-1}),
\end{align}
where the last line uses the
following lemma, proven in the sequel.
    \begin{lemma}
      \label{lem:variance_square}
      If $X$ is a random variable with $\abs*{X}\leq{}1$, then
      \begin{align}
        \Var(X^2)\leq{}4\Var(X).
      \end{align}
    \end{lemma}
We now appeal to the following lemma, also proven in the sequel.
\begin{lemma}
  \label{lem:variance_conc}
  Under the same setting as \cref{lem:advantage_concentration}, we have
  that for any $\delta\in(0,1)$, with probability at
  least $1-2\delta$, for all $H'\in\brk{H}$,
  \begin{align}
    \sum_{h=1}^{H'}\Var^{\pi}\brk*{V^{\pi}_{h+1}(x_{h+1})\mid{}x_h}
    \leq{} 8 + 32\log(2\delta^{-1}).
  \end{align}
\end{lemma}
Putting together all of the developments so far, we have that with probability at least $1-5\delta$,
for all $H'\in\brk{H}$,
\begin{align}
  \abs*{\sum_{h=1}^{H'}Q_h(x_h,a_h)-V_h(x_h)}
  &\leq{}  4   \sum_{h=1}^{H'}\Var^{\pi}\brk*{V_{h+1}(x_{h+1})\mid{}x_h}
  +6 +
    14\log(2\delta^{-1})\\
  &\leq{}  38 +
  142\log(2\delta^{-1}).
\end{align}

\end{proof}

    \begin{proof}[\pfref{lem:variance_square}]
      Note that we have
      \begin{align}
        \Var(X^2)
        = \En\brk*{(X^2-\En\brk*{X^2})^2}
        \leq{}     \En\brk*{(X^2-\En\brk*{X}^2)^2}
        \leq{} 4\En\brk*{(X-\En\brk*{X})^2},
      \end{align}
      where the last line uses that $\abs*{a^2-b^2}\leq{}2\abs*{a-b}$
      for $a,b\in\brk{-1,1}$.
    \end{proof}

    \begin{proof}[\pfref{lem:variance_conc}]
      Abbreviate $V\equiv{}V^{\pi}$. By telescoping, we can write
  \begin{align}
    Z_{H'} \ldef{} &\sum_{h=1}^{H'}\Var^{\pi}\brk*{V_{h+1}(x_{h+1})\mid{}x_h}\\
              &= \sum_{h=1}^{H'}\En^{\pi}\brk*{V^2_{h+1}(x_{h+1})\mid{}x_h}
                  -
                \prn*{\En^{\pi}\brk*{V_{h+1}(x_{h+1})\mid{}x_h}}^2\\
    &= \sum_{h=1}^{H'}\En^{\pi}\brk*{V^2_{h+1}(x_{h+1})\mid{}x_h}
      - V_{h+1}^2(x_{h+1})
      + \sum_{h=1}^{H'}V_h^2(x_{h})
                  - \prn*{\En^{\pi}\brk*{V_{h+1}(x_{h+1})\mid{}x_h}}^2
      +V_{H'+1}^2(x_{H'+1})- V^2_{1}(x_1)\\
              &\leq \sum_{h=1}^{H'}\En^{\pi}\brk*{V^2_{h+1}(x_{h+1})\mid{}x_h}
      - V_{h+1}^2(x_{h+1})
      + \sum_{h=1}^{H'}V_h^2(x_{h})
                  - \prn*{\En^{\pi}\brk*{V_{h+1}(x_{h+1})\mid{}x_h}}^2
                + 1.
  \end{align}
  For the latter term, since $\abs*{a^{2}-b^{2}}\leq{}2\abs*{a-b}$ for
  $a,b\in\brk{0,1}$, we have that
  \begin{align}
    \sum_{h=1}^{H'}V_h^2(x_{h})
    - \prn*{\En^{\pi}\brk*{V_{h+1}(x_{h+1})\mid{}x_h}}^2
    &\leq{}    2\sum_{h=1}^{H'}\abs*{V_h(x_{h})
      - \En^{\pi}\brk*{V_{h+1}(x_{h+1})\mid{}x_h}}\\
    &=    2\sum_{h=1}^{H'}\abs*{\En^{\pi}\brk*{r_h\mid{}x_h}}
    \leq{}2\sum_{h=1}^{H'}\En^{\pi}\brk*{r_h\mid{}x_h},
  \end{align}
  By \cref{lem:multiplicative_freedman}, we have that with probability
at least $1-\delta$, for all $H'\in\brk{H}$,
\begin{align}
  \sum_{h=1}^{H'}\En^{\pi}\brk*{r_h\mid{}x_h}
  \leq{} \frac{3}{2}\sum_{h=1}^{H'}r_h + 4\log(2\delta^{-1})
  \leq{} \frac{3}{2} + 4\log(2\delta^{-1}).
\end{align}
For the first term, by \cref{lem:freedman}, we have that for all
$\eta\in(0,1)$, with probability at least $1-\delta$, for all $H'\in\brk{H}$,
\begin{align}
  \sum_{h=1}^{H'}\En^{\pi}\brk*{V^2_{h+1}(x_{h+1})\mid{}x_h}-
  V_{h+1}(x_{h+1})
  &\leq{} \eta
  \sum_{h=1}^{H'}\En^{\pi}\brk*{\prn*{V_{h+1}^2(x_{h+1}) -
  \En^{\pi}\brk*{V^2_{h+1}(x_{h+1})\mid{}x_h}}^2\mid{}x_h}
    +\eta^{-1}\log(\delta^{-1})\\
    &= \eta
  \sum_{h=1}^{H'}\Var^{\pi}\brk*{V_{h+1}^2(x_{h+1})\mid{}x_h}
      +\eta^{-1}\log(\delta^{-1})\\
  &\leq{} 4\eta
  \sum_{h=1}^{H'}\Var^{\pi}\brk*{V_{h+1}(x_{h+1})\mid{}x_h}
    +\eta^{-1}\log(\delta^{-1})\\
  &= 4\eta{}Z_{H'}
  +\eta^{-1}\log(\delta^{-1}),
\end{align}
where the last inequality uses \cref{lem:variance_square}. Putting everything together and setting $\eta=1/8$, we conclude that with probability at
least $1-2\delta$, for all $H'\in\brk{H}$
\begin{align}
  Z_{H'}\leq{} \frac{1}{2}Z_{H'} + 16\log(2\delta^{-1}) + 4,
\end{align}
which yields the result after rearranging.
  
\end{proof}

\arxiv{\subsection{Proof of \creftitle{thm:variance_lower}}}
\label{sec:stoch_lower_proof}

\begin{proof}[\pfref{thm:variance_lower}]
          \newcommand{\pia}{\pi\ind{\afrak}}%
        \newcommand{\pib}{\pi\ind{\bfrak}}%
        \newcommand{\ra}{r\ind{\afrak}}%
        \newcommand{\rb}{r\ind{\bfrak}}%
              For this proof, we consider a slightly more general online imitation
      learning model in which the learner is
        allowed to select $a_h\ind{i}$ based on the sequence
        $(x\ind{i}_1,a\ind{i}_1,\astari_1),\ldots,(x\ind{i}_{h-1},a\ind{i}_{h-1},\astari_{h-1}),
        (x_h\ind{i}, \astari_h)$ at training time; this subsumes the offline imitation learning model.
        Let $H\in\bbN$, $n\in\bbN$, and $\sigma^2\in\brk{H,H^2}$ be
        given. Fix a parameter $K\in\bbN$
        such that $H/K$ is an integer and a parameter
        $\Delta\in(0,1/2)$ be fixed; both parameters will be chosen at
        the end of the proof.

  We first specify the dynamics for the reward-free MDP $\Mstar$ and
  the policy class $\Pi$. Let $\cA=\crl{\afrak,\bfrak}$, and let
        $\cX=\crl*{\sfrak,\afrak,\bfrak}$. We consider the following
        (deterministic) dynamics. For $h\in\cH\ldef{}\brk*{1, K+1, 2K+1,\ldots}$, 
        always the state is always $x_h=\sfrak$. For such a step
        $h\in\cH$, choosing $a_h=\afrak$ sets
        $x_h=\afrak$ for the next $K-1$ steps until returning to
        $\sfrak$ at time $h+K$, and choosing $a_h=\bfrak$ sets
        $x_h=\bfrak$ until returning to $\sfrak$ at time $h+K$ (that is, the action has no
        effect for $h\notin\cH$).

        We consider a class $\Pi=\crl{\pia,\pib}$ consisting of two experts $\pia$ and $\pib$. $\pia$ sets
        $\pia_h(\afrak\mid{}\sfrak)=\frac{1}{2}+\Delta$ for $h\in\cH$
        and sets $\pi_h(x)=\afrak$ for all $h\notin\cH$ and $x\in\cX$. Meanwhile, $\pib$ sets
        $\pib(\bfrak\mid{}\sfrak)=\frac{1}{2}+\Delta$ for $h\in\cH$
        and sets $\pi_h(x)=\afrak$ for all $h\notin\cH$ and $x\in\cX$.

        We consider two choices of reward function, $\ra$ and
        $\rb$. $\ra$ sets $\ra_h(\sfrak,\afrak)=1$ and
        $\ra_h(\sfrak,\bfrak)=0$ for $h\in\cH$, and sets
        $\ra_h(\afrak,\cdot)=1$ and $\ra_h(\bfrak,\cdot)=0$ for
        $h\notin\cH$. Meanwhile, $\rb$ sets $\rb_h(\sfrak,\bfrak)=1$
        and $\rb_h(\sfrak,\afrak)=0$ for $h\in\cH$ and sets
        $\rb_h(\afrak,\cdot)=0$ and $\rb_h(\bfrak,\cdot)=1$ for $h\notin\cH$.

Let a \emph{problem instance} $\cI=(\Mstar,r,\pistar)$ refer to a tuple
consisting of the reward-free MDP $\Mstar$, a reward function
$r=\crl*{r_h}_{h=1}^{H}$, and an expert policy $\pistar$. We consider
four problem instances altogether:
$(\Mstar, \ra,\pia)$,  $(\Mstar, \rb,\pia)$, $(\Mstar, \ra,\pib)$, and $(\Mstar, \rb,\pib)$.

Let
  $\bbP\ind{\afrak}$ denote the law of $o\ind{1},\ldots,o\ind{n}$ when
  $\afrak$ when we execute the algorithm on the underlying instance,
  and likewise for $\bfrak$ (recall that the law does not depend on
  the choice of reward function, since this is not observed); let $\En\ind{\afrak}\brk*{\cdot}$ and
  $\En^{\bfrak}\brk*{\cdot}$ denote the corresponding expectations. In addition, for any policy $\pi$, let
  $\bbP^{\pia\mid{}\pi}$ denote the law of
  $o=(x_1,a_1,\astar_1),\ldots,(x_H,a_H,\astar_H)$ when we execute
  $\pi$ in the online imitation learning framework and the expert
  policy is $\pistar=\pia$, and define $\bbP^{\pib\mid{}\pi}$ analogously.

  We begin by lower bounding the regret. Consider a fixed policy
  $\pihat=\crl*{\pihat_h:\cX\to\Delta(\cX)}$, and let
        $\pibar(a)\ldef{}\frac{1}{\abs*{\cH}}\sum_{h\in\cH}\pihat_h(a\mid{}\sfrak)$. Observe
        that for instance $(\Mstar,\ra,\pia)$, 
        we have
        \begin{align}
          J_{\ra}(\pia) - J_{\ra}(\pihat) =
          \prn*{\frac{1}{2}+\Delta}H-K\sum_{h\in\cH}\pihat_h(\afrak\mid{}\sfrak)
          =\prn*{\frac{1}{2}+\Delta}H-H\pibar(a)
        \end{align}
        and for instance $(\Mstar,\rb,\pia)$, 
        \begin{align}
          J_{\rb}(\pia) - J_{\rb}(\pihat) 
          =\prn*{\frac{1}{2}-\Delta}H - K\sum_{h\in\cH}\pihat_h(\bfrak\mid{}\sfrak)
          = H\pibar(a)- \prn*{\frac{1}{2}+\Delta}H.
        \end{align}
        Likewise, for instance $(\Mstar,\rb,\pib)$, we have
        \begin{align}
          J_{\rb}(\pib) - J_{\rb}(\pihat) =
          \prn*{\frac{1}{2}+\Delta}H-K\sum_{h\in\cH}\pihat_h(\bfrak\mid\sfrak)
          = \pibar(\afrak)H - \prn*{\frac{1}{2}-\Delta}H
        \end{align}
        and for instance $(\Mstar,\ra,\pib)$, 
        \begin{align}
          J_{\ra}(\pib) - J_{\ra}(\pihat) = \prn*{\frac{1}{2}-\Delta}H
          -K\sum_{h\in\cH}\pihat_h(\afrak\mid\sfrak)
          = \prn*{\frac{1}{2}-\Delta}H - \pibar(\afrak)H.
        \end{align}
        We conclude that for any $\veps>0$, since the law of the
        dataset is independent of the choice of the reward function,
        \begin{align}
          &          \max\crl*{\bbP\ind{\afrak}\brk*{J_{\ra}(\pia) -
          J_{\ra}(\pihat)\geq\veps{}H}, \bbP\ind{\afrak}\brk*{J_{\rb}(\pia) -
          J_{\rb}(\pihat)\geq\veps{}H},
          \bbP\ind{\bfrak}\brk*{J_{\rb}(\pib) -
          J_{\rb}(\pihat)\geq\veps{}H}, \bbP\ind{\bfrak}\brk*{J_{\ra}(\pib) -
          J_{\ra}(\pihat)\geq\veps{}H}
                         }\\
          &\geq{}          \max\left\{
            \begin{aligned}
              &\bbP\ind{\afrak}\brk*{
                \prn*{\frac{1}{2}+\Delta}H-\pibar(\afrak)H
                \geq\veps{}H} ,
              \bbP\ind{\afrak}\brk*{\pibar(\afrak)H-\prn*{\frac{1}{2}+\Delta}H\geq\veps{}H},\\
              &\bbP\ind{\bfrak}\brk*{\pibar(\afrak)H -
                \prn*{\frac{1}{2}-\Delta}H \geq\veps{}H},
              \bbP\ind{\bfrak}\brk*{\prn*{\frac{1}{2}-\Delta}H
                -\pibar(\afrak)H \geq\veps{}H}
            \end{aligned}
            \right\}\\
                    &\geq{}          \frac{1}{2}\max\crl*{\bbP\ind{\afrak}\brk*{          \abs*{\prn*{\frac{1}{2}+\Delta}-\pibar(\afrak)}H \geq\veps{}H},
          \bbP\ind{\bfrak}\brk*{\abs*{\pibar(\afrak) - \prn*{\frac{1}{2}-\Delta}}H
                      \geq\veps{}H}
                      }\\
          &=          \frac{1}{2}\max\crl*{\bbP\ind{\afrak}\brk*{          \abs*{\prn*{\frac{1}{2}+\Delta}-\pibar(\afrak)} \geq\veps{}},
          \bbP\ind{\bfrak}\brk*{\abs*{\pibar(\afrak) - \prn*{\frac{1}{2}-\Delta}}
                      \geq\veps{}}
                      }\\
                              &\geq{}          \frac{1}{4}\prn*{\bbP\ind{\afrak}\brk*{          \abs*{\prn*{\frac{1}{2}+\Delta}-\pibar(\afrak)} \geq\veps}
                                +\bbP\ind{\bfrak}\brk*{\abs*{\pibar(\afrak) - \prn*{\frac{1}{2}-\Delta}}
                      \geq\veps}
                                }\\
          &\geq{}          \frac{1}{4}\prn*{ 1- \bbP\ind{\afrak}\brk*{          \abs*{\prn*{\frac{1}{2}+\Delta}-\pibar(\afrak)} \leq\veps}
                                +\bbP\ind{\bfrak}\brk*{\abs*{\pibar(\afrak) - \prn*{\frac{1}{2}-\Delta}}
                      \geq\veps}
            }\\
          &\geq{}          \frac{1}{4}\prn*{ 1- \bbP\ind{\afrak}\brk*{          \abs*{\prn*{\frac{1}{2}-\Delta}-\pibar(\afrak)} \geq\veps}
                      +\bbP\ind{\bfrak}\brk*{\abs*{\pibar(\afrak) - \prn*{\frac{1}{2}-\Delta}}
                      \geq\veps}
            }\\
                    &\geq{}          \frac{1}{4}\prn*{ 1- \Dtv{\bbP\ind{\afrak}}{\bbP\ind{\bfrak}}
          },
        \end{align}
        where the second inequality uses the union bound
        (i.e. $\bbP\brk*{\abs{x}\geq\veps} = \bbP\brk*{x\geq\veps\cup
          -x\geq\veps}
        \leq{} \bbP\brk{x\geq\veps} + \bbP\brk{-x\geq\veps}$), and the second-to-last inequality holds as long as
        $\veps<\Delta$.
         In particular, this implies that
        \begin{align}
          \max\left\{
            \begin{aligned}
              &\bbP\ind{\afrak}\brk*{J_{\ra}(\pia) -
                J_{\ra}(\pihat)\geq\frac{\Delta{}H}{2}},
              \bbP\ind{\afrak}\brk*{J_{\rb}(\pia) -
                J_{\rb}(\pihat)\geq\frac{\Delta{}H}{2}},\\
              &\bbP\ind{\bfrak}\brk*{J_{\rb}(\pib) -
                J_{\rb}(\pihat)\geq\frac{\Delta{}H}{2}},
              \bbP\ind{\bfrak}\brk*{J_{\ra}(\pib) -
                J_{\ra}(\pihat)\geq\frac{\Delta{}H}{2}}
            \end{aligned}
            \right\}
              \geq{}  \frac{1}{4}\prn*{ 1- \Dtv{\bbP\ind{\afrak}}{\bbP\ind{\bfrak}}}.
        \end{align}

          Next, using Lemma D.2 of \citet{foster2024online}, we can bound
  \begin{align}
    \Dtvs{\bbP\ind{\afrak}}{\bbP\ind{\bfrak}}
    \leq{}\Dhels{\bbP\ind{\afrak}}{\bbP\ind{\bfrak}}
    \leq{}7\En\ind{\afrak}\brk*{\sum_{i=1}^{n}
    \Dhels{\bbP^{\pia\mid{}\pi\ind{i}}}{\bbP^{\pib\mid{}\pi\ind{i}}}
    }.
  \end{align}
  Observe that for a given episode $i$, regardless of how the policy $\pi\ind{i}$ is selected:
  \begin{itemize}
  \item The feedback for steps $h\notin\cH$ is identical under $\bbP^{\afrak}$ and $\bbP^{\bfrak}$.
  \item The feedback at step $h\in\cH$ differs only in the
    distribution of $\astar_h\sim{}\pia(\sfrak)$ versus
    $\astar_h\sim\pib(\sfrak)$. This is equivalently to
    $\Ber(\nicefrac{1}{2}+\Delta)$ feedback versus
    $\Ber(\nicefrac{1}{2}-\Delta)$ feedback.
  \end{itemize}
  As a result, using Lemma D.2 of \citet{foster2024online} once more,
  we have
  \begin{align}
    \Dhels{\bbP^{\pia\mid{}\pi\ind{i}}}{\bbP^{\pib\mid{}\pi\ind{i}}}
    \leq{}7\sum_{h\in\cH}\Dhels{\Ber(\nicefrac{1}{2}+\Delta)}{\Ber(\nicefrac{1}{2}-\Delta)}
  \end{align}
        Since $\Delta\in(0,1/2)$, we have
        $\Dhels{\Ber(\nicefrac{1}{2}+\Delta)}{\Ber(\nicefrac{1}{2}-\Delta)}\leq{}\bigoh(\Delta^2)$
        (e.g., \citet[Lemma A.7]{foster2021statistical}). We conclude that
        \begin{align}
          \Dtvs{\bbP\ind{\afrak}}{\bbP\ind{\bfrak}}
          \leq\bigoh\prn*{n\cdot{}\abs{\cH}\cdot{}\Delta^2}
          = \bigoh\prn*{n\cdot{}\frac{H}{K}\cdot{}\Delta^2}
        \end{align}
        We set $\Delta^2=c\cdot\frac{K}{Hn}$ for $c>0$
        sufficiently small so that $
        \Dtvs{\bbP\ind{\afrak}}{\bbP\ind{\bfrak}}\leq{}1/2$, and
        conclude that on at least one of the four problem instances,
        the algorithm must have
        \begin{align}
          J(\pistar) - J(\pihat) \geq{} \bigom(\Delta{}H) = \bigom\prn*{\sqrt{\frac{HK}{n}}}
        \end{align}
        with probability at least $1/8$.

        Finally, we compute the variance and choose the parameter $K$. Observe that for all of the choices of expert policy
        and reward function described above, we have
        $\Qstar_h(x_h,\pistar(x_h))-\Qstar_h(x_h,a)=0$ for
        $h\notin\cH$, while
        \begin{align}
          \abs*{\Qstar_h(x_h,\pistar(x_h))-\Qstar_h(x_h,a)}\leq{}K
        \end{align}
        for $h\in\cH$, so we can take $\mutil\leq{}K$.
        Consequently, we have
        \begin{align}
\sigmastar^2=          \sum_{h=1}^{H}\En^{\pistar}\brk*{(\Qstar_h(x_h,\pistar(x_h))-\Qstar_h(x_h,a_h))^2}
          &\leq{}
          \sum_{h\in\cH}\En^{\pistar}\brk*{(\Qstar_h(\sfrak,\pistar(\sfrak))-\Qstar_h(\sfrak,a_h))^2}\\
          &\leq{} \frac{H}{K}\cdot{}K^2=HK.
        \end{align}
        We conclude by setting $K=\sigma^2/H$, which is admissible for
        $\sigma^2\in\brk*{H, H^2}$ (up to a loss in absolute
        constants, we can assume that $\sigma^2/H$ is an integer
        without loss of generality).
\end{proof}

\subsection{Additional Proofs}

\begin{proof}[\pfref{prop:sigma_ltv}]
We have
  \begin{align}
    \sigmastar^2 = \sum_{h=1}^{H}\En^{\pistar}\brk*{(\Qstar_h(x_h,a_h)-\Vstar_h(x_h))^2}.
  \end{align}
  Note that
  $\Qstar_h(x_h,a_h)=\En\brk*{r_h+\Vstar_h(x_{h+1})\mid{}x_h,a_h}$. Hence,
  by
  Jensen's inequality we can bound
  \begin{align}
    \En^{\pistar}\brk*{(\Qstar_h(x_h,a_h)-\Vstar_h(x_h))^2}
    &\leq{}\En^{\pistar}\brk*{\En\brk*{(r_h+\Vstar_{h+1}(x_{h+1})-\Vstar_h(x_h))^2\mid{}x_h,a_h}}\\
    &=\En^{\pistar}\brk*{\En^{\pistar}\brk*{(r_h+\Vstar_{h+1}(x_{h+1})-\Vstar_h(x_h))^2\mid{}x_h}}\\
    &=\En^{\pistar}\brk*{\Var^{\pistar}\brk*{r_h+\Vstar_{h+1}(x_{h+1})\mid{}x_h}},
  \end{align}
  so that
  \begin{align}
    \sigmastar^2
    &\leq{}
    \En^{\pistar}\brk*{\sum_{h=1}^{H}\Var^{\pistar}\brk*{r_h+\Vstar_{h+1}(x_{h+1})\mid{}x_h}}\\
    &\leq{}
    \En^{\pistar}\brk*{\sum_{h=0}^{H}\Var^{\pistar}\brk*{r_h+\Vstar_{h+1}(x_{h+1})\mid{}x_h}}
    =\Var^{\pistar}\brk*{\sum_{h=1}^{H}r_h}\leq{}R^2,
  \end{align}
  where the second to last inequality follows from \cref{lem:ltv}.
  
\end{proof}

\arxiv{
\section{Proofs from \creftitle{sec:online}}
\label{sec:proofs_online}
\arxiv{\subsection{Proof of \creftitle{prop:benefits_representation}}}

\begin{proof}[\pfref{prop:benefits_representation}]
Let $N\in\bbN$ be given. We set $\cX=\crl{\xfrak,\yfrak,\zfrak}$,
$\cA=\brk*{N}\cup\crl*{\afrak,\bfrak}$, and $H=2$. We consider a family of
problem instances $\crl*{(M, \pistar, r)}$ indexed by a subset $S\subset\brk{N}$ with $\abs{\cS}=N/2$ and
an action $\astar\in\crl{\afrak,\bfrak}$ as follows. For a given pair $(S,\astar)$:
\begin{itemize}
\item The dynamics are as
  follows. We have $x_1=\xfrak$ deterministically. For simplicity, we
  assume that only actions in $\brk{N}$ are available at step $h=1$. If $a_1\in{}S$, then
  $x_2=\yfrak$, otherwise $x_2=\zfrak$.
\item The reward at step $1$ is $r_1(\cdot,\cdot)=0$, and the reward
  at step $2$ is given by $r_2(\yfrak,\cdot)=1$ and $r_2(\zfrak,a) = \indic\crl*{a=\astar}$.
\item The expert $\pistar$ sets $\pistar(\xfrak)=\unif(S)$,
  $\pistar(\yfrak)=\unif(\crl{\afrak,\bfrak})$, and
  $\pistar(\zfrak)=\astar$.
\end{itemize}
Let us refer to the problem instance above as $\cI_{S,\astar}=\crl*{(M_{S,\astar},\pistar_{S,\astar},r_{S,\astar})}$, and
let $J_{S,\astar}(\pi)$ denote the expected reward under this instance.

\paragraph{Upper bound for online imitation learning}
Consider the algorithm that sets $\pihat\ind{i}_1=\unif(\brk{N})$ for
each $i\in\brk{n}$. If we play
for $n=\log_2(\delta^{-1})$ episodes, we will see $x_2=\zfrak$ in at least
one episode with probability at least $1-\delta$, at which point we
will observe $\astar=\pistar(\zfrak)$, and we can return the policy $\pihat$
that sets 
$\pihat_1(\xfrak)=\unif(\brk{N})$ and $\pihat_2(\cdot)=\astar$; this
policy has zero regret.

Note that if we define $\Pi=\crl*{\pistar_{S,\astar}}_{\abs{S}=N/2,\astar\in\crl{\afrak,\bfrak}}$
as the natural policy class for the family of instances above, then
the algorithm above is equivalent to running \dagger with the online learning algorithm that, at iteration
$i$, sets
\[
\pihat_h\ind{i}=\unif\prn*{\crl*{\pi\in\Pi_h\mid{}\pi_2(\zfrak)=\astari[j]_2\;\forall{}j<i:x_2\ind{j}=\zfrak}},
\]
and choosing the final policy as $\pihat=\pihat\ind{i}$ for any
iteration $i$ after $x_2=\zfrak$ is encountered.

\paragraph{Lower bound for offline imitation learning}
Consider the offline imitation learning setting. When the underlying
instance is $\cI_{S,\astar}$, we observe a dataset $\cD$ consisting of
$n$ trajectories generated by executing
$\pistar_{S,\astar}$ in $M_{S,\astar}$. The trajectories never visit
the state $\zfrak$, so $\astar$ is not identifiable, and we can do no
better than guessing $\astar$ uniformly in this state. Letting
$\En_{S,\astar}$ denote the law of $\cD$ under instance
$\cI_{S,\astar}$, we have
$J_{S,\astar}(\pihat)=\pihat_1(S\mid\xfrak)+\pihat_1(S^{c}\mid{}\xfrak)\pihat_2(\astar\mid{}\zfrak)$. It
follows that for any $S$, since the law of $\cD$ does not depend on $\astar$,
\begin{align}
  \max_{\astar\in\crl{\afrak,\bfrak}}\En_{S,\astar}\brk*{J_{S,\astar}(\pistar_{S,\astar})-J_{S,\astar}(\pihat)}
  &\geq{}
  \En_{S,\afrak}\brk*{1-\pihat_1(S\mid{}\xfrak)-\pihat_1(S^{c}\mid{}\xfrak)/2}\\
  &= \frac{1}{2}  \En_{S,\afrak}\brk*{1-\pihat_1(S\mid{}\xfrak)}.
\end{align}
Note that if $\pihat$ is proper in the sense that
$\pihat_1(\cdot{}\xfrak)=\unif(\Shat)$ for some $\Shat\subset\brk{N}$
with $\abs{\Shat}=N/2$, we have
$1-\pihat_1(S\mid{}\xfrak)=1-\frac{2}{N}\abs{\Shat\cup{}S}$.
We conclude that if
$\En_{S,\astar}\brk*{J_{S,\astar}(\pistar_{S,\astar})-J_{S,\astar}(\pihat)}\leq{}
\frac{1}{8}$, then
$\En_{S,\astar}\brk[\big]{\abs{\Shat\cap{}S}}\geq{}\frac{3}{8}N$. From
here, it follows from standard lower bounds for discrete distribution
estimation (e.g., \citet{canonne2020short}) that any such estimator $\Shat$
requires $n=\bigom(N)$ samples for a worst-case choice of $S$.
\end{proof}

\arxiv{\subsection{Background and Proof for
    \creftitle{prop:benefits_value}}}
\label{sec:value}

Before proving \cref{prop:benefits_value}, we first formally introduce
the value-based feedback model we consider.

\paragraph{Background on value-based feedback}

We can consider two models for imitation learning with value-based feedback, inspired by \citet{ross2014reinforcement,sun2017deeply}.

\begin{itemize}
\item \textbf{Offline setting.}
In the offline setting, we receive $n$ trajectories
$(x_1,a_1),\ldots,(x_H,a_H)$ generated by executing $\pistar$ in $\Mstar$. For
each state in each such trajectory, we observe $\Astar_h(x_h,\cdot)$, where
$\Astar_h(x,a)=\Qstar_h(x,\pistar(x))-\Qstar_h(x,a)$ is the advantage
function for $\pistar$.\footnote{Our results are not sensitive to
  whether the learner observes the advantage function or the value
  function itself; we choose this formulation for concreteness.}\loose
  
\item \textbf{Online setting.} The online setting is as follows. There are $n$ at
episodes. For each episode $i$, we execute a policy $\pihat\ind{i}$, and receive
a ``trajectory'' $o\ind{i} = (x\ind{i}_1,a\ind{i}_1,\astari_1),\ldots,(x\ind{i}_H,a\ind{i}_H,\astari_H)$, where
$a\ind{i}_h\sim\pihat\ind{i}(x\ind{i}_h)$ and
$\astari_h\sim\pistar(x\ind{i}_h)$. In addition, for
each state in the trajectory, we observe $\Astar_h(x_h,\cdot)$. After the $n$ episodes conclude, we
output a final policy $\pihat$ on which performance is evaluated.
\end{itemize}

\begin{proof}[\pfref{prop:benefits_value}]
  We only sketch the proof, as it is quite similar to
  \cref{prop:benefits_representation}. Let $N\in\bbN$ be given. We set
  $\cS=\crl{\xfrak,\yfrak,\zfrak}$, 
  $\cA=\brk*{N}$, and $H=2$. We consider a class of problem instances
  $\crl*{(M,\pistar,r)}$ indexed by sets $S_1,S_2\subset\brk{N}$ with
$\abs{\cS_1}=\abs{S_2}=N/2$ defined as follows.
 For a given pair $(S_1,S_2)$:
\begin{itemize}
\item The dynamics are as
  follows. We have $x_1=\xfrak$ deterministically. If $a_1\in{}S_1$, then
  $x_2=\yfrak$, otherwise $x_2=\zfrak$.
\item The reward function sets $r_1(\xfrak,\cdot)=0$,
  $r_2(\yfrak,\cdot)=1$, and $r_2(\zfrak,a) = \indic\crl*{a\in{}S_2}$.
\item The expert $\pistar$ sets $\pistar(\xfrak)=\unif(S_1)$,
  $\pistar(\zfrak)=\unif(S_2)$, and $\pistar(\yfrak)=\unif(\brk{N})$
\end{itemize}
We refer to the problem instance above as $\cI_{S_1,S_2}=\prn{M_{S_1,S_2},\pistar_{S_1,S_2},r_{S_1,S_2}}$, and
let $J_{S_1,S_2}(\pi)$ denote the expected reward under this instance.

\paragraph{Upper bound for online imitation learning with value-based
  feedback}
Consider an algorithm that sets $\pihat\ind{i}_1=\unif(\brk{N})$ for
each $i\in\brk{n}$. If we play
for $n=\log_2(\delta^{-1})$ episodes, we will see $x_2=\zfrak$ in at least
one episode with probability at least $1-\delta$, at which point we
will observe $\Astar_2(\zfrak,\cdot)$.
We can pick an arbitrary action with $\Astar_2(\zfrak,\cdot)=0$ and
return the policy $\pihat$
that sets $\pihat_1(\xfrak)=\unif(\brk{N})$ and $\pihat_2(\cdot)=a$; this
policy has zero regret.

Note that if we define $\Pi=\crl*{\pistar_{S_1,S_2}}_{\abs{S_1}=\abs{S_2}=N/2}$
as the natural policy class for the family of instances above, then
the algorithm above is equivalent to running \aggrevate with the online learning algorithm that, at iteration
$i$, sets
\[
\pihat_h\ind{i}=\unif\prn[\big]{\crl[\big]{\pi\in\Pi_h\mid{}\pi_2(\zfrak)\in\argmax_{a}\Astar_2(x_2\ind{j},a)\;\forall{}j<i:x_2\ind{j}=\zfrak}},
\]
and choosing the final policy as $\pihat=\pihat\ind{i}$ for any
iteration $i$ after $x_2=\zfrak$ is encountered.

\paragraph{Lower bound for offline imitation learning}
Consider the offline imitation learning setting. When the underlying
instance is $\cI_{S_1,S_1}$, we observe a dataset $\cD$ consisting of
$n$ trajectories generated by executing
$\pistar_{S_1,S_2}$ in $M_{S_1,S_2}$. The trajectories never visit
the state $\zfrak$, so $S_2$ is not identifiable, and we can do no
better than guessing uniformly in this state. Letting
$\En_{S_1,S_2}$ denote the law of $\cD$ under instance
$\cI_{S_1,S_2}$, we have
$J_{S_1,S_2}(\pihat)=\pihat_1(S_1\mid\xfrak)+\pihat_1(S_1^{c}\mid{}\xfrak)\pihat_2(S_2\mid{}\zfrak)$. It
follows that for any $(S_1,S_2)$, since the law of $\cD$ does not depend on $S_2$,
\begin{align}
  \max_{S_2:\abs{S_2}=N/2}\En_{S_1,S_2}\brk*{J_{S_1,S_2}(\pistar_{S_1,S_2})-J_{S_1,S_2}(\pihat)}
  &\geq{}
  \En_{S_1}\brk*{1-\pihat_1(S_1\mid{}\xfrak)-\pihat_1(S_1^{c}\mid{}\xfrak)/2}\\
  &= \frac{1}{2}  \En_{S_1}\brk*{1-\pihat_1(S_1\mid{}\xfrak)},
\end{align}
with the convention that $\En_{S_1}$ denotes the law of $\cD$ for an
arbitrary choice of $\cS_2$. If $\pihat$ is proper in the sense that
$\pihat_1(\cdot{}\xfrak)=\unif(\wh{S_1})$ for some $\wh{S_1}\subset\brk{N}$
with $\abs{\wh{S_1}}=N/2$, we have
$1-\pihat_1(S_1\mid{}\xfrak)=1-\frac{2}{N}\abs{\wh{S_1}\cup{}S_1}$. We conclude that if
$\En_{S_1,S_2}\brk*{J_{S_1,S_2}(\pistar_{S_1,S_2})-J_{S_1,S_2}(\pihat)}\leq{}
\frac{1}{8}$, then
$\En_{S_1,}\brk[\big]{\abs{\wh{S_1}\cap{}S_1}}\geq{}\frac{3}{8}N$. From
here, it follows from standard lower bounds for discrete distribution
estimation (e.g., \citet{canonne2020short}) that any such estimator $\Shat$
requires $n=\bigom(N)$ samples for a worst-case choice of $S$.

\paragraph{Lower bound for online imitation learning without value-based-feedback}
Consider an online imitation learning algorithm that does not receive value-based
feedback. We claim, via an argument similar to the one above, that if
the algorithm that ensures
\[
\En_{S_1,S_2}\brk*{J_{S_1,S_2}(\pistar_{S_1,S_2})-J_{S_1,S_2}(\pihat)} \leq c
\]
on all instances
for a sufficiently small absolute constant $c$, then it can be used to
produce estimators $\wh{S_1},\wh{S_2}\subset\brk{N}$ such that with constant
probability, either 
$\abs[\big]{\wh{S_1}\cap{}S_1}\geq{}\frac{3}{8}N$ or
$\abs[\big]{\wh{S_2}\cap{}S_2}\geq{}\frac{3}{8}N$.
From
here, it should follow from standard arguments that this requires $n=\bigom(N)$ samples for a worst-case choice of $S_1$ and
$S_2$.\loose

\end{proof}

\arxiv{\subsection{Proof of \creftitle{prop:benefits_exploration}}}

\begin{proof}[\pfref{prop:benefits_exploration}]%
      \newcommand{\pia}{\pi\ind{\afrak}}%
    \newcommand{\pib}{\pi\ind{\bfrak}}%
    We consider a slight variant of the construction from
    \cref{prop:lb_deterministic}. Let $n$ and $H$ be given, and let $\Delta\in(0,1/3)$ be a parameter
  whose value will be chosen later. We first specify the dynamics for $\Mstar$. Set $\cX=\crl*{\xfrak,\yfrak,\zfrak}$
  and $\cA=\crl*{\afrak,\bfrak,\cfrak}$. The initial state distribution sets
  $P_0(\xfrak)=1-\Delta$ and $P_0(\yfrak)=\Delta$. The transition dynamics are:
  \begin{itemize}
  \item
    $P_h(x'=\cdot\mid{}x=\xfrak,a)=\indic_{\xfrak}\cdot\indic\crl{a\in\crl{\afrak,\bfrak}} +
    \indic_{\zfrak}\cdot\indic\crl{a=\cfrak}$.
  \item $P_h(x'\mid{}x,a)=\indic\crl*{x'=x}$ for
    $x\in\crl{\yfrak,\zfrak}$. 
  \end{itemize}
  In other words, $\yfrak$ and $\zfrak$ are terminal states. For
  state $\xfrak$, actions $\afrak$ and $\bfrak$ are self-loops, but action
  $\cfrak$ transitions to $\zfrak$.

  The
  expert policies are $\pi\ind{\afrak}$, which sets
  $\pi_h\ind{\afrak}(x)=\afrak$ for all $h$ and $x\in\cX$, and
  $\pi\ind{\bfrak}$, which sets $\pi_h\ind{\bfrak}(\xfrak)=\afrak$ and sets
  $\pi_h\ind{\bfrak}(\yfrak)=\pi_h\ind{\bfrak}(\zfrak)=\bfrak$. We
  have $\Pi=\crl*{\pi\ind{\afrak},\pi\ind{\bfrak}}$.

  We consider two problem instances for the lower bound,
  $\cI\ind{\afrak}=(\Mstar,\pia,r\ind{\afrak})$, and $\cI\ind{\bfrak}=(\Mstar,\pib,r\ind{\bfrak})$.
  For problem instance $\cI\ind{\afrak}$,
  the expert policy is $\pia$. We set
  $r_h\ind{\afrak}(\xfrak,\cdot)=r_h\ind{\afrak}(\zfrak,\cdot)=0$,
  $r_h\ind{\afrak}(\yfrak,a)=\indic\crl*{a=\afrak}$ for all $h$. On the
  other hand, for problem instance $\cI\ind{\bfrak}$,
  the expert policy is $\pib$. We set
  $r_h\ind{\bfrak}(\xfrak,\cdot)=r_h\ind{\bfrak}(\zfrak,\cdot)=0$,
  $r_h\ind{\bfrak}(\yfrak,a)=\indic\crl*{a=\bfrak}$ for all $h$. Note
  that both of these choices for the reward function satisfy $\mu=1$,
  and that $\pia$ and $\pib$ are optimal policies for the respective
  instances. Let $J\ind{\afrak}$ denote the expected reward function
  for instance $\afrak$, and likewise for $\bfrak$.\loose

\paragraph{Upper bound on online sample complexity}  

We consider the following online algorithm. For episodes $t=1,\ldots,$:
\begin{itemize}
\item If $x_1\neq{} \xfrak$, proceed to the next episode.
\item If $x_1=\xfrak$, take action $\cfrak$, and observe
  $a_2=\pistar(\zfrak)$. If $a_2=\afrak$, return $\pihat=\pia$, and if
  $a_2=\bfrak$, return $\pihat=\pib$.
\end{itemize}
For any $\Delta\leq{}e^{-1}$, this algorithm will terminate after
$\log(1/\delta)$ episodes with probability at least $1-\delta$, and
whenever the algorithm terminates, it is clear that
$\pihat=\pistar$. In particular, this leads to zero regret for
\emph{any choice of reward function.}

\paragraph{Lower bound on offline sample complexity}
By setting $\Delta\propto{}\frac{1}{n}$, an argument essentially
identical to the proof of \cref{prop:lb_deterministic} shows that any
offline imitation learning algorithm must have
\begin{align}
      \max\crl*{
    \En\ind{\afrak}\brk*{J\ind{\afrak}(\pia) -     J\ind{\afrak}(\pihat)},
    \En\ind{\bfrak}\brk*{J\ind{\bfrak}(\pib) -
  J\ind{\bfrak}(\pihat)}}
  \approxgeq \Delta{}H \approxgeq{} \frac{H}{n}.
\end{align}
For the sake of avoiding repetition, we omit the details. Finally, we observe that since neither policy in $\Pi$ takes the
action $\cfrak$, \dagger---when equipped with any online learning algorithm
that predicts from a mixture of policies in $\Pi$, such as in
\cref{prop:dagger_finite})---will never take the action $\cfrak$,
and hence is subject to the $\frac{H}{n}$ lower bound from
\cref{prop:lb_deterministic} as well.
  
\end{proof}

}

\part{Additional Results}

\section{Additional Lower Bounds}
\label{sec:additional_lower}
This section contains additional lower bounds that complement the
results in \cref{sec:main,sec:stochastic}:
\begin{itemize}
\item \cref{sec:lb_weaker} shows that the conclusion of
  \cref{thm:lb_active} continues to hold even for online imitation
  learning in an \emph{active} sample complexity framework.
\item \cref{sec:instance_dependent} presents an instance-dependent
  lower bound for stochastic experts, complementing the minimax lower
  bound in \cref{thm:variance_lower}.
\item \cref{sec:converse} investigates the extent to which
  \cref{thm:bc_deterministic,thm:bc_stochastic} are tight on a per-policy basis.
\end{itemize}

\subsection{Lower Bounds for Online Imitation Learning in Active
  Interaction Model}
\label{sec:lb_weaker}
For the online imitation learning setting introduced in
\cref{sec:background}, we measure sample complexity in terms of the
total number of episodes of online interaction, and expert feedback is
available in every episode. In this section, we consider a more
permissive sample complexity framework inspired by active learning
\citep{hanneke2014theory,sekhari2024selective}. Here, as in
\cref{sec:background}, the learner interacts with the underlying MDP
$\Mstar$ through multiple episodes. At each episode $i\in\brk{n}$ the learner
executes a policy
$\pi\ind{i}=\crl*{\pi\ind{i}_h:\cX\to\Delta(\cA)}_{h=1}^{H}$, and at
any step $h$ in the episode, they can decide whether to query the expert
for an action $\astar_h\sim{}\pistar_h(x_h)$ at the current state
$x_h$. We set $M\ind{i}=1$ if the learner queries the expert at any point
during episode $i$ and set $M\ind{i}=0$ otherwise, and define the
\emph{active sample complexity} $M\ldef{}\sum_{i=1}^{n}M\ind{i}$ as the
total number of queries.

It is clear that the active sample complexity satisfies $m\leq{}n$, and in some
cases we might hope for it to be much smaller than the total number of
episodes, at least for a well-designed algorithm. While this can
indeed be the case for MDPs that satisfies (fairly strong)
distributional assumptions \citep{sekhari2024selective}, we will show
that the lower bound in \cref{prop:lb_deterministic} continues to hold in
this framework (up to a logarithmic factor), meaning that online
interaction in the active sample complexity framework cannot improve
over \loglossbc in general.\loose
    \begin{theorem}[Lower bound for deterministic experts in active
      sample complexity framework]
      \label{thm:lb_active}
        For any $m\in\bbN$ and $H\in\bbN$, there exists a reward-free
        MDP $\Mstar$ with $\abs{\cX}=\abs{\cA}=m+1$, a class of reward functions $\cR$ with
        $\abs*{\cR}=m+1$, and a class of
  deterministic policies $\Pi$ with $\log\abs{\Pi}=\log(m)$ with the following
  property. For any online imitation learning algorithm in the active
  sample complexity framework that has sample complexity
  $\En\brk*{M}\leq{}c\cdot{}m$ for an absolute constant $c>0$, there exists a deterministic reward
  function $r=\crl*{r_h}_{h=1}^{H}$ with $r_h\in\brk{0,1}$ and
  (optimal) expert policy $\pistar\in\Pi$
  with $\murec=1$ such that the expected suboptimality is lower bounded as\loose
  \begin{align}
    \En\brk*{J(\pistar)-J(\pihat)}
    \geq{} c\cdot\frac{H}{m}
  \end{align}
  for an absolute constant $c>0$.
  In addition, the dynamics, rewards, and expert policies are all
  stationary.\loose
\end{theorem}
Since this example has $\log\abs{\Pi}=\log(M)$, it follows that the
sample complexity bound for \loglossbc in \cref{thm:bc_deterministic}
(which uses $M=n$) can be improved by no more than a $\log(n)$ factor
through online interaction in the active framework.

\begin{proof}[\pfref{thm:lb_active}]
      \newcommand{\pia}{\pi\ind{\afrak}}%
      \newcommand{\pib}{\pi\ind{\bfrak}}%
      \newcommand{\bbPbar}{\wb{\bbP}}%
  Let $m\in\bbN$ and $H\in\bbN$ be fixed. 
  We first specify the dynamics for the reward-free MDP $\Mstar$. Set $\cX=\crl*{\xfrak_1,\ldots,\xfrak_m}$
  and $\cA=\crl*{\afrak,\bfrak}$. The initial state distribution is
  $P_0=\unif(\xfrak_1,\ldots,\xfrak_m)$. The transition dynamics are
  $P_h(x'\mid{}x,a)=\indic\crl*{x'=x}$ for all $h$; that is,
  $\xfrak_1,\ldots,\xfrak_m$ are all self-looping terminal states.

Let a \emph{problem instance} $\cI=(\Mstar,r,\pistar)$ refer to a tuple
consisting of the reward-free MDP $\Mstar$, a reward function
$r=\crl*{r_h}_{h=1}^{H}$, and an expert policy $\pistar$. We consider
$m+1$ problem instances $\cI\ind{0},\ldots,\cI\ind{m}$ parameterized
by a collection of policies $\Pi=\crl*{\pi\ind{0},\ldots,\pi\ind{m}}$
and reward functions $\cR=\crl*{r\ind{0},\ldots,r\ind{m}}$.
\begin{itemize}
\item For problem instance $\cI\ind{0}=(\Mstar,r\ind{0},\pi\ind{0})$,
  the expert policy is $\pi\ind{0}$, which sets
  $\pi_h\ind{0}(x)=\afrak$ for all $x\in\cX$ and $h\in\brk{H}$. The
  reward function $r\ind{0}$ sets $r_h(x,a)=\indic\crl*{a=\afrak}$ for
  all $x\in\cX$ and $h\in\brk{H}$.
\item For each problem instance $\cI\ind{j}=(\Mstar,r\ind{j},\pi\ind{j})$,
  the expert policy is $\pi\ind{j}$, which for all $h\in\brk{H}$ sets
  $\pi\ind{j}_h(x)=\afrak$ for $x\neq{}\xfrak_j$ and sets
  $\pi\ind{h}(\xfrak_j)=\bfrak$. The reward function $r\ind{j}$ sets
  $r_h(x,a)=\indic\crl*{a=\afrak,x\neq{}\xfrak_j} +
  \indic\crl*{a=\bfrak,x=\xfrak_j}$ for all $h\in\brk{H}$.
\end{itemize}
Let $J\ind{j}$ denote the expected reward under instance $j$. Note
that all
instances satisfy $\mu=1$, and that $\pi\ind{j}$ is an optimal policy for each instance $j$.

  Going forward, we fix the online imitation learning algorithm under consideration and let
  $\bbP\ind{j}$ denote the law of $o\ind{1},\ldots,o\ind{n}$ when
  $\afrak$ when we execute the algorithm on instance $\cI\ind{j}$; let
  $\En\ind{j}\brk*{\cdot}$ denote the corresponding expectation. In addition, for any policy $\pi$, let
  $\bbP^{\pi\ind{j}\mid{}\pi}$ denote the law of
  $o=(x_1,a_1,\astar_1),\ldots,(x_H,a_H,\astar_H)$ when we execute
  $\pi$ in the online imitation learning framework when the underlying
  instance is $\cI\ind{j}$, with the
  convention that $\astar_h=\perp$ if the learner does not query the
  expert in episode $j$.

  Our aim is to lower bound
  \begin{align}
    \max_{j\in\crl*{0,\ldots,m}}\En\ind{j}\brk*{J\ind{j}(\pi\ind{j})-J\ind{j}(\pihat)}
  \end{align}
To this end, define
$\rho_j(\pi,\pi')=\sum_{h=1}^{H}\En_{a_h\sim{}\pi_h(\xfrak_j),a'_h\sim\pi'_h(\xfrak_j)}\indic\crl*{a_h\neq{}a'_h}$
and $\rho(\pi,\pi')=\frac{1}{m}\rho_j(\pi,\pi')$, and observe that
  \begin{align}
    \En\ind{0}\brk*{J\ind{0}(\pi\ind{0})-J\ind{0}(\pihat)}
    &=
\En\ind{0}\brk*{\frac{1}{m}\sum_{j=1}^{m}\sum_{h=1}^{H}\En_{a_h\sim{}\pihat_h(\xfrak_j)}\brk*{\indic\crl*{a_h\neq\pi\ind{0}_h(\xfrak_j)}}}\\
    &=
    \En\ind{0}\brk*{\rho(\pihat, \pi\ind{0})} \geq{}
    \frac{H}{2m}\cdot{}\bbP\ind{0}\brk*{\rho(\pihat, \pi\ind{0}) \geq{} \frac{H}{2m}}.
  \end{align}
Next, note that for any $i\in\brk{m}$, if
$\rho(\pihat,\pi\ind{0})<\frac{H}{2m}$, then
$\rho_j(\pihat,\pi\ind{0})<\frac{H}{2}$, which means that
$\rho_j(\pihat,\pi\ind{j})\geq\frac{H}{2}$. It follows that
\begin{align}
  \En\ind{j}\brk*{J\ind{j}(\pi\ind{j})-J\ind{j}(\pihat)}
  &=   \En\ind{j}\brk*{\frac{1}{m}\rho_j(\pihat,\pi\ind{j})}
    \geq{} \frac{H}{2m}\bbP\ind{j}\brk*{\rho(\pihat, \pi\ind{0}) < \frac{H}{2m}},
\end{align}
and if we define $\bbPbar=\En_{j\sim\unif(\brk{m})}\bbP\ind{j}$, then
\begin{align}
  \En_{j\sim\unif(\brk{m})}\En\ind{j}\brk*{J\ind{j}(\pi\ind{j})-J\ind{j}(\pihat)}
  &    \geq{} \frac{H}{2m} \bbPbar\brk*{\rho(\pihat, \pi\ind{0}) < \frac{H}{2m}}.
\end{align}
Combining these observations, we find that
\begin{align}
  \max_{i\in\crl*{0,\ldots,m}}\En\ind{j}\brk*{J\ind{j}(\pi\ind{j})-J\ind{j}(\pihat)}
  &\geq{} \frac{H}{4m}\prn*{
  \bbP\ind{0}\brk*{\rho(\pihat, \pi\ind{0}) \geq{} \frac{H}{2m}}
  + \bbPbar\brk*{\rho(\pihat, \pi\ind{0}) < \frac{H}{2m}}
  }\\
&    \geq{} \frac{H}{4m}(1-\Dtv{\bbP\ind{0}}{\bbPbar}).
\end{align}

It remains to bound the total variation distance.
  Next, using Lemma D.2 of \citet{foster2024online}, we can bound
  \begin{align}
    \Dtvs{\bbP\ind{0}}{\bbPbar}
    \leq{}\Dhels{\bbP\ind{0}}{\bbPbar}
    \leq{}\En_{j\sim\unif\brk*{m}}\brk*{\Dhels{\bbP\ind{0}}{\bbP\ind{j}}}
    \leq{}7\En_{j\sim\unif\brk*{m}}\En\ind{0}\brk*{\sum_{t=1}^{n}
    \Dhels{\bbP^{\pi\ind{0}\mid{}\pi\ind{t}}}{\bbP^{\pi\ind{j}\mid{}\pi\ind{t}}}
    }.
  \end{align}
  Since the feedback the learner receives for a given episode $t$ is
  identical under instances $\cI\ind{0}$ and $\cI\ind{j}$ is identical
  unless i) $x_1=\xfrak_j$, and ii) the learner decides to query the
  expert for feedback (i.e., $M\ind{t}=1$), we can bound
  \begin{align}
    \Dhels{\bbP^{\pi\ind{0}\mid{}\pi\ind{t}}}{\bbP^{\pi\ind{j}\mid{}\pi\ind{0}}}
    \leq{} 2\bbP^{\pi\ind{0}\mid{}\pi\ind{t}}\brk*{x_1\ind{t}=\xfrak_j,M\ind{t}=1}
  \end{align}
  and hence
  \begin{align}
    \En_{j\sim\unif\brk*{m}}\En\ind{0}\brk*{\sum_{t=1}^{n}
    \Dhels{\bbP^{\pi\ind{0}\mid{}\pi\ind{t}}}{\bbP^{\pi\ind{j}\mid{}\pi\ind{t}}}
    }
    &\leq{}
    2 \En_{j\sim\unif\brk*{m}}\En\ind{0}\brk*{\sum_{t=1}^{n}
    \bbP^{\pi\ind{0}\mid{}\pi\ind{t}}\brk*{x_1\ind{t}=\xfrak_j,M\ind{t}=1}
      }\\
    &=
      \frac{2}{m} \En\ind{0}\brk*{\sum_{t=1}^{n}
      \sum_{j=1}^{m}\bbP^{\pi\ind{0}\mid{}\pi\ind{t}}\brk*{x_1\ind{t}=\xfrak_j,M\ind{t}=1}
      }\\
        &=
      \frac{2}{m} \En\ind{0}\brk*{\sum_{t=1}^{n}
      \bbP^{\pi\ind{0}\mid{}\pi\ind{t}}\brk*{M\ind{t}=1}
          }\\
            &=
      \frac{2}{m} \En\ind{0}\brk*{M
          }.
  \end{align}
  It follows that if $\En\ind{0}\brk*{M}\leq{}m/56$, then $
  \Dtv{\bbP\ind{0}}{\bbPbar}\leq{}1/2$, so that the algorithm must have
  \begin{align}
      \max_{i\in\crl*{0,\ldots,m}}\En\ind{j}\brk*{J\ind{j}(\pi\ind{j})-J\ind{j}(\pihat)}
&    \geq{} \frac{H}{8m}.
  \end{align}
  
\end{proof}

\subsection{An Instance-Dependent Lower Bound for Stochastic Experts}
\label{sec:instance_dependent}

In this section, we further investigate the optimality of \loglossbc
for stochastic experts (\cref{thm:bc_stochastic}). Recall that when
$\log\abs{\Pi}=\bigoh(1)$ the leading-order term in
\cref{thm:bc_stochastic} scales as roughly $\sqrt{\sigmastar^2/n}$,
where the salient quantity is the \emph{variance}
\[
  \sigmastar^2\ldef{}\sum_{h=1}^{H}\En^{\pistar}\brk*{(\Qstar_h(x_h,\pistar(x_h))-\Qstar_h(x_h,a_h))^2}
\]
for the expert policy $\pistar$. \cref{thm:variance_lower} shows that
this is optimal qualitatively, in the sense that for any value
$\sigma^2$, there exists a class of MDPs where the
$\sigmastar^2\leq\sigma^2$, and where the minimax rate is at least $\sqrt{\sigma^2/n}$.

In what follows, we will prove that for the special case of
\emph{autoregressive} MDPs (that is, the special case of the imitation learning problem in which the state
takes the form $x_h=a_{1:h-1}$; cf. \cref{sec:autoregressive}),
\cref{thm:variance_lower} is optimal on a \emph{per-policy}
basis. Concretely, we prove a \emph{local minimax} lower bound
\citep{donoho1991geometrizingii} which states that for any policy $\pistar$ and any
reward function $\rstar$, there exists a difficult ``alternative''
policy $\pitil$, such that in worst case over rewards
$r\in\crl*{-\rstar,+\rstar}$ and expert policies
$\pi\in\crl{\pistar,\pitil}$, any algorithm must have regret at least $\sqrt{\sigma^2/n}$.

\begin{theorem}
  \label{thm:value_estimation}
  Consider the offline imitation learning setting, and let $\Mstar$ be
  an autoregressive MDP. Let a reward function $\rstar$ with
  $\sum_{h=1}^{H}\rstar_h\in\brk*{0,R}$ almost surely be fixed, and let an expert
  policy $\pistar$ be given.
  For any $n\in\bbN$, there exists an alternative policy $\pitil$ such
  that
  \begin{align}
    \min_{\texttt{Alg}}\max_{\pi\in\crl*{\pistar,\pitil}}\max_{r\in\crl{\rstar,-\rstar}}\bbP\brk*{J(\pi)-J(\pihat)
    \geq{} c\cdot\sqrt{\frac{\sigma^2_{\pistar}}{n}}} \geq \frac{1}{4}
  \end{align}
for all $n \geq{} c'\cdot{}\frac{R^2}{\sigma^2_{\pistar}}$, where
$c,c'>0$ are absolute constants.
\end{theorem}

\cref{thm:value_estimation} suggests that the leading term in
\cref{thm:bc_stochastic} cannot be improved substantially without
additional assumptions, on a (nearly) per-instance basis. The restriction to $n \geq{}
  c'\cdot{}\frac{R^2}{\sigma^2_{\pistar}}$ in
  \cref{thm:value_estimation} is somewhat natural, as this corresponds
  to the regime in which the $\sqrt{\sigma^2_{\pistar}/n}$
  term in
  \cref{thm:bc_stochastic} dominates the lower-order term.

\begin{proof}[\pfref{thm:value_estimation}]
  We begin by observing that for any $\Delta>0$,
  \begin{align}
        \min_{\texttt{Alg}}\max_{\pi\in\crl*{\pistar,\pitil}}\max_{r\in\crl{\rstar,-\rstar}}\bbP\brk*{J_r(\pi)-J_r(\pihat)
    \geq{} \Delta}
    \geq{}
    \min_{\texttt{Alg}}\max_{\pi\in\crl*{\pistar,\pitil}}\bbP\brk*{\abs*{J_{\rstar}(\pi)-J_{\rstar}(\pihat)}
    \geq{} \Delta}.
  \end{align}
  with the convention that $J_r(\pi)$ denotes the expected reward
  under $r$; we abbreviate $J(\pi)\equiv{}J_{\rstar}(\pi)$ going forward. Let $\bbP_n^{\pi}$ denote the law of the offline
  imitation learning dataset under $\pi$. If we set $\Delta=\abs{J(\pistar)-J(\pitil)}/2$, then
  by the standard Le Cam two-point argument, we have that
  \begin{align}
    &\max\crl*{\bbP_n^{\pistar}\brk*{\abs*{J(\pistar)-J(\pihat)}
    \geq{} \Delta},
    \bbP_n^{\pitil}\brk*{\abs*{J(\pitil)-J(\pihat)}
    \geq{} \Delta}}\\
    &\geq{}
      \frac{1}{2}\prn*{
    1 - \bbP_n^{\pistar}\brk*{\abs*{J(\pistar)-J(\pihat)}
    < \Delta}+
    \bbP_n^{\pitil}\brk*{\abs*{J(\pitil)-J(\pihat)}
    \geq{} \Delta}
      }\\
        &\geq{}
          \frac{1}{2}\prn*{
    1 - \bbP_n^{\pistar}\brk*{\abs*{J(\pitil)-J(\pihat)}
          \geq{}\Delta}+
          \bbP_n^{\pitil}\brk*{\abs*{J(\pitil)-J(\pihat)}
          \geq{} \Delta}
          }\\
    &\geq{}
      \frac{1}{2}\prn*{1-\Dtv{\bbP_n^{\pistar}}{\bbP_n^{\pitil}}}\
    \geq{} \frac{1}{2}\prn*{1-\sqrt{n\cdot{}\Dhels{\bbP^{\pistar}}{\bbP^{\pitil}}}},
  \end{align}
  where the final inequality uses the standard tensorization property
  for Hellinger distance (e.g., \citet{wainwright2019high}).

  We will proceed by showing that 
  \begin{align}
    \label{eq:central}
    \omega_{\pistar}(\veps)\ldef{}\sup_{\pi}\crl*{\abs*{J(\pi)-J(\pistar)}\mid{}\Dhels{\bbP^{\pistar}}{\bbP^{\pi}}\leq\veps^2}
    \geq\bigom(1)\cdot\sqrt{\sigma^2_{\pistar}\cdot{}\veps^2},
  \end{align}
  for any $\veps>0$ sufficiently small, from which the result will follow
  by setting $\veps^2\propto1/n$ and
  \[
\pitil=\argmax_{\pi}\crl*{\abs*{J(\pi)-J(\pistar)}\mid{}\Dhels{\bbP^{\pistar}}{\bbP^{\pi}}\leq\veps^2}
\geq\bigom(1)\cdot\sqrt{\sigma^2_{\pistar}\cdot{}\veps^2}.
\]
To prove this, we will appeal to the following technical lemma.
  \begin{lemma}
    \label{lem:kl_dual}
    For any distribution $\bbQ$ and function $h$ with
    $\abs*{h}\leq{}R$ almost surely, it holds that for all
    $0\leq{}\veps^2\leq{}\frac{\Var_{\bbQ}\brk{h}}{4R^2}$, there exists a
    distribution $\bbP$ such that
    \begin{enumerate}
    \item $\En_{\bbP}\brk*{h}-\En_{\bbQ}\brk*{h} \geq{} 2^{-3}\sqrt{\Var_\bbQ\brk*{h}\cdot\veps^2}$
    \item $\Dkl{\bbQ}{\bbP}\leq\veps^2$.
    \end{enumerate}
  \end{lemma}
  Since stochastic policies $\pi$ in the autoregressive MDP $\Mstar$ are equivalent
  to arbitrary joint laws over the sequence $a_{1:H}$ (via Bayes'
  rule) and $J(\pi)=\En^{\pi}\brk*{\sum_{h=1}^{H}\rstar_h}$,
  \cref{lem:kl_dual} implies that for any $\veps^2\leq{}\Var^{\pistar}\brk*{\sum_{h=1}^{H}\rstar_h}/4R^2$, there exists a policy $\pitil$ such
  that (i) $\Dhels{\bbP^{\pistar}}{\bbP^{\pitil}}\leq
  \Dkl{\bbP^{\pistar}}{\bbP^{\pitil}}\leq\veps^2$, and (ii)
  \begin{align}
    J(\pitil) - J(\pistar)
    \geq{} 2^{-3}\sqrt{\Var^{\pistar}\brk*{\sum_{h=1}^{H}\rstar_h}\cdot\veps^2}.
  \end{align}
This establishes \cref{eq:central}. The result now follows by setting $\veps^2=\frac{c}{n}$ for an
absolute constant $c>0$ so that $\sqrt{n\cdot{}\Dhels{\bbP^{\pistar}}{\bbP^{\pitil}}}\leq{}1/2$, which is admissible whenever $n \geq{}
c'\cdot{}\frac{R^2}{\sigma^2_{\pistar}}$. Finally, we observe that for any deterministic MDP, by \cref{lem:ltv},
  \begin{align}
    \Var^{\pistar}\brk*{\sum_{h=1}^{H}r_h}
    = \En^{\pistar}\brk*{\sum_{h=1}^{H}\Var^{\pistar}\brk*{r_h +
    V_{h+1}^{\pistar}(x_{h+1})\mid{}x_{h}}}
    = \En^{\pistar}\brk*{\sum_{h=1}^{H}(Q^{\pistar}_h(x_h,a_h)-V_h^{\pistar}(x_h))^2}=\sigma^2_{\pistar},
  \end{align}
  since deterministic MDPs satisfy
  \begin{align}
    Q^{\pistar}_h(x_h,a_h) = r_h(x_h,a_h) + V_{h+1}^{\pistar}(x_{h+1})
  \end{align}
  almost surely, and since $\En^{\pistar}\brk*{Q^{\pistar}_h(x_h,a_h)\mid{}x_h}=V^{\pistar}_h(x_h)$.

\end{proof}

\begin{proof}[\pfref{lem:kl_dual}]
  Recall that we assume the domain is countable, so that $\bbQ$
admits a probability mass function $q$. We will define $\bbP$ via the
probability mass function
\begin{align}
  p(x) = \frac{q(x)e^{\eta{}h(x)}}{\sum_{x'}q(x')e^{\eta{}h(x')}}
\end{align}
for a parameter $\eta>0$. We begin by observing that
\begin{align}
  \Dkl{\bbQ}{\bbP}
  =\log\prn*{\En_{\bbQ}\brk*{e^{\eta{}h}}
  } - \eta{}\En_{\bbQ}\brk*{h}
  = \log\prn*{\En_{\bbQ}\brk*{e^{\eta{}(h-\En_{\bbQ}\brk*{h})}}}.
\end{align}
We now use the following lemma.
\begin{lemma}
  \label{lem:bernstein_mgf}
  For any random variable $X$ with $\abs*{X}\leq{}R$ almost surely and any $\eta\in(0,(2R)^{-1})$,
    \begin{align}
\frac{\eta^2}{8}\Var\brk*{X}  \leq  \log\prn*{\En\brk*{e^{\eta{}(X-\En\brk*{X})}}} \leq{} \eta^2\Var\brk*{X}.
    \end{align}
  \end{lemma}
  Hence, as long as $\eta\leq(2R)^{-1}$,
  \begin{align}
    \Dkl{\bbQ}{\bbP}
    \leq{} \eta^2\Var_{\bbQ}\brk*{h}.
  \end{align}
  We set $\eta = \min\crl*{\sqrt{\frac{\veps^2}{\Var_\bbQ\brk{h}}},
    \frac{1}{2R}}$ so that $\Dkl{\bbQ}{\bbP}\leq\veps^2$.

  Next, we compute that
  \begin{align}
    0 \leq{} \Dkl{\bbP}{\bbQ}
    = \eta\En_{\bbP}\brk*{h}
    - \log\prn*{\En_{\bbQ}\brk*{e^{\eta{}h}}},
  \end{align}
  so that
  \begin{align}
    \En_{\bbP}\brk*{h}
    -\En_{\bbQ}\brk*{h}
    \geq{}
      \eta^{-1}\log\prn*{\En_{\bbQ}\brk*{e^{\eta{}h}}}-\En_{\bbQ}\brk*{h}
    = \eta^{-1}\log\prn*{\En_{\bbQ}\brk*{e^{\eta{}(h-\En_{\bbQ}\brk*{h})}}}.
  \end{align}
  Since $\eta\leq{}(2R)^{-1}$, \cref{lem:bernstein_mgf} yields
  \begin{align}
    \En_{\bbP}\brk*{h}
    -\En_{\bbQ}\brk*{h} \geq{} \frac{\eta}{8}\Var_{\bbQ}\brk*{h} = \frac{1}{8}\sqrt{\Var_{\bbQ}\brk{h}\cdot\veps^2}
  \end{align}
  as long as $\veps^2\leq\frac{\Var_{\bbQ}\brk*{h}}{4R^2}$.

  \end{proof}

\begin{proof}[\pfref{lem:bernstein_mgf}]
  Note that $e^{x}\leq{}1+x+(e-2)x^2\leq{}1+x+x^2$ whenever
  $\abs*{x}\leq{}1$, and similarly $e^x\geq{}1+x+\frac{x^2}{4}$ for
  $\abs*{x}\leq{}1$. It follows that if $\eta\leq{}(2R)^{-1}$,
  \begin{align}
1+  \frac{\eta^2}{4}\Var(X)  \leq  \En\brk*{e^{\eta{}(X-\En\brk*{X})}}
    \leq{} 1+\eta^2\Var(X).
  \end{align}
  We conclude by using that
  $\frac{x}{2} \leq \log(1+x) \leq{} x$
  for $x\in\brk{0,1}$.
  
\end{proof}

  \subsection{Tightness of the Hellinger Distance Reduction}
  \label{sec:converse}

\cref{thm:bc_deterministic} and \cref{thm:bc_stochastic} are
supervised learning reductions that bound the regret of any policy
$\pihat$ in terms of its Hellinger distance
$\Dhels{\bbP^{\pihat}}{\bbP^{\pistar}}$ to the expert policy
$\pistar$. The following result shows that these reductions are tight
in a fairly strong instance-dependent sense: Namely, for any pair of policies
$\pihat$ and $\pistar$, and for any reward-free MDP $\Mstar$, it is
possible to design a reward function $r=\crl*{r_h}_{h=1}^{H}$ for
which each term in \cref{eq:bc_stochastic} of \cref{thm:bc_stochastic}
is tight, and such that \cref{thm:bc_deterministic} is tight; the only caveat is that we require the reward function to be
\emph{non-Markovian}, in the sense that $r_h$ depends on the full
history $x_{1:h}$ and $a_{1:h}$.
      \begin{theorem}[Converse to \cref{thm:bc_deterministic,thm:bc_stochastic}]
        \label{thm:hellinger_converse}
            Let a reward-free MDP $\Mstar$ and a pair of (potentially stochastic)
            policies $\pihat$ and $\pistar$ be given.
            \begin{enumerate}
            \item               For any $R>0$, there exists a
              non-Markovian reward function $r=\crl*{r_h}_{h=1}^{H}$
              with $\sum_{h=1}^{H}r_h\in\brk*{0,R}$ such that
              \begin{align}
                \label{eq:converse1}
        J(\pistar)-J(\pihat)
        \geq{} \frac{R}{6}\cdot{}\DhelsX{\big}{\bbP^{\pihat}}{\bbP^{\pistar}}.
              \end{align}
            \item For any $\sigma^2>0$, there exists a non-Markovian
              reward function $r=\crl*{r_h}_{h=1}^{H}$ for which
$\sigmastar^2\ldef\sum_{h=1}^{H}\En^{\pistar}\brk*{(\Qstar_h(x_{1:h},a_{1:h})-\Vstar_h(x_{1:h},a_{1:h-1}))^2}\leq\sigma^2$,
and such that\footnote{Note that since the reward function under consideration is
  non-Markovian, the value functions $Q^{\pistar}_h$ and
  $V^{\pistar}_h$ depend on the full history $x_{1:h},a_{1:h-1}$.}
\begin{align}
  \label{eq:converse2}
        J(\pistar)-J(\pihat)
                        \geq{} \frac{1}{9}\sqrt{\sigma^2\cdot\Dhels{\bbP^{\pihat}}{\bbP^{\pistar}}}.
\end{align}
\item      For any $R>0$ and $\sigma^2>0$, there exists a
              non-Markovian reward function $r=\crl*{r_h}_{h=1}^{H}$
              with $\sum_{h=1}^{H}r_h\in\brk*{0,R}$ and
              $\sigma_{\pistar}^2\leq\sigma^2$ simultaneously
              such that
              \begin{align}
                \label{eq:tv_converse}
                J(\pistar) - J(\pihat) 
                \geq{} \frac{1}{9}\min\crl*{\sqrt{\sigma^2\cdot\Dhels{\bbP^{\pihat}}{\bbP^{\pistar}}},R\cdot\DhelsX{\big}{\bbP^{\pihat}}{\bbP^{\pistar}}}.
              \end{align}
            \end{enumerate}
          \end{theorem}
\cref{eq:converse1} shows that there exist reward functions with
bounded range for which \cref{thm:bc_deterministic} and the
lower-order term in \cref{eq:bc_stochastic} of
\cref{thm:bc_stochastic} are tight, while \cref{eq:converse2} shows
that there exist reward functions with bounded variance (but not
necessarily bounded range) for which the leading term in
\cref{eq:bc_stochastic} or \cref{thm:bc_stochastic} is tight.\loose
          
Note that for some MDPs, the state $x_h$ already contains the full
history $x_{1:h-1},a_{1:h-1}$, so the assumption of non-Markovian
rewards is without loss of generality. For MDPs that do not
have this property, \cref{thm:hellinger_converse} leaves open the possibility that
\cref{thm:bc_deterministic,thm:bc_stochastic} can be improved on a
per-MDP basis.

\begin{proof}[\pfref{thm:hellinger_converse}]
      \newcommand{\bbQbar}{\wb{\bbQ}}%
    \newcommand{\bbPbar}{\wb{\bbP}}%
    Consider a pair of measures $\bbP$ and $\bbQ$, and set
    $\bbPbar\ldef{}\frac{1}{2}(\bbP+\bbQ)$. Consider the function
    \[
      h = 1-\frac{1}{2}\frac{\bbQ}{\bbPbar}\in\brk*{0,1}.
    \]
    Using \cref{lem:pinsker}, we observe that
    \begin{align}
      \label{eq:chi1}
      \En_{\bbP}\brk*{h}-\En_{\bbQ}\brk*{h}
      = 2\prn*{ \En_{\bbPbar}\brk*{h}-\En_{\bbQ}\brk*{h}}
      =\En_{\bbQ}\brk*{\frac{\bbQ}{\bbPbar}}-\En_{\bbPbar}\brk*{\frac{\bbQ}{\bbPbar}}
      = \Dchis{\bbQ}{\bbPbar}
      \geq{} \frac{1}{6}\Dhels{\bbQ}{\bbP}.
    \end{align}
    We also observe that by concavity of variance,
    \begin{align}
      \label{eq:chi2}
      \frac{1}{2}\prn*{\Var_{\bbP}\brk*{h}+\Var_{\bbQ}\brk*{h}}
      \leq{}\Var_{\bbPbar}\brk*{h}
      =
      \frac{1}{4}\En_{\bbPbar}\brk*{\prn*{\frac{\bbQ}{\bbPbar}-\En_{\bbPbar}\brk*{\frac{\bbQ}{\bbPbar}}}}^2
=       \Dchis{\bbQ}{\bbPbar}\leq{}\Dhels{\bbQ}{\bbP}.
    \end{align}
    To apply this observation to the theorem at hand, let a parameter $B>0$ be
    given, let
    $\bbPbar\ldef{}\frac{1}{2}(\bbP^{\pistar}+\bbP^{\pihat})$, and
    consider the non-Markov reward function $r$ that sets
    $r_{1},\ldots,r_{h-1}=0$ and
    \begin{align}
      r_H(\tau) = B\cdot\prn*{1-\frac{1}{2}\frac{\bbP^{\pihat}}{\bbPbar}}\in\brk*{0,B}.
    \end{align}
    Then by \cref{eq:chi1}, we have that
    \begin{align}
      J(\pistar)-J(\pihat)
      \geq{} \frac{B}{6}\cdot{}\Dhels{\bbP^{\pihat}}{\bbP^{\pistar}}.
    \end{align}
    At the same time, by \cref{eq:chi2}, we have that
    \begin{align}
      \Var^{\pistar}\brk*{\sum_{h=1}^{H}r_h}
      =       \Var^{\pistar}\brk*{r_H}
      \leq{} 2B^2\cdot{}\Dhels{\bbP^{\pihat}}{\bbP^{\pistar}},
    \end{align}
    and by \cref{prop:sigma_ltv},
    \begin{align}
      \sigma_{\pistar}^2
      \leq{}       \Var^{\pistar}\brk*{\sum_{h=1}^{H}r_h}.
    \end{align}

    To conclude, note that if we set $B^2=R^2$, then
    $\sum_{h=1}^{H}r_h\in\brk*{0,R}$ and
          \begin{align}
        J(\pistar)-J(\pihat)
        \geq{} \frac{R}{6}\cdot{}\Dhels{\bbP^{\pihat}}{\bbP^{\pistar}}.
          \end{align}
          Meanwhile, if we set
                \begin{align}
                  B^2 = \frac{\sigma^2}{2 \Dhels{\bbP^{\pihat}}{\bbP^{\pistar}}},
                \end{align}
                then $\sigmastar^2\leq\sigma^2$ and
                      \begin{align}
        J(\pistar)-J(\pihat)
                        \geq{} \frac{1}{9}\sqrt{\sigma^2\cdot\Dhels{\bbP^{\pihat}}{\bbP^{\pistar}}}.
      \end{align}

Finally, if we set
      \begin{align}
        B^2 = \frac{\sigma^2}{2 \Dhels{\bbP^{\pihat}}{\bbP^{\pistar}}}\wedge{}R^2.
      \end{align}
      Then $\sum_{h=1}^{H}r_h\in\brk{0,R}$, $\sigma_{\pistar}^2\leq{} \sigma^2$, and
      \begin{align}
        J(\pistar)-J(\pihat)
        \geq{} \frac{B}{6}\cdot{}\Dhels{\bbP^{\pihat}}{\bbP^{\pistar}}
        \geq{} \min\crl*{\frac{1}{9}\sqrt{\sigma^2\cdot\Dhels{\bbP^{\pihat}}{\bbP^{\pistar}}},\frac{R}{6}\cdot\Dhels{\bbP^{\pihat}}{\bbP^{\pistar}}}.
      \end{align}
    
\end{proof}

\newpage

\end{document}